\documentclass[11pt]{amsart}
\usepackage[letterpaper,margin=1in]{geometry}

\usepackage{amsfonts}
\usepackage{graphicx}
\usepackage{algorithmic}
\DeclareGraphicsExtensions{.pdf,.png,.jpg}

\usepackage[hypertexnames=false]{hyperref}
\usepackage{xcolor}

\theoremstyle{plain}
\newtheorem{theorem}{Theorem}[section]
\newtheorem{proposition}[theorem]{Proposition}
\newtheorem{lemma}[theorem]{Lemma}
\newtheorem{corollary}[theorem]{Corollary}

\theoremstyle{definition}
\newtheorem{definition}[theorem]{Definition}
\newtheorem{assumption}[theorem]{Assumption}
\theoremstyle{remark}
\newtheorem{remark}[theorem]{Remark}
\newtheorem{example}[theorem]{Example}

\usepackage[capitalize,nameinlink]{cleveref}

\crefname{hypothesis}{Hypothesis}{Hypotheses}
\crefname{fact}{Fact}{Facts}

\title[A convergence result of a continuous model of deep learning]{A convergence result of a continuous model of deep~learning via a \L{}ojasiewicz--Simon inequality}

\author[N. Isobe]{Noboru Isobe}
\address{RIKEN, Tokyo, Japan}
\email{noboru.isobe@riken.jp}
\urladdr{https://researchmap.jp/noboru-isobe?lang=en}
\thanks{This work was funded by JSPS KAKENHI 22J20130.}

\usepackage{amsopn}
\usepackage[T1]{fontenc}

\usepackage{url}
\usepackage[doi=false,isbn=false, sortcites=true,eprint=true, date=year, style=numeric,uniquename=init,giveninits=true,sorting=nty,url=false,maxnames=99,backref=false]{biblatex} % alphabetic
\DeclareFieldFormat*{volume}{\mkbibbold{#1}} 
\DeclareFieldFormat*{title}{\mkbibquote{#1\adddot}}  
\DeclareFieldFormat*{booktitle}{\mkbibemph{#1}}  
\DeclareSourcemap{
  \maps[datatype=bibtex]{
    \map{
      \step[fieldsource=edition,
            match=\regexp{^(\d+)\.\s+(ed.|edition)$},
            replace=\regexp{$1}]
    }
    \map[overwrite]{
      \step[fieldsource=shortbooktitle, final]
      \step[fieldset=booktitle, origfieldval]
      \step[fieldset=series, null]
      \step[fieldset=volume, null]
      \step[fieldset=number, null]
      \step[fieldset=publisher, null]
      \step[fieldset=address, null]
      \step[fieldset=location, null]
      \step[fieldset=organization, null]
    }
    \map[overwrite]{
      \step[fieldsource=shortjournal, final]
      \step[fieldset=journal, origfieldval]
    }
  }
}
\AtEveryBibitem{\clearlist{language}\clearname{editor}}
\AtBeginBibliography{\emergencystretch=1em}
\AtEveryBibitem{%
  \iffieldundef{journaltitle}{}{%
    \clearfield{eprint}\clearfield{eprinttype}\clearfield{eprintclass}}%
  \iffieldundef{booktitle}{}{%
    \clearfield{eprint}\clearfield{eprinttype}\clearfield{eprintclass}}%
}
\DeclareFieldFormat*{journaltitle}{\mkbibemph{#1}} 
\DeclareFieldFormat[article,Article,book,inbook,incollection,inproceedings,patent,thesis,unpublished,misc]{year}{#1} %
\renewbibmacro{in:}{} %
\renewbibmacro*{issue+date}{%
  \printfield{issue}%
  \setunit*{\addspace}%
  \usebibmacro{date}%
  \newunit} %
\hypersetup{
    colorlinks=true,
    citecolor=blue,
    linkcolor=blue,
    urlcolor=teal,
    breaklinks=true
}
\usepackage{amssymb}
\usepackage{mathrsfs}
\expandafter\let\expandafter\oldproof\csname\string\proof\endcsname

\usepackage{array}
\usepackage{latexsym}
\usepackage{amsmath}
\allowdisplaybreaks[4]
\usepackage{stackengine}
\usepackage{empheq}

\usepackage{here}
\usepackage{longtable}
\usepackage{threeparttable}
\usepackage{upgreek}
\usepackage{boldline}
\usepackage{multirow}
\usepackage[super]{nth}
\usepackage{enumitem}
\setlist[description]{%
  topsep=10pt,               % space before start / after end of list
  itemsep=5pt,               % space between items
  font={\bfseries\rmfamily}, % set the label font
}
\newlist{myenum}{enumerate}{3}
\setlist[myenum,1]{label=\textbf{(\arabic*)},
                   ref  =\textbf{(\arabic*)}}
\setlist[myenum,2]{label=\textbf{(\alph*)},
                   ref  =\themyenumi\textbf{(\alph*)}}
\setlist[myenum,3]{label=\bfseries(\roman*),
                   ref  =\themyenumii\textbf{.(\roman*)}}

\usepackage{mathtools}
\usepackage{etoolbox}
\usepackage[italic=true]{derivative}
\usepackage{physics}
\DeclarePairedDelimiter{\absdelim}{\lvert}{\rvert}
\DeclarePairedDelimiter{\normdelim}{\lVert}{\rVert}
\DeclareDocumentCommand{\abs}{m e{_^}}{\mathchoice
    {\absdelim*{#1}\IfValueT{#2}{_{#2}}\IfValueT{#3}{^{#3}}}
    {\absdelim{#1}\IfValueT{#2}{_{#2}}\IfValueT{#3}{^{#3}}}
    {\absdelim{#1}\IfValueT{#2}{_{#2}}\IfValueT{#3}{^{#3}}}
    {\absdelim{#1}\IfValueT{#2}{_{#2}}\IfValueT{#3}{^{#3}}}}
\DeclareDocumentCommand{\norm}{m e{_^}}{\mathchoice
    {\normdelim*{#1}\IfValueT{#2}{_{#2}}\IfValueT{#3}{^{#3}}}
    {\normdelim{#1}\IfValueT{#2}{_{#2}}\IfValueT{#3}{^{#3}}}
    {\normdelim{#1}\IfValueT{#2}{_{#2}}\IfValueT{#3}{^{#3}}}
    {\normdelim{#1}\IfValueT{#2}{_{#2}}\IfValueT{#3}{^{#3}}}}
\usepackage{nicefrac}
\mathtoolsset{showonlyrefs} % number only referenced equations; used with align (no align starred)
\DeclareDocumentCommand{\eval}{m e{_^}}{\mathchoice
    {\left.{#1}\right|\IfValueT{#2}{_{#2}}\IfValueT{#3}{^{#3}}}
    {{#1}\rvert\IfValueT{#2}{_{#2}}\IfValueT{#3}{^{#3}}}
    {{#1}\rvert\IfValueT{#2}{_{#2}}\IfValueT{#3}{^{#3}}}
    {{#1}\rvert\IfValueT{#2}{_{#2}}\IfValueT{#3}{^{#3}}}}
\SetMathAlphabet{\mathcal}{normal}{OMS}{cmsy}{m}{n} % fixes ugly \mathcals
\SetMathAlphabet{\mathcal}{bold}{OMS}{cmsy}{m}{n} % fixes ugly \mathcals
\DeclareFontFamily{U}{matha}{\hyphenchar\font45}
\DeclareFontShape{U}{matha}{m}{n}{
<-6> matha5 <6-7> matha6 <7-8> matha7
<8-9> matha8 <9-10> matha9
<10-12> matha10 <12-> matha12
}{}
\DeclareSymbolFont{matha}{U}{matha}{m}{n}

\DeclareFontFamily{U}{mathx}{\hyphenchar\font45}
\DeclareFontShape{U}{mathx}{m}{n}{
<-6> mathx5 <6-7> mathx6 <7-8> mathx7
<8-9> mathx8 <9-10> mathx9
<10-12> mathx10 <12-> mathx12
}{}
\DeclareSymbolFont{mathx}{U}{mathx}{m}{n}

\DeclareMathDelimiter{\vvvert} {0}{matha}{"7E}{mathx}{"17}%

\DeclarePairedDelimiterX{\normiii}[1]
{\vvvert}
{\vvvert}
{\ifblank{#1}{\:\cdot\:}{#1}}
\usepackage{appendix}

\usepackage{lmodern}
\SetMathAlphabet{\mathcal}{normal}{OMS}{cmsy}{m}{n} % fixes ugly \mathcals
\SetMathAlphabet{\mathcal}{bold}{OMS}{cmsy}{m}{n} % fixes ugly \mathcals
\DeclareSymbolFont{EulerExtension}{U}{euex}{m}{n}
\DeclareMathSymbol{\euintop}{\mathop} {EulerExtension}{"52}
\DeclareMathSymbol{\euointop}{\mathop} {EulerExtension}{"48}

\DeclareSymbolFont{cmletters}{OML}{cmm}{m}{it}
\DeclareSymbolFont{cmsymbols}{OMS}{cmsy}{m}{n}
\DeclareSymbolFont{cmlargesymbols}{OMX}{cmex}{m}{n}
\DeclareMathSymbol{\myjmath}{\mathord}{cmletters}{"7C}
\DeclareMathSymbol{\myamalg}{\mathbin}{cmsymbols}{"71}
\DeclareMathSymbol{\mycoprod}{\mathop}{cmlargesymbols}{"60}
\let\jmath\myjmath

\numberwithin{equation}{section}

\allowdisplaybreaks[4]

\usepackage{makecell}
\usepackage{tikz}
\usetikzlibrary{calc}
\tikzset{
    ncbar angle/.initial=90,
    ncbar/.style={
        to path=(\tikztostart)
        -- ($(\tikztostart)!#1!\pgfkeysvalueof{/tikz/ncbar angle}:(\tikztotarget)$)
        -- ($(\tikztotarget)!($(\tikztostart)!#1!\pgfkeysvalueof{/tikz/ncbar angle}:(\tikztotarget)$)!\pgfkeysvalueof{/tikz/ncbar angle}:(\tikztostart)$)
        -- (\tikztotarget)
    },
    ncbar/.default=0.5cm,
}

\tikzset{square left brace/.style={ncbar=0.5cm}}
\tikzset{square right brace/.style={ncbar=-0.5cm}}

\tikzset{round left paren/.style={ncbar=0.5cm,out=120,in=-120}}
\tikzset{round right paren/.style={ncbar=0.5cm,out=60,in=-60}}
\usetikzlibrary{positioning, fit}   
\tikzset{block/.style={draw, thick, text width=2cm ,minimum height=1.3cm, align=center},   
line/.style={-latex}     
}

\crefname{theorem}{Theorem}{Theorems}
\crefname{appendix}{Appendix}{Appendices}
\crefname{definition}{Definition}{Definitions}
\crefname{problem}{Problem}{Problems}
\crefname{fact}{Fact}{Facts}
\crefname{proposition}{Proposition}{Propositions}
\crefname{lemma}{Lemma}{Lemmas}
\crefname{corolary}{Corolary}{Corolaries}
\crefname{assumption}{Assumption}{Assumptions}
\crefname{claim}{Claim}{Claims}
\crefformat{myenumi}{#2#1#3}
\crefrangeformat{myenumi}{#3#1#4 to #5#2#6}
\crefmultiformat{myenumi}{#2#1#3}{ and #2#1#3}{, #2#1#3}{, and #2#1#3}
\crefformat{myenumii}{#2#1#3}
\crefrangeformat{myenumii}{#3#1#4 to #5#2#6}
\crefmultiformat{myenumii}{#2#1#3}{ and #2#1#3}{, #2#1#3}{, and #2#1#3}
\crefformat{myenumiii}{#2#1#3}
\crefrangeformat{myenumiii}{#3#1#4 to #5#2#6}
\crefmultiformat{myenumiii}{#2#1#3}{ and #2#1#3}{, #2#1#3}{, and #2#1#3}

\crefname{remark}{Remark}{Remarks}
\crefname{example}{Example}{Examples}
\crefname{corollary}{Corollary}{Corollaries}
\crefname{subsubsection}{Subsubsection}{Subsubsections}
\crefname{subsection}{Subsection}{Subsections}
\crefname{section}{Section}{Sections}
\crefname{chapter}{Chapter}{Chapters}
\crefname{table}{Table}{Tables}
\crefname{figure}{Figure}{Figures}
\crefname{algorithm}{Algorithm}{Algorithms}
\crefname{myenumi}{item}{items}
\crefname{myenumii}{item}{items}
\crefname{myenumiii}{item}{items}
\AtBeginDocument{%
\crefformat{appendix}{Appendix~#2#1#3}%
\Crefformat{appendix}{Appendix~#2#1#3}%
\crefrangeformat{appendix}{Appendices~#3#1#4--#5#2#6}%
\Crefrangeformat{appendix}{Appendices~#3#1#4--#5#2#6}%
\crefmultiformat{appendix}{Appendices~#2#1#3}{ and #2#1#3}{, #2#1#3}{, and #2#1#3}%
\Crefmultiformat{appendix}{Appendices~#2#1#3}{ and #2#1#3}{, #2#1#3}{, and #2#1#3}%
\Crefformat{section}{#2Section~#1#3}%
\Crefrangeformat{section}{Sections~#3#1#4--#5#2#6}%
\Crefmultiformat{section}{Sections~#2#1#3}{ and #2#1#3}{, #2#1#3}{, and #2#1#3}%
\Crefformat{subsection}{#2Subsection~#1#3}%
\Crefrangeformat{subsection}{Subsections~#3#1#4--#5#2#6}%
\Crefmultiformat{subsection}{Subsections~#2#1#3}{ and #2#1#3}{, #2#1#3}{, and #2#1#3}%
\crefformat{section}{#2\S#1#3}%
\crefrangeformat{section}{\S\S#3#1#4--#5#2#6}%
\crefmultiformat{section}{#2\S#1#3}{ and #2\S#1#3}{, #2\S#1#3}{, and #2\S#1#3}%
\crefformat{subsection}{#2\S#1#3}%
\crefrangeformat{subsection}{\S\S#3#1#4--#5#2#6}%
\crefmultiformat{subsection}{#2\S#1#3}{ and #2\S#1#3}{, #2\S#1#3}{, and #2\S#1#3}%
}

\renewcommand{\ref}{\cref}

\newcommand{\e}{\mathrm{e}}
\newcommand{\Leb}{\mathrm{Leb}}
\newcommand{\opn}{\textup{op}}
\newcommand{\Pc}{\mathcal{P}_\mathrm{c}}
\newcommand{\Pcal}{\mathcal{P}}
\newcommand{\Ptwo}{\Pcal_2}

\newcommand{\Bcal}{\mathcal{B}}
\newcommand{\Dcal}{\mathcal{D}}
\newcommand{\Fcal}{\mathcal{F}}
\newcommand{\Lcal}{\mathcal{L}}
\newcommand{\Ccinf}{C_c^\infty}
\newcommand{\R}{\mathbb{R}}

\newcommand{\N}{\mathbb{N}}
\newcommand{\C}{\mathbb{C}}
\newcommand{\Y}{\mathcal{Y}}
\newcommand{\X}{\mathcal{X}}
\newcommand{\XtimesY}{\X\times\Y}
\newcommand{\metricsp}{\mathcal{M}}
\newcommand{\Hilbertsp}{\mathcal{H}}
\newcommand{\strongsp}{\mathcal{V}}
\newcommand{\AC}{\mathrm{AC}}
\newcommand{\Ltwotheta}{L^2\qty(I;\mathcal{P}_2(\R^m))}
\newcommand{\varThetaSpace}{L^2\qty(I;L^2\qty(\Omega;\R^m))}
\newcommand{\varThetaSpaceInf}{L^\infty\qty(I;L^\infty\qty(\Omega;\R^m))}

\newcommand{\tildeE}{\widetilde{E}}
\newcommand{\OptPlan}{\Gamma_{\textup{o}}}

\NewDocumentCommand{\lascan}{m e{_^}}{\mathchoice
    {\left\langle#1\right\rangle\IfValueT{#2}{_{#2}}\IfValueT{#3}{^{#3}}}
    {\langle#1\rangle\IfValueT{#2}{_{#2}}\IfValueT{#3}{^{#3}}}
    {\langle#1\rangle\IfValueT{#2}{_{#2}}\IfValueT{#3}{^{#3}}}
    {\langle#1\rangle\IfValueT{#2}{_{#2}}\IfValueT{#3}{^{#3}}}}
\def\la#1\ra{\lascan{#1}}
\newcommand{\ra}{\rangle}

\newcommand{\one}{\mathrm{I}}
\newcommand{\two}{\one\hspace{-1.2pt}\one}
\newcommand{\three}{\one\hspace{-1.2pt}\one\hspace{-1.2pt}\one}
\makeatletter
\newcommand{\pushright}[1]{\ifmeasuring@#1\else\omit\hfill$\displaystyle#1$\fi\ignorespaces}
\newcommand{\pushleft}[1]{\ifmeasuring@#1\else\omit$\displaystyle#1$\hfill\fi\ignorespaces}
\makeatother
\DeclareMathOperator{\Div}{div}
\DeclareMathOperator{\Dir}{Dir}

\DeclareMathOperator{\Ran}{Ran}
\DeclareMathOperator{\Law}{Law}
\DeclareMathOperator{\Expect}{\mathbb{E}}

\DeclareMathOperator*{\supp}{supp}
\DeclareMathOperator*{\minimize}{minimize}

\DeclareMathOperator{\Hessian}{Hess}

\DeclareMathOperator{\Id}{Id}
\DeclareMathOperator{\Tan}{Tan}

\def\Set#1{\Setsplit#1\Setsplit}
\def\Setsplit#1|#2\Setsplit{\mathchoice
  {\left\{\:#1\;\middle|\;#2\:\right\}}
  {\{\,#1\mid#2\,\}}
  {\{\,#1\mid#2\,\}}
  {\{\,#1\mid#2\,\}}}%
\renewcommand{\th}{%
    \ifmmode% math mode
        ^\mathrm{th}%
    \else%
        \textsuperscript{th}\xspace%
    \fi%
}
\makeatletter
\newcommand{\subalign}[1]{%
  \vcenter{%
    \Let@ \restore@math@cr \default@tag
    \baselineskip\fontdimen10 \scriptfont\tw@
    \advance\baselineskip\fontdimen12 \scriptfont\tw@
    \lineskip\thr@@\fontdimen8 \scriptfont\thr@@
    \lineskiplimit\lineskip
    \ialign{\hfil$\m@th\scriptstyle##$&$\m@th\scriptstyle{}##$\hfil\crcr
      #1\crcr
    }%
  }%
}
\makeatother

\ifpdf
\hypersetup{
  pdftitle={A convergence result of a continuous model of deep learning via a Lojasiewicz--Simon inequality},
  pdfauthor={N. Isobe}
}
\fi

\makeatletter
\renewcommand{\tableofcontents}{%
  {\centering\large\bfseries Contents\par}\smallskip
  \@starttoc{toc}%
}
\let\l@paragraph\@gobbletwo% keep \paragraph run-in heads out of the table of contents
\makeatother
\begin{document}

\begin{abstract}
We study an idealized training process for deep neural networks in a continuous-depth, mean-field model in which each layer is parameterized by a probability measure on a Euclidean parameter space.
The training dynamics are formulated as a Wasserstein-type gradient flow of an objective with a fixed $L^2$-regularization.
Under suitable analyticity and growth assumptions, together with a coercivity assumption and sufficient regularity of the initial data, we prove that every curve of maximal slope converges to a single critical point of the objective as the training time tends to infinity.
The proof combines compactness of the curve with a \L{}ojasiewicz--Simon inequality for the metric slope.
To establish the inequality, we lift the objective to a Hilbert space of random variables and use the analyticity of the lifted gradient in a stronger $L^\infty$ topology to overcome its lack of continuous differentiability in the Hilbert-space topology.
Our convergence result does not require global displacement convexity, a Polyak--\L{}ojasiewicz-type condition, or initialization near a minimizer; the objective may remain genuinely nonconvex.
\end{abstract}

\maketitle

% \tableofcontents

\section{Introduction}\label{sec:Intro}
%%%%%%%%%%%%%%%%%%%%%%
One of the major challenges in the mathematical theory of deep learning is analyzing the training dynamics.
Training dynamics refer to the evolution of parameters in deep neural networks (DNNs) during gradient-based optimization.
In continuous time, this evolution is idealized as a gradient flow of an objective function.
In a large-scale limit of neural networks, the parameters are represented by probability measures, and the training dynamics can be formulated as a Wasserstein-type gradient flow \cite{pmlr-v99-mei19a,Fernández-Real2022,E2020}.
For shallow two-layer mean-field neural networks (NNs), convergence of the flow to global minimizers has been proven under certain structural and initialization assumptions \cite{NEURIPS2018_a1afc58c,bach2021gradient}.
In contrast, existing results for representative deep architectures such as ResNets \cite{He_2016_CVPR} and for neural ODEs, which arise as their continuous-depth limits \cite{ChenRBD18,Haber17,E2020}, have largely focused on the optimality of a long-time limit~\cite{ding21on,ding22overparam}.
These works show that \emph{if the gradient flow converges}, then its limit is a global minimizer with zero loss.

This paper addresses the central question of whether the training trajectory of a mean-field neural ODE converges as training time tends to infinity.
Several convergence results are known under additional assumptions.
Jabir, \v{S}i\v{s}ka, and Szpruch \cite{jabir2021meanfieldneural} study entropy-regularized mean-field dynamics and obtain exponential convergence when the regularization is sufficiently strong.
Barboni, Peyr\'e, and Vialard \cite{Barboni25} prove convergence under a local Polyak--\L{}ojasiewicz condition for initial data with sufficiently small objective values.
In an optimal-control setting, Scagliotti \cite{Scagliotti2023} studies a parameterized ODE whose vector field is linear in the parameter and proves convergence of the associated gradient flow.
These results, however, do not directly apply to the training of neural ODEs with a fixed regularization, nonlinear network parameterizations, or initial data that need not lie in a prescribed neighborhood of a minimizer.
The analytical challenge lies in proving convergence to a single critical point of a nonconvex Wasserstein gradient flow without global displacement convexity or a Polyak--\L{}ojasiewicz inequality.

In this paper, we study an $L^2$-regularized training objective and prove long-time convergence of its Wasserstein-type gradient flow to a critical point in this nonconvex setting.
Under analytic and growth assumptions on the network vector field and the loss, together with a parameter-Hessian condition, every curve of maximal slope for the objective, starting from sufficiently regular initial data, converges in the $L^2$-in-depth Wasserstein space to a critical point (\cref{thm:global_conv}); the theorem establishes convergence of the training trajectory itself and does not assert that the limiting critical point is a global minimizer.
The proof combines compactness of the orbit, obtained from a contraction estimate for the parameter characteristics, with a \L{}ojasiewicz--Simon inequality for the metric slope.
To establish the inequality, we lift the objective to a Hilbert space of random variables; the analysis is nonstandard because the lifted gradient is not continuously differentiable in this Hilbert space.

\subsection{Model and training dynamics}\label{subsec:intro_model}
We now introduce the continuous-depth mean-field model precisely.
Let $\X=\R^d$ and let $\Y\subset\R^d$ be compact.
An input--label pair $(x,y)\in\XtimesY$ is sampled from a compactly supported data distribution $\mu_0\in\Pc(\XtimesY)$.
At each depth $t\in I\coloneqq[0,1]$, the network parameters are described by a probability measure $\eta_t\in\Ptwo(\R^m)$.
Here, $\Pc(S)$ and $\Ptwo(S)$ denote the sets of Borel probability measures on a metric space $S$ that are compactly supported and that have finite second moment, respectively.
Thus the trainable object is a layerwise parameter law
\(
    \eta=(\eta_t)_{t\in I}\in\Ltwotheta.
\)
Given a continuous vector field $v\colon\R^d\times\R^m\to\R^d$, we define the averaged vector field by
\begin{align}
    v_{\eta}(x)
    &\coloneqq
    \Expect_{\varTheta\sim\eta}\qty[v_\varTheta(x)]
    =\int_{\R^m}v(x,\theta)\dd{\eta(\theta)},
    &
    v_\theta(x)&\coloneqq v(x,\theta).
    \label{eq:avg_field}
\end{align}
The feature, or state, associated with an initial input $x\in\R^d$ is propagated through depth by the flow map $X_t^\eta\coloneqq X_{0,t}^\eta$, defined through the ODE starting from $s\in I$:
\begin{equation}
    X_{s,s}^\eta(x)=x,\qquad
    \dv{t}X_{s,t}^\eta(x)=v_{\eta_t}(X_{s,t}^\eta(x))
    \label{eq:forward_flow}
\end{equation}
for $t\in(s,1)$.
An explicit Euler discretization of \eqref{eq:forward_flow} in the depth variable $t$ gives a residual network \cite{He_2016_CVPR,Haber17,ChenRBD18}.
When $\eta_t$ is empirical, say
\(
    \eta_t=\frac1N\sum_{i=1}^N\delta_{\theta_i(t)},
\)
then
\(
    v_{\eta_t}(x)=\frac1N\sum_{i=1}^N v(x,\theta_i(t))
\)
is a residual layer of width $N$; general parameter laws $\eta_t$ thus describe the infinite-width limit.
% For instance, $v(x,\theta)=\theta\sigma(x)$ with a scalar activation function $\sigma$ gives
% \[
%     \dv{t}X_t^\eta(x)
%     =\int_{\R^m}\theta\sigma(X_t^\eta(x))\dd{\eta_t(\theta)}.
% \]

The induced law of the feature--label pair at depth $t$ is
\(
    \mu_t^\eta\coloneqq (X_t^\eta\times\Id_\Y)_\#\mu_0,
\)
where $T_\#\mu$ denotes the push-forward of a measure $\mu$ by a measurable map $T$, and $\Id_\Y$ is the identity map on $\Y$.
Equivalently, $\mu^\eta_\bullet$ solves the continuity equation
\begin{equation}
\left\{
    \begin{aligned}
        \partial_t\mu_t^\eta+
        \Div_x\qty(\mu_t^\eta(x,y)v_{\eta_t}(x))&=0,\\
        \mu_0^\eta&=\mu_0
    \end{aligned}
\right.
\label{eq:ODE2}
\end{equation}
for $(x,y)\in\XtimesY$ and $t\in(0,1)$.
This is the continuous DNN considered in this paper, in the same measure-valued spirit as the continuity-equation formulation of \cite{bonnet2022measure}; following \cite{isobe2023variational}, we regard $v$ as a neural network (NN).

Given a loss function $\ell\colon\XtimesY\to\R$, the terminal loss and the objective are
\begin{align}
    L(\eta)
    &\coloneqq
    \int_{\XtimesY}\ell(X_1^\eta(x),y)\dd{\mu_0(x,y)}
    =\int_{\XtimesY}\ell\dd{\mu_1^\eta},
    \label{eq:loss_term_L}\\
    J(\eta)
    &\coloneqq
    L(\eta)+\frac{\epsilon}{2}
    \int_0^1\int_{\R^m}\abs{\theta}^2\dd{\eta_t(\theta)}\dd{t},
    \label{eq:J}
\end{align}
with $\epsilon>0$, where $\abs{\bullet}$ denotes the Euclidean norm.
The first term is the population or empirical risk, i.e., the average of the loss over the data distribution, and the second term is a fixed $L^2$-regularization.
The regularization is the mean-field analogue of weight decay and confines the energy sublevels to bounded second-moment sets.
A standard choice is the squared loss \(\ell(z,y)=\nicefrac{1}{2}\abs{z-y}^2\), for which \(L(\eta)\) is the mean squared discrepancy between the terminal network output \(X_1^\eta(x)\) and the label \(y\).
We study the minimization of $J$ over $\Ltwotheta$.
The forward equation belongs to the family of mean-field continuous ResNets studied in \cite{E2020,lu20b,ding21on,ding22overparam}; the role of the fixed $L^2$-regularization is discussed in \cref{subsec:related}.

Beyond minimizers of $J$, our main interest is the long-time behavior of the training dynamics.
The metric space for the training dynamics is the $L^2$-in-depth Wasserstein space $\metricsp\coloneqq\Ltwotheta$, endowed with
\(
    W_2(\eta,\zeta)_{L^2(I)}
    \coloneqq
    (\int_0^1 W_2^2(\eta_t,\zeta_t)\dd{t})^{\nicefrac{1}{2}},
\)
where $W_2$ denotes the $2$-Wasserstein distance.
Formally, the gradient flow of $J$ is written as
\begin{equation}
    \dv{\eta}{\tau}\qty(\tau)=-\operatorname{grad}J(\eta(\tau)),
    \label{eq:intro_gradient_flow_formal}
\end{equation}
where $\tau\ge0$ is the training time.
Since $\metricsp$ is a metric space rather than a Hilbert space, we use the metric formulation by curves of maximal slope in the sense of \cite{AGS}.
A training trajectory is denoted by $\eta(\bullet)$, so that $\eta(\tau)=(\eta_t(\tau))_{t\in I}$ with the depth variable $t\in I$, and its initial point in training time by $\eta_0\coloneqq\eta(0)$.
For a functional $E$ on $\metricsp$, $\abs{\partial E}$ denotes its metric slope, whose precise definition is recalled in \cref{subsec:grad_flow}; in particular, a critical point of $J$ means a point $\overline{\eta}$ with $\abs{\partial J}(\overline{\eta})=0$.

We thus consider a curve $\eta(\bullet)\colon[0,+\infty)\to\metricsp$ satisfying the formal gradient-flow equation \eqref{eq:intro_gradient_flow_formal} in the sense of curves of maximal slope for $J$ with respect to the metric slope $\abs{\partial J}$, and study whether $\eta(\tau)$ converges in $\metricsp$ as $\tau\to+\infty$.

\subsection{Main results and assumptions}\label{subsec:intro_main_results}
Our main result asserts the full convergence of the training dynamics: the whole gradient-flow trajectory converges to a single critical point of $J$, not merely along a subsequence of training times.
The assumptions are stated below, with the parameter-Hessian condition formulated precisely in \cref{assump:curvature}; the main theorem reads as follows.

\begin{theorem}[Long-time convergence to a critical point]\label{thm:global_conv}
Let the initial datum $\eta_0$ belong to $\Ltwotheta$, and let $\eta(\bullet)\colon[0,+\infty)\to\Ltwotheta$ be the curve of maximal slope for $J$ starting from $\eta_0$.
Suppose that \cref{assump:bound_dNN,assump:initial_data} hold and that \cref{assump:curvature} holds with $E=J(\eta_0)$.
Then, there exists a critical point $\eta^\ast\in\Ltwotheta$ of $J$ such that
\(
\eta(\tau)\to\eta^\ast\text{ in }\Ltwotheta
\)
as $\tau\to+\infty$.
\end{theorem}
\begin{remark}
    The convergence rate is algebraic when $\alpha<\nicefrac{1}{2}$ and exponential when $\alpha=\nicefrac{1}{2}$, where $\alpha$ is the \L{}ojasiewicz exponent in \eqref{eq:grad_ineq_J_rough}; see~\cite[Theorem 3.27]{Hauer19}.
\end{remark}
The theorem does not require displacement convexity, a Polyak--\L{}ojasiewicz-type nondegeneracy condition, or initialization near a minimizer.

The key estimate behind \cref{thm:global_conv} is a \L{}ojasiewicz--Simon gradient inequality for the metric slope of $J$; this class of inequalities goes back to \cite{Lojasiewicz71,Lojasiewicz99,Simon83}.
In the notation of \cref{thm:grad_ineq}, for each critical point $\overline{\eta}$ there exist $C>0$, $r>0$, and $\alpha\in(0,1/2]$ such that
\begin{equation}
    \abs{J(\zeta)-J(\overline{\eta})}^{1-\alpha}
    \le C\abs{\partial J}(\zeta)
    \label{eq:grad_ineq_J_rough}
\end{equation}
whenever $W_2(\zeta,\overline{\eta})_{L^2(I)}<r$.
The precise version is proved in \cref{sec:grad_ineq}.

We also establish the two basic existence statements needed for the variational problem and its training dynamics: the existence of a minimizer of $J$ in \cref{thm:ex_minima}, and the well-posedness of the gradient flow in \cref{thm:wellposedness_grad}.

We now state the assumptions used in \cref{thm:global_conv}.
For the gradient-flow analysis and for the \L{}ojasiewicz--Simon inequality, we use the following analyticity and derivative-boundedness assumption.

\begin{assumption}[Analyticity and growth bounds]\label{assump:bound_dNN}
The loss $\ell\colon\XtimesY\to\R$ is continuous and satisfies \(\abs{\ell(x,y)}\le A+B\abs{x}^2\) for some $A$, $B>0$, and $L(\eta)\ge0$ for every $\eta\in\Ltwotheta$.
The vector field $v\colon\R^d\times\R^m\to\R^d$ is real-analytic.
Moreover, for every compact subset $K\subset\R^d$, there exists a complex neighborhood $U\subset\C^d$ of $K$ such that $\ell$ extends to a continuous function on $U\times\Y$ that is holomorphic in its first variable.
In addition, there exist constants $C_0$, $C_x>0$, $C_{\theta\theta}\ge0$, and $p\in[0,2)$ such that
\begin{subequations}\label{eq:bound_dNN}
    \begin{align}
    \norm{D_\theta^{j}v(x,\theta)}_{\opn} &{}\le C_0(1+\abs{\theta}^{(p-j)_+})(1+\abs{x}) && (j=0,1),\label{eq:Cth}\\
    \norm{D_x^{i}D_\theta^{j}v(x,\theta)}_{\opn} &{}\le C_x(1+\abs{\theta}^{2-j})(1+\abs{x})^{j} && (i\ge1,\ i+j\le2),\label{eq:Cx}\\
    \norm{D_\theta^{2}v(x,\theta)}_{\opn} &{}\le C_{\theta\theta}(1+\abs{x}^2),\label{eq:Cthth}
    \end{align}
\end{subequations}
where $D_\theta^{0}v\coloneqq v$, $(x)_+\coloneqq\max\{x,0\}$ for $x\in\R$, $\abs{\theta}^0\coloneqq1$, and $D_\theta D_x v=D_xD_\theta v$; $\norm{\bullet}_{\opn}$ denotes the operator norm of a multilinear map.
\end{assumption}

The subquadratic condition $p<2$ allows the second-moment regularization to control the growth of the state equation; it is used in the compactness argument for minimizers and in the uniform support estimate along the flow.
Analyticity is used only in the \L{}ojasiewicz--Simon argument, whereas the finite-order derivative bounds are used for the slope formula and the well-posedness of the gradient flow.
To obtain compactness of the orbit, we impose regularity of the initial curve in $t$.

\begin{assumption}[Depth regularity of the initial parameters]\label{assump:initial_data}
There exists a bounded set $D\subset\R^m$ such that $\eta_0\in W^{1,2}(I;\Ptwo(D))$, the Sobolev space in \eqref{def:Sobolev}.
\end{assumption}

This assumption is used only to establish the relative compactness of the training orbit in \cref{prop:Orbit_in_Sobolev}; comparable depth-regularity assumptions appear in related analyses of continuous-depth models~\cite{Scagliotti2023}.
For $R>0$, set
\[
    \Dcal_R
    \coloneqq
    \Set{\eta\in\Ltwotheta|\int_0^1\int_{\R^m}\abs{\theta}^2\dd{\eta_t(\theta)}\dd{t}\le R^2}.
\]
For $\eta\in\Ltwotheta$, define the costate by
\begin{equation}
    \varphi_t^\eta(x,y)
    \coloneqq
    \ell\qty(X_{t,1}^\eta(x),y).
    \label{eq:intro_costate}
\end{equation}
Its role in the gradient representation is established in \cref{prop:grad}.

\begin{assumption}[Parameter-Hessian condition]\label{assump:curvature}
Let $E>0$ and set $R_E\coloneqq\sqrt{2E/\epsilon}+1$.
Define
\begin{equation}\label{eq:coercivity_margin}
    a_R\coloneqq\epsilon-C_{\theta\theta}\adjustlimits\sup_{\eta\in\Dcal_R}\sup_{t\in I}\int_{\XtimesY}\abs{\nabla_x\varphi_t^{\eta}(x,y)}\qty(1+\abs{x}^2)\dd{\mu_t^{\eta}(x,y)}.
\end{equation}
We assume that \(a_{R_E}>0\).
\end{assumption}
% This condition requires the fixed $L^2$-regularization to dominate the contribution of the parameter Hessian $D_\theta^2v$ on the moment ball associated with the relevant energy level.
\Cref{rem:strong_convexity} below shows that, under this condition, the map $\R^m\ni\theta\mapsto\fdv{J}{\eta}\qty[\eta](t,\theta)$ is strongly convex for each fixed $\eta$, which is unrelated to displacement convexity of $J$.
It is automatic when $v$ is affine in $\theta$.
\Cref{sec:examples} gives a sufficient condition, a nonconvex example satisfying the assumption, and a nonconvergent example when it fails.

\subsection{Proof strategy: a \texorpdfstring{\L{}ojasiewicz--Simon}{Lojasiewicz--Simon} inequality}\label{subsec:intro_strategy}
Compactness of the orbit gives an $\omega$-limit point, and the \L{}ojasiewicz--Simon inequality gives convergence to that point.
The compactness follows from the contraction of the parameter characteristics at the rate of the coercivity margin, which balances the source coming from nearby layers and yields a depth-Lipschitz estimate uniform in the training time.
We now explain the proof of the gradient inequality \eqref{eq:grad_ineq_J_rough}.
The first step lifts the objective from the Wasserstein space $\metricsp$ to a Hilbert space of random variables and identifies the lifted gradient using the L-derivative calculus from mean-field optimal control \cite{Carmona2018}.
The second step establishes a \L{}ojasiewicz--Simon inequality for the lifted functional.
The classical Hilbert-space argument assumes an analytic energy with a sufficiently regular gradient and a Fredholm Hessian at the critical point; here the lifted gradient is not $C^1$ in the Hilbert topology, although its linewise second variation admits a coercive-plus-compact decomposition, and analyticity is available on the smaller space of bounded parameter fields.
These provide the key nonstandard ingredients for the nonsmooth reduction of Feireisl, Issard-Roch, and Petzeltov\'{a} \cite{FEIREISL20041}.
A finite-rank correction then reduces the inequality to the finite-dimensional \L{}ojasiewicz inequality.
Finally, the inequality is transferred back to the metric slope on $\metricsp$ through a lifting and a localization of the objective, which also underlies the well-posedness of the flow.
This is the point at which the analyticity of $v$ and $\ell$ is used, and the resulting argument replaces the usual nondegenerate-Hessian/Polyak--\L{}ojasiewicz route by a local mechanism adapted to the Wasserstein geometry.

\subsection{Related work}\label{subsec:related}

\paragraph{Regularization in continuous-depth models}
Related continuous-depth formulations use different regularizations: \cite{ding22overparam} lets the $L^2$-penalty vanish during training, whereas the fixed $L^2$-regularizations in \cite{bonnet2022measure,isobe2023variational} act on deterministic controls.
In contrast, \eqref{eq:J} keeps a fixed $L^2$-regularization on the layerwise parameter laws.

\paragraph{Convergence and fibered Wasserstein geometry}
Nonconvex gradient flows need not converge \cite{CHERIDITO2021101540,LyuL20}.
Other convergence results rely on additional structural or initialization assumptions \cite{Chizat2022,Korolev22}.
The $L^2$-in-depth Wasserstein metric used here is an instance of the fibered Wasserstein distance developed by Peszek and Poyato \cite{Peszek2023}.
In the neural-network setting, Barboni, Peyr\'{e}, and Vialard \cite{Barboni25} use the same geometry, under the name conditional optimal transport, for infinitely deep and arbitrarily wide ResNets.
Because the present objective contains a fixed $L^2$-regularization, existing global-optimality results for unregularized risks do not directly identify the limiting critical point obtained here; the case $\epsilon=0$ is outside the scope of our argument.

\paragraph{\L{}ojasiewicz--Simon inequalities}
\L{}ojasiewicz-type inequalities in a Wasserstein geometry arise for McKean--Vlasov dynamics \cite{choi2026wassersteinlojasiewiczinequalitiesasymptoticsmckeanvlasov} and in the Wasserstein gradient flow of a Coulomb discrepancy \cite{BoufadeneVialard25}; these concern different (free-energy) functionals and are related in spirit rather than directly comparable.
For real-analytic functions on $\R^d$, \L{}ojasiewicz's gradient inequality yields convergence \cite{Lojasiewicz71,Lojasiewicz99}; Simon extended it to infinite dimensions using the Fredholm property of the Hessian \cite{Simon83}, and \cite{AKAGI20192663,Feireisl2000} dispense with analyticity through elliptic-operator structure.
Sufficient conditions on Hilbert spaces \cite{CHILL2003572,HARAUX20112826,RUPP2020108708}
% ,Yagi2021} 
and abstract two-space results \cite{zbMATH07244587} are available.
These theorems do not apply directly to the present functional because of the regularity mismatch described in \cref{subsec:intro_strategy}; the nonsmooth reduction in \cref{subsec:feireisl} is designed to bridge it.
Abstract metric-space Kurdyka--\L{}ojasiewicz theory \cite{Hauer19} yields convergence once a slope inequality is available, but it does not provide directly verifiable analytic hypotheses for the present functional, and \cite{Blanchet18} requires the flow to start near a critical point.
We adapt the nonsmooth \L{}ojasiewicz--Simon reduction of Feireisl, Issard-Roch, and Petzeltov\'{a} \cite{FEIREISL20041} to the present setting.

\subsection{Organization of the paper}\label{subsec:organize}
\Cref{sec:prelim,sec:grad_repr,sec:orbit} develop the metric framework, the gradient representation, and the compactness estimates for the training dynamics.
\Cref{sec:grad_ineq} proves the \L{}ojasiewicz--Simon inequality and derives \cref{thm:global_conv}, while \cref{sec:examples} discusses examples and limitations.
\Cref{sec:calc,sec:exist_wellposed,sec:technical_proofs} collect the Wasserstein calculus, the existence and well-posedness arguments, and the technical proofs of semiconvexity and orbit compactness.

%%%%%%%%%%%%%%%%%%%%%
\section{Metric framework and well-posedness}\label{sec:prelim}
 We fix the notation and review the essential concepts in gradient flows, probability measures, and optimal control.

\paragraph{Notation}
For a Banach space $X$, $\norm{\bullet}_X$ denotes the norm of $X$ and $\Bcal(X)$ denotes the bounded linear operators on $X$.
For a function(al) $F\colon X\to Y$ from a Banach space $X$ to another Banach space $Y$, $DF(x)[\bullet]\colon X\to Y$ and $\Hessian F(x)[\bullet,\bullet]\colon X\times X\to Y$ denote the Fr\'{e}chet derivative and Hessian, respectively, of $F$ at $x\in X$.
For a random variable $\varTheta$, we will denote by $\Expect[\varTheta]$ the expected value of $\varTheta$, and by $\Law\varTheta$ the law (distribution) of $\varTheta$.
For a Banach space $X$, $x\in X$ and $r>0$, $B_X(x,r)$ denotes the open ball of radius $r$ centered at $x$.
Throughout the paper, the subscript $t$ is used to represent the ``time in a continuous DNN \eqref{eq:ODE2},'' which means a layer in DNNs, with $t\in I$.
The symbol $\Leb$ denotes the Lebesgue measure on $I$.
In addition, $\tau\geq0$ represents the time in the gradient flow \eqref{eq:intro_gradient_flow_formal}.
The symbol $C>0$ denotes a generic constant independent of $\tau\ge0$.
For $R>0$, $C_R>0$ denotes a generic constant depending only on $R$, the constants in \cref{assump:bound_dNN}, $\mu_0$, $\ell$, and $\epsilon$.

\subsection{Gradient flow on metric spaces and metric-valued function spaces}\label{subsec:grad_flow}
First, we review the notations in \cite[Chapter 1]{AGS}.
Let $(X,d)$ be a separable complete metric space, $q\in[1,+\infty]$, and $T\in(0,+\infty]$.
For an absolutely continuous curve \(u\in\AC^q(0,T;X)\) and a proper lower semicontinuous functional \(E\colon X\to(-\infty,+\infty]\), we write \(\abs{\dot u}(t)\coloneqq\lim_{s\to t}\nicefrac{d(u_s,u_t)}{\abs{s-t}}\) and \(\abs{\partial E}(u)\coloneqq\limsup_{v\to u}\nicefrac{(E(u)-E(v))_+}{d(u,v)}\) for the metric derivative and the local slope, respectively \cite[Definitions 1.1.2 and 1.2.4]{AGS}.
A point \(u\) is critical if \(\abs{\partial E}(u)=0\).

Subsequently, we overview the metric-valued Lebesgue space, the details of which are partially presented in \cite[Section 5.4]{AGS}, \cite[Section 2.3]{Lisini2007}, and \cite[Definition 3.1]{LAVENANT2019688}. 
Let
% \[
\(
    \Lcal^2(I;X)\) be the set of square-integrable functions from $I$ to $X$
%     \coloneqq\Set{u_\bullet\colon I\to X|
%     \begin{array}{c}
%         \text{$u$ is Lebesgue measurable and }\\
%         \displaystyle\int_0^1d\qty(u_t,\bar{u})^2\dd{t}<\infty\text{ for some }\bar{u}\in X.
%     \end{array}
%     },
% \]
and set
\begin{equation}
    d(u^0,u^1)_{L^2(I)}\coloneqq\qty(\int_0^1d(u^0_t,u^1_t)^2\dd{t})^{\nicefrac{1}{2}},\label{eq:d_L2_notation}
\end{equation}
for $u^0$, $u^1\in\Lcal^2(I;X)$.
The Lebesgue space $L^2(I;X)$ is the quotient space of $\Lcal^2(I;X)$ with respect to equality almost everywhere, endowed with $d(\bullet,\bullet)_{L^2(I)}$.
In particular, with $X=\Ptwo(\R^m)$ and $d=W_2$, \eqref{eq:d_L2_notation} is the metric $W_2(\bullet,\bullet)_{L^2(I)}$ on $\metricsp$ in \cref{subsec:intro_model}.

Similar to the procedure stated in \cite[Subsection 2.4]{Lisini2007}, we define the Sobolev space $W^{1,2}(I;X)$ as follows:
\begin{equation}
    W^{1,2}(I;X)\coloneqq\Set{u\in L^2(I;X)|\sup_{h\in(0,1)}\int_0^{1-h}\qty(\frac{d(u(t+h),u(t))}{h})^2\dd{t}<+\infty}.\label{def:Sobolev}
\end{equation}
The following proposition provides different characterizations of $W^{1,2}(I;X)$:
\begin{proposition}\label{prop:Sobolev}
The space $W^{1,2}(I;X)$ coincides with $\AC^2(I;X)$.
Moreover, it coincides with
\(
   \Set{u\in L^2(I;X)|\Dir(u)\coloneqq\frac{1}{2}\int_0^1\abs{\dot{u}}^2\dd{t}<+\infty}.
\)
\end{proposition}
% \begin{proof}
    The former follows from \cite[Lemma 2.5]{Lisini2007} and the latter from \cite[Proposition 3.8]{LAVENANT2019688}.
% \end{proof}

\subsection{Probability measures}\label{subsec:Wasserstein}
From now on, $\Pcal(X)$ denotes the set of regular and Borel probability measures on $X$, endowed with the narrow topology defined in \cite[Section 5.1]{AGS}.
By the Prokhorov theorem in \cite[Theorem 5.1.3]{AGS}, relative compactness with respect to the narrow topology is equivalent to tightness.
For $\mu$, $\nu\in\Pcal(X)$, $\supp\mu\subset X$ denotes the support of $\mu$, and $\Gamma(\mu,\nu)\subset\Pcal(X^2)$ denotes the transport plans between $\mu$ and $\nu$, according to \cite[Section 5.2]{AGS}.

We use the standard notation \(\Pcal_q(X)\), \(W_q\), and \(\OptPlan(\mu,\nu)\) for probability measures with finite \(q\)-moment, the \(q\)-Wasserstein distance, and optimal transport plans, respectively \cite[\S 5.1, 5.2, and 7.1]{AGS}.
We set \(m_q(\mu)\coloneqq\int_X\abs{x}^q\dd{\mu(x)}\) on a Euclidean space; thus \(m_q(\mu)=W_q^q(\mu,\delta_0)\).
The continuity equation in \eqref{eq:ODE2} is considered according to \cite[Definition 3.8]{isobe2023variational} or \cite[Section 8.1]{AGS}.
Based on \cite[Lemma 3.9]{isobe2023variational} or \cite[Proposition 8.1.8]{AGS}, the solution of \eqref{eq:ODE2} is expressed as $\mu^\eta_t=(X_{0,t}^\eta\times\Id_\Y)_\#\mu_0$, where $X^\eta$ is the flow map in \eqref{eq:forward_flow}.
The Gronwall inequality and \eqref{eq:Cth} bound the solution of \eqref{eq:forward_flow} uniformly in $t\in I$ as
\begin{equation}
    \abs{X_t^{\eta}(x)}\le \qty(\abs{x}+C_0\qty(1+\int_0^1m_p(\eta_s)\dd{s}))\exp\qty(C_0\qty(1+\int_0^1m_p(\eta_s)\dd{s})), \label{eq:Gronwall_bound_X}
\end{equation}
where $C_0$ and $p$ are the constants in \eqref{eq:Cth}.
The following lemma follows from the Gronwall bound \eqref{eq:Gronwall_bound_X} applied to $\mu^\eta_t=(X_t^\eta\times\Id_\Y)_\#\mu_0$, since $\supp\mu_0$ is compact and $\int_0^1m_p(\eta_t)\dd{t}\le1+R^2$ for $\eta\in\Dcal_R$ and $p<2$.
\begin{lemma}[Boundedness of supports]\label{lem:supp_bound}
For every $R>0$, the supports $\supp\mu^\eta_t$ of the solution $\mu^\eta$ of \eqref{eq:ODE2} are bounded uniformly in $t\in I$ and $\eta\in\Dcal_R$.
\end{lemma}

\subsection{Differential formulas in optimal control theory}
In connection with the ODE \eqref{eq:forward_flow}, consider the following probability-theoretic formulation
\[
    X_{0}^{\Law\varTheta}(x)=x,\quad\dv{t}X_{t}^{\Law\varTheta}(x)=\Expect\qty[v\qty(X_{t}^{\Law\varTheta}(x),\varTheta_t)],
\]
where $\varTheta\in\varThetaSpace$ and $\Omega$ is a probability space introduced in \cref{subsec:Lcalculus}.
Note that $X^{\Law\varTheta}$ depends only on the laws of the random variables $\varTheta_t$. 
In \cref{sec:grad_ineq}, $X_1^{\Law\varTheta}$ is differentiated with respect to $\varTheta$. 
The optimal control theory in \cite{bressan2007} yields an explicit formula for the derivative.

\begin{lemma}[A version of {\cite[Theorem 3.2.6]{bressan2007}}]\label{lem:linearizedODE}
Under \cref{assump:bound_dNN}, we have
\begin{equation}
    \eval{\dv{\varepsilon}X_t^{\Law\qty(\varTheta+\varepsilon\varDelta)}(x)}_{\varepsilon=0}=\int_0^t\Expect\qty[M^{\Law\varTheta}(x;t,u)\qty(\mdif{\theta} v)(X_u^{\Law\varTheta}(x),\varTheta_u)\qty[\varDelta_u]]\dd{u},\label{eq:differential_formula}
\end{equation}
   for each $t\in I$ and $\varDelta\in\varThetaSpace$.
   Here, $M^{\Law\varTheta}(x;t_0,t_1)\in\R^{d\times d}$ is the fundamental solution of the linearized ODE $\dot{z}_t=\Expect[(\mdif{x}v)(X_t^{\Law\varTheta}(x),\varTheta_t)[z_t]]$, mapping data at $t_1$ to $t_0$.
\end{lemma}
\begin{proof}
    Note that from the Cauchy--Schwarz inequality, the integrand on the right-hand side of \eqref{eq:differential_formula} is integrable, and the expectation is differentiated under the integral sign by \eqref{eq:Cth} and dominated convergence.
    The rest of the proof is the same as that in \cite[Theorem 3.2.6]{bressan2007}.
\end{proof}
Note that the fundamental solution $M^\eta(x;t_0,t_1)$ associated with $\eta\colon I\to\Ptwo(\R^m)$ solves the linear ODE $\dot{z}_t=(\int(\mdif{x}v)(X_t^\eta(x),\theta)\dd{\eta_t(\theta)})z_t$, whose coefficient satisfies $\norm{\int(\mdif{x}v)(X_t^\eta(x),\theta)\dd{\eta_t(\theta)}}\le C_x(1+m_2(\eta_t))$ by \eqref{eq:Cx}.
Therefore, we have
\begin{equation}
    \norm{M^\eta(x;t_0,t_1)}_{\opn}\le \exp\qty(C_x\qty(1+\int_0^1 m_2(\eta_s)\dd{s})). \label{eq:Gronwall_bound_M}
\end{equation}
For a point $z$ at a layer $t$, the same notation $M^{\eta}(z;1,t)$ denotes the fundamental solution of the linearization along $X_{t,\bullet}^{\eta}(z)$, which obeys the same bound \eqref{eq:Gronwall_bound_M}.

The differential calculus on $\Ltwotheta$---the lifting to random variables, the L-derivative, and the Wasserstein derivative---is developed in \cref{sec:calc}; the consequences used in the main text are the explicit L-derivative of the objective and the minimal-selection representation of $\abs{\partial J}$, established in \cref{prop:grad,cor:slope}.

\subsection{Localized semiconvexity and well-posedness of the gradient flow}\label{subsec:wellposed}
Recall the moment ball $\Dcal_R$ of \cref{subsec:intro_main_results}.
Since \(\{J\le E\}\subset\Dcal_{\sqrt{2E/\epsilon}}\), every gradient-flow trajectory remains in a bounded moment set.
We establish semiconvexity of \(J\) on each \(\Dcal_R\); we therefore apply the gradient-flow theory to the localized functional \(J_R\coloneqq J+\iota_{\Dcal_R}\) and then transfer the resulting flow
back to \(J\).
The generalized geodesics and the corresponding notion of \(\lambda\)-convexity are recalled in \cref{def:generalized_convexity}.

\begin{lemma}[$\lambda$-convexity of $L$ along generalized geodesics]\label{lem:conv_J}
    For every $R>0$, there exists a number $\lambda_R\in\R$ such that the functional $L$ defined in \eqref{eq:loss_term_L} is $\lambda_R$-convex along every generalized geodesic defined in \cref{def:generalized_convexity} whose $\eta^1$ and $\eta^2$ lie in $\Dcal_R$.
\end{lemma}
The proof, based on stability estimates for the state and costate along generalized geodesics, is given in \cref{appendix:lambda_convex}.

For $R>0$, define $J_R\coloneqq J+\iota_{\Dcal_R}$, where $\iota_{\Dcal_R}$ is zero on $\Dcal_R$ and $+\infty$ otherwise.
\begin{corollary}[Localized semiconvexity]\label{cor:conv_JR}
    For every $R>0$, the functional $J_R$ is proper, lower semicontinuous, coercive and $\lambda_R^J$-convex along generalized geodesics with $\lambda_R^J\coloneqq\lambda_R+\epsilon\in\R$.
\end{corollary}

\begin{lemma}[Localization]\label{lem:localization}
    Let $R>0$, and let $\eta\in\Dcal_R$ satisfy the strict inequality $\int_0^1m_2(\eta_t)\dd{t}<R^2$.
    Then $J_R=J$ on a ball centered at $\eta$ in $\Ltwotheta$, and consequently $\abs{\partial J_R}(\eta)=\abs{\partial J}(\eta)$ and $\partial^-J_R(\eta)=\partial^-J(\eta)$.
\end{lemma}

\begin{theorem}[Well-posedness and stability of the gradient flow]\label{thm:wellposedness_grad}
    Suppose that \cref{assump:bound_dNN} holds.
    For each $\eta_0\in\Ltwotheta$, there exists a unique curve of maximal slope $\eta(\bullet)\colon[0,+\infty)\to\Ltwotheta$ for $J$ with respect to $\abs{\partial J}$ such that $\eta(0)=\eta_0$.
    Moreover, if $\eta^1$ and $\eta^2$ are the curves of maximal slope starting from $\eta_0^1$ and $\eta_0^2$, respectively, then for every $R>\max_{i=1,2}\sqrt{2J(\eta_0^i)/\epsilon}$ and $\tau\ge0$,
    \[
        W_2\qty(\eta^1(\tau),\eta^2(\tau))_{L^2(I)}\le\e^{-\lambda_R^J\tau}W_2\qty(\eta_0^1,\eta_0^2)_{L^2(I)}.
    \]
\end{theorem}

The proofs of \cref{cor:conv_JR}, \cref{lem:localization}, and \cref{thm:wellposedness_grad} are given in \cref{subsec:wellposed_proofs}.

%%%%%%%%%%%%%%%%%%%%%%%%%%%%%%%%%%%%%%%%%%%%%%%%
\section{Gradient representation}\label{sec:grad_repr}
%%%%%%%%%%%%%%%%%
In this section, we derive an explicit representation of the linear functional derivative of $L$, following \cite{baravdish2022learningCG,bonnet2022measure}; it is expressed through the costate, or adjoint state, of the underlying optimal control problem.
We first define the linear functional derivative on $\Ltwotheta$; the related calculus is in \cref{sec:calc}.
\begin{definition}[Functional derivative]\label{def:fdv}
    A function $E\colon\Ltwotheta\to\R$ is said to admit a \emph{linear functional derivative} if there exists a function
    \[
        \fdv{E}{\eta}\colon\Ltwotheta\times I\times\R^m\ni(\eta,t,\theta)\longmapsto\fdv{E}{\eta}\qty[\eta]\qty(t,\theta)\in\R,
    \]
    called an \emph{admissible representative}, satisfying the following conditions:
    \begin{myenum}
        \item For every $\eta$, $\eta^\prime\in\Ltwotheta$ and a.e.~$t\in I$, the map $I\times\R^m\ni(r,\theta)\mapsto\fdv{E}{\eta}\qty[(1-r)\eta+r\eta^\prime](t,\theta)$ is continuous.\label{enum:1}
        \item For each $\eta\in\Ltwotheta$ and $\theta\in\R^m$, $\fdv{E}{\eta}[\eta](\bullet,\theta)\colon I\to\R$ is measurable.\label{enum:2}
        \item For any bounded subset $K\subset\Ltwotheta$, there exists $C_K>0$ such that $\abs{\fdv{E}{\eta}[\eta](t,\theta)}\le C_K(1+\abs{\theta}^2)$ for a.e.~$t\in I$, every $\eta\in K$, and every $\theta\in\R^m$.\label{enum:3}
        \item For any $\eta$, $\eta^\prime\in\Ltwotheta$, it holds that
        \begin{equation}\label{eq:Taylor_def}
            E(\eta^\prime)-E(\eta)=\int_0^1\int_0^1\int_{\R^m}\fdv{E}{\eta}\qty[(1-r) \eta + r \eta^\prime](t,\theta)\dd{\qty(\eta^\prime_t-\eta_t)(\theta)}\dd{t}\dd{r}.
        \end{equation}\label{enum:4}
    \end{myenum}
\end{definition}
The identity \eqref{eq:Taylor_def} determines an admissible representative only up to an additive function of the depth variable: we identify two admissible representatives $G_E$ and $\widehat{G}_E$ if, for every $\eta\in\Ltwotheta$, there exists a measurable function $c_\eta\colon I\to\R$ such that $\widehat{G}_E[\eta](t,\theta)=G_E[\eta](t,\theta)+c_\eta(t)$ for a.e.~$t\in I$ and every $\theta\in\R^m$.
Any two admissible representatives are equivalent in this sense by the uniqueness argument for linear functional derivatives, localized to measurable subsets of $I$; cf.~\cite[Section 5.4]{Carmona2018}.
All statements below are independent of this choice.

\begin{proposition}[Representation formula for gradients]\label{prop:grad}
Under \cref{assump:bound_dNN}, an admissible representative of the linear functional derivative of $L$ defined in \eqref{eq:loss_term_L} is given, for a.e.~$t\in I$ and every $\theta\in\R^m$, by
\[
    \fdv{L}{\eta}\qty[\eta](t,\theta)=\la\nabla_x\varphi_t^\eta,v_\theta\ra_{\mu_t^\eta}\coloneqq\int_{\XtimesY}\la\nabla_x\varphi_t^\eta(x,y),v(x,\theta)\ra\dd{\mu_t^\eta(x,y)},
\]
where the costate $\varphi^\eta$ is defined by \eqref{eq:intro_costate}; equivalently, it solves the transport equation $\partial_t\varphi+\nabla_x\varphi\cdot\int_{\R^m}v_\theta\dd{\eta_t(\theta)}=0$ with $\varphi_{t=1}=\ell$.
\end{proposition}
\begin{proof}
    Set $\eta^r\coloneqq(1-r)\eta+r\eta^\prime$ for $\eta$, $\eta^\prime\in\Ltwotheta$ and $r\in I$.
    By the adjoint method, solving the backward transport equation $\partial_t\varphi+\nabla_x\varphi\cdot v_{\eta^r_t}=0$ with $\varphi_{t=1}=\ell$ along characteristics as in \cite[Theorem 4 and Proposition 3]{chertovskih2023nonlocal}, and using $v_{\eta^r_t}=(1-r)v_{\eta_t}+rv_{\eta^\prime_t}$, we obtain
    \[
      \dv{r}L(\eta^r)=\int_0^1\la\nabla_x\varphi_t^{\eta^r},v_{\eta^\prime_t}-v_{\eta_t}\ra_{\mu_t^{\eta^r}}\dd{t}=\int_0^1\int_{\R^m}\la\nabla_x\varphi_t^{\eta^r},v_\theta\ra_{\mu_t^{\eta^r}}\dd{(\eta^\prime_t-\eta_t)(\theta)}\dd{t}.
    \]
    The map $r\mapsto L(\eta^r)$ is of class $C^1$ by \cref{assump:bound_dNN} and the bounds \cref{eq:Gronwall_bound_X,eq:Gronwall_bound_M}, and integrating over $r\in I$ yields \eqref{eq:Taylor_def}; the remaining measurability, continuity, and growth conditions follow from the same bounds.
\end{proof}
\Cref{prop:grad,prop:Ldiff_fdv}, applied to $L$ together with the regularization term, show that $J$ is continuously L-differentiable, so \cref{prop:L=Wdiff} applies to $J$; for $R>(\int_0^1m_2(\eta_t)\dd{t})^{\nicefrac{1}{2}}$, \cref{lem:localization} and \cref{lem:minimal_select} applied to $J_R$ yield the following.
\begin{corollary}[L-derivative and slope representation]\label{cor:slope}
    Under \cref{assump:bound_dNN}, $J$ is continuously L-differentiable with, for a.e.~$t\in I$ and every $\theta\in\R^m$,
    \[
        \nabla J\qty[\eta](t,\theta)=\epsilon\theta+\int_{\XtimesY}(\mdif{\theta}v)(x,\theta)^\top\nabla_x\varphi_t^{\eta}(x,y)\dd{\mu_t^{\eta}(x,y)}.
    \]
    Moreover, for every $\eta\in\Ltwotheta$, the subdifferential is the singleton $\partial^-J(\eta)=\{\nabla J\qty[\eta]\}$, and the metric slope is $\abs{\partial J}(\eta)=\norm{\nabla J\qty[\eta]}_{L^2(I;L^2(\eta_t))}$.
\end{corollary}

Thanks to \cref{cor:slope}, we obtain an a priori estimate for critical points of $J$.
\begin{lemma}[$L^\infty$ estimate of critical points]\label{lem:Linf_crit_pt}
    Under \cref{assump:bound_dNN}, for every $E>0$ there exists $R>0$ such that every critical point $\eta\in\metricsp$ of $J$ with $J(\eta)\le E$ satisfies $\supp\eta_t\subset\Set{\theta\in\R^m|\abs{\theta}\le R}$ for a.e.~$t\in I$.
\end{lemma}
\begin{proof}
Let $\eta\in\metricsp$ be a critical point of $J$ with $J(\eta)\le E$. By \cref{cor:slope}, $\norm{\nabla J\qty[\eta]}_{L^2(I;L^2(\eta_t))}=\abs{\partial J}(\eta)=0$, so, for $\dd{\eta_t\dd t}$-a.e.~$(t,\theta)$,
\begin{equation}\label{eq:criticality}
    0=\epsilon\theta+\int_{\XtimesY}(\mdif{\theta}v)(x,\theta)^\top\nabla_x\varphi_t^{\eta}(x,y)\dd{\mu_t^{\eta}(x,y)}.
\end{equation}
For a.e.~$t\in I$, \eqref{eq:criticality} holds for $\eta_t$-a.e.~$\theta$, and hence for every $\theta\in\supp\eta_t$ by continuity.
Since $J(\eta)\le E$ gives $\int_0^1m_2(\eta_t)\dd{t}\le 2E/\epsilon$, the supports $\supp\mu_t^\eta$ are bounded by a constant depending only on $E$ through \cref{lem:supp_bound}, and so is the costate gradient $\nabla_x\varphi_t^{\eta}(x,y)=M^\eta(x;1,t)^\top\nabla_x\ell(X_{t,1}^\eta(x),y)$ from \cref{prop:grad}, by the fundamental-solution bound \eqref{eq:Gronwall_bound_M} and the boundedness of $\nabla_x\ell$ on the reachable compact set from \cref{assump:bound_dNN}; together with $\abs{(\mdif{\theta}v)(x,\theta)}\le C_0(1+\abs{\theta}^{(p-1)_+})(1+\abs{x})$ from \eqref{eq:Cth}, \eqref{eq:criticality} yields
\(
    \epsilon\abs{\theta}\le C\qty(1+\abs{\theta}^{(p-1)_+}),
\)
where $C$ depends only on $E$ and $p\in[0,2)$ is the exponent in \cref{assump:bound_dNN}.
Since $(p-1)_+<1$, this forces $\abs{\theta}\le R$ with $R\coloneqq\max\{1,(2C/\epsilon)^{1/(1-(p-1)_+)}\}$, so $\supp\eta_t\subset\Set{\theta\in\R^m|\abs{\theta}\le R}$ for a.e.~$t\in I$.
\end{proof}
    
\begin{remark}[Interpretation of \cref{assump:curvature}]\label{rem:strong_convexity}
    Differentiating the formula in \cref{cor:slope} with respect to \(\theta\) and using \eqref{eq:Cthth} gives
    \(
        \Hessian_\theta\fdv{J}{\eta}[\eta](t,\theta)\succeq a_R\Id_m
    \)
    for \(\eta\in\Dcal_R\),  \(\theta\in\R^m\), and a.e.~\(t\in I\).
    Thus, under \cref{assump:curvature}, the functional derivative is \(a_{R_E}\)-strongly convex in \(\theta\), or, equivalently, its gradient is \(a_{R_E}\)-strongly monotone.
    This property yields the contraction of the characteristics in \cref{sec:orbit} and the principal coercivity in \cref{prop:H2}, but it does not imply displacement convexity of \(J\), because the dependence of the costate on \(\eta\) remains uncontrolled; see \cref{ex:nonconvex}.
\end{remark}

\section{Boundedness and compactness of the orbit}\label{sec:orbit}
We establish uniform support and regularity bounds for the gradient flow of \cref{thm:wellposedness_grad}.
The support bound uses only \cref{assump:bound_dNN} and a bounded initial support, whereas orbit compactness also requires \cref{assump:initial_data,assump:curvature}.
First, we discuss the supports.
\begin{proposition}[Uniform support bound along the gradient flow]\label{prop:orbit_in_L^infty}
    Let $\eta(\bullet)\in\AC^2(0,+\infty;\Ltwotheta)$ be the curve of maximal slope for $J$ with respect to $\abs{\partial J}$.
    Suppose that \cref{assump:bound_dNN} holds and that there exists a bounded set $B\subset\R^m$ such that $\supp\eta(0)_t\subset B$ for a.e.~$t\in I$.
    Then there exists another bounded set $D\subset\R^m$ such that $\supp\eta(\tau)_t\subset D$ for a.e.~$t\in I$ and $\tau\geq0$.
\end{proposition}
\begin{proof}
    From \cref{lem:maximal_slope=grad_flow} applied to $J_R$ with $R>\sqrt{2J(\eta_0)/\epsilon}$, together with \cref{lem:localization,cor:slope}, we see that the curve of maximal slope $\eta$ satisfies the following continuity equation
    \begin{align}
        &\partial_\tau\eta(\tau)+\Div_\theta(u\qty[\eta(\tau)]\eta(\tau))=0,\quad\tau\in\qty(0,+\infty),
    \end{align}
    in the sense of distributions, where
    % \begin{align}
    \(
        u\qty[\eta(\tau)](t,\theta)\coloneqq-\qty(\nabla_\theta\fdv{L}{\eta}\qty[\eta(\tau)](t,\theta)+\epsilon\theta)\in\R^m\),
        % \\
        % &
        \(\fdv{L}{\eta}\qty[\eta(\tau)](t,\theta)=\la\nabla_x\varphi^{\eta(\tau)}_t,v_{\theta}\ra_{\mu_t^{\eta(\tau)}}\).
    % \end{align}
    Note that $u$ is a nonlocal vector field. 
    Under the identification of $\eta(\tau)$ with the measure $\int\eta(\tau)_t\dd{t}$ on $I\times\R^m$, the curve $\eta(\bullet)$ is narrowly continuous by \cref{lem:eqiv_P_2}, and it solves the continuity equation above with the Borel vector field $g\qty[\eta(\tau)]\coloneqq(0,u\qty[\eta(\tau)])\in\R\times\R^m$, which is square-integrable against $\int\eta(\tau)_t\dd{t}$ locally in $\tau$ by \cref{def:grad_flow_continuity}.
    Hence the superposition principle \cite[Theorem 8.2.1]{AGS}, applied on finite training-time intervals, represents $\eta(\bullet)$ by integral curves of $g$, along each of which the coordinate $t$ is constant.
    Moreover, for a.e.~$t\in I$, the time-dependent vector field $(\tau,\theta)\mapsto u\qty[\eta(\tau)](t,\theta)$ is Borel measurable in $\tau$ and globally Lipschitz in $\theta$ with a locally integrable Lipschitz bound and at most linear growth, by \cref{assump:bound_dNN,lem:supp_bound} and the costate bounds in the proof of \cref{lem:Linf_crit_pt}; hence, for fixed $t$, the characteristic equation $\dv{\tau}\theta(\tau)=u\qty[\eta(\tau)](t,\theta(\tau))$ admits a unique Carath\'{e}odory solution.
    Disintegrating the superposition measure with respect to the initial point and using this uniqueness, we obtain
    \begin{equation}
        \eta(\tau)=\Phi_\tau\qty[\eta_0]_\#\eta_0,\label{eq:rep_for_nonlocal}
    \end{equation}
    for all $\tau\geq0$, where $\Phi_\tau[\eta_0]\colon I\times B\to I\times\R^m$, $\tau\geq0$, is the flow map associated with the family of vector fields $g$, i.e., $\Phi_\tau[\eta_0]$ satisfies the following equation
    \begin{equation}
        \Phi_\tau\qty[\eta_0](t,\theta)=(t,\theta)+\int_0^\tau g\qty[\eta(\sigma)]\qty(\Phi_\sigma\qty[\eta_0](t,\theta))\dd{\sigma},\label{eq:integral_eq_nonlocal}
    \end{equation}
    for $(t,\theta)\in I\times B$.
    Henceforth, $\Phi^{\R^m}_\tau[\eta_0](t,\theta)$ denotes the projection of $\Phi_\tau[\eta_0](t,\theta)$ from $I\times\R^m$ onto $\R^m$.
    Since 
    \[
        \dv{\Phi^{\R^m}_\tau\qty[\eta_0](t,\theta)}{\tau}=-\epsilon\Phi^{\R^m}_\tau\qty[\eta_0](t,\theta)-\nabla_\theta\fdv{L}{\eta}\qty[\eta(\tau)](\Phi_\tau\qty[\eta_0](t,\theta)),
    \]
    we find that
    \begin{align}
        \dv{\tau}\abs{\Phi^{\R^m}_\tau\qty[\eta_0](t,\theta)}^2&{}=-2\epsilon\abs{\Phi^{\R^m}_\tau\qty[\eta_0](t,\theta)}^2-2\Phi^{\R^m}_\tau\qty[\eta_0](t,\theta)\cdot\nabla_\theta\fdv{L}{\eta}\qty[\eta(\tau)](\Phi_\tau\qty[\eta_0](t,\theta)),\label{eq:similar_comp_to_scagliotti}
    \end{align}
    by differentiating the characteristic flow with respect to $\tau$ and using the chain rule for the gradient-flow vector field, as in the proof of \cite[Lemma 4.5]{Scagliotti2023}.
    The second term on the right-hand side of \eqref{eq:similar_comp_to_scagliotti} is bounded as follows:
    \[
    \begin{aligned}
        &{}\abs{\Phi^{\R^m}_\tau\qty[\eta_0](t,\theta)\cdot\nabla_\theta\fdv{L}{\eta}\qty[\eta(\tau)](\Phi_\tau\qty[\eta_0](t,\theta))}\\
        {}\leq{}&C\abs{\Phi^{\R^m}_\tau\qty[\eta_0](t,\theta)}\sup_{(x,y)\in B_{\XtimesY}}\abs{\qty(\mdif{\theta} v)\qty(x,\Phi^{\R^m}_\tau\qty[\eta_0](t,\theta))}\\
        {}\leq{}& C\qty(1+\abs{\Phi^{\R^m}_\tau\qty[\eta_0](t,\theta)}^{1+(p-1)_+}),
    \end{aligned}
    \]
    by \cref{assump:bound_dNN,lem:supp_bound}, where $B_{\XtimesY}$ is a ball containing the supports $\supp\mu_t^{\eta(\tau)}$ for all $t\in I$ and $\tau\ge0$, and the constant $C$ bounds the costate $\nabla_x\varphi^{\eta(\tau)}$ on $B_{\XtimesY}$ as in the proof of \cref{lem:Linf_crit_pt}. Since $1+(p-1)_+<2$, the Young inequality yields constants $C_1$, $C_2>0$ independent of $t$, $\theta$ and $\tau$ such that
   \[
    \dv{\tau}\abs{\Phi^{\R^m}_\tau\qty[\eta_0](t,\theta)}^2\leq-C_1\abs{\Phi^{\R^m}_\tau\qty[\eta_0](t,\theta)}^2+C_2,
   \]
   and the comparison principle gives
   \(
   % \begin{aligned}
    % &
    \abs{\Phi^{\R^m}_\tau\qty[\eta_0](t,\theta)}^2\le\e^{-C_1\tau}\abs{\theta}^2+\nicefrac{C_2}{C_1}\qty(1-\e^{-C_1\tau})\leq\max\{\abs{\theta}^2,\nicefrac{C_2}{C_1}\},
   % \end{aligned}
   \)
   which is bounded uniformly in $t$, $\theta\in\supp\eta(0)_t$ and $\tau\ge0$ by the boundedness of $B$.
   This boundedness and \eqref{eq:rep_for_nonlocal} imply the desired conclusion.
\end{proof}

Subsequently, we prove the relative compactness of the orbit of the gradient flow.
\begin{proposition}[Depth regularity and relative compactness of the orbit]\label{prop:Orbit_in_Sobolev}
    Let $\eta(\bullet)\colon[0,+\infty)\to\Ltwotheta$ be the curve of maximal slope for $J$ starting from $\eta_0$.
    Suppose that \cref{assump:bound_dNN,assump:initial_data} hold and that \cref{assump:curvature} holds with $E\coloneqq J(\eta_0)$, and set $a\coloneqq a_{R_E}>0$.
    Then there exist a bounded set $D\subset\R^m$ and a constant $C>0$ such that
    \(
        W_2\qty(\eta(\tau)_t,\eta(\tau)_s)\le\e^{-a\tau}W_2\qty((\eta_0)_t,(\eta_0)_s)+\frac{C}{a}\abs{t-s}
    \)
    for every $\tau\ge0$ and a.e.~$(s,t)\in I^2$.
    In particular, $\eta(\tau)\in W^{1,2}(I;\Ptwo(D))$ with $\sup_{\tau\ge0}\Dir(\eta(\tau))<+\infty$, and the orbit $\Set{\eta(\tau)|\tau\ge0}$ is relatively compact in $\Ltwotheta$.
\end{proposition}
The estimate combines the same-layer contraction induced by the coercivity margin \(a>0\) with an \(O(\abs{t-s})\) bound for the layer-dependent source; see \cref{sec:compactness}.
% The proof splits the depthwise distance along the training flow.
% The coercivity margin $a$ makes the parameter characteristics $\Phi^{\R^m}_\tau\qty[\eta_0]$ of \eqref{eq:rep_for_nonlocal} contract,
% \[
%     \dv{\tau}\abs{\Phi^{\R^m}_\tau\qty[\eta_0](t,\theta^1)-\Phi^{\R^m}_\tau\qty[\eta_0](t,\theta^2)}^2\le-2a\abs{\Phi^{\R^m}_\tau\qty[\eta_0](t,\theta^1)-\Phi^{\R^m}_\tau\qty[\eta_0](t,\theta^2)}^2
% \]
% for $\theta^1$, $\theta^2$ in a common bounded set containing the supports $\supp\eta(0)_t$ for a.e.~$t\in I$, so the depth variation inherited from the initial datum decays at rate $\e^{-a\tau}$, while the characteristics at two nearby layers differ by a source of size $C\abs{t-s}$, which yields the Lipschitz term $\nicefrac{C}{a}\abs{t-s}$.
% This yields a bound on $\Dir(\eta(\tau))$ uniform in $\tau$, and the relative compactness of the orbit follows from the Rellich theorem.
% The detailed proof is postponed to \cref{sec:compactness}.

%%%%%%%%%%%%%%%%%
\section{\texorpdfstring{\L{}ojasiewicz--Simon}{Lojasiewicz--Simon} inequality and long-time convergence}\label{sec:grad_ineq}
%%%%%%%%%%%%%%%%%
Throughout this section, $K_X\coloneqq\pi_{\R^d}(\supp\mu_0)$ denotes the projection of $\supp\mu_0$ onto $\R^d$, a compact set, and we work under \cref{assump:curvature} at an energy level $E>0$.
The supremum defining $a_R$ in \eqref{eq:coercivity_margin} is finite for every $R>0$ by \cref{lem:supp_bound} and the costate bound in the proof of \cref{lem:Linf_crit_pt}, and $R_1\le R_2$ implies $a_{R_1}\ge a_{R_2}$.
We first prove the following \L{}ojasiewicz--Simon inequality.

\begin{theorem}[\L{}ojasiewicz--Simon gradient inequality for $J$ in \eqref{eq:J}]\label{thm:grad_ineq}
Let $E>0$, let $\overline{\eta}\in\Ltwotheta$ be a critical point of $J$ with $J(\overline{\eta})\le E$, and suppose that \cref{assump:bound_dNN,assump:curvature} hold.
Then there exist $\alpha\in(0,1/2]$, $C>0$ and $r>0$ such that
    \[
        \abs{J(\zeta)-J(\overline{\eta})}^{1-\alpha}\leq C\abs{\partial J}(\zeta),
    \]
    for all $\zeta\in \Ltwotheta$ with $W_2(\zeta,\overline{\eta})_{L^2(I)}< r$.
\end{theorem}
\Cref{assump:curvature} does not imply convexity and cannot in general be omitted; see \cref{sec:examples}.

\subsection{Lifting the objective}
Following the notation of \cref{subsec:Lcalculus}, we write $\Hilbertsp\coloneqq\varThetaSpace$ and $\strongsp\coloneqq\varThetaSpaceInf$ throughout this section, where $\Omega=I\times I$ carries the product Lebesgue measure $\mathbb{P}=\Leb\otimes\Leb$, and let $\widetilde{J}$ be the lifting of $J$, i.e., define $\widetilde{J}\colon\Hilbertsp\to\R_{\geq0}$ by
\begin{equation}
    \widetilde{J}(\varTheta)\coloneqq J(\Law\varTheta)
    =\int_{\XtimesY}\ell\dd{\mu^{\Law\varTheta}_1}+\frac{\epsilon}{2}\norm{\varTheta}^2_{\Hilbertsp},\label{eq:def_Jtil}
\end{equation}
for $\varTheta\in\Hilbertsp$.
Herein, we will prove that the following gradient inequality holds for $\widetilde{J}$.
\begin{theorem}[\L{}ojasiewicz--Simon gradient inequality for $\widetilde{J}$]\label{thm:grad_ineq_Jtilde}
Let $E>0$ and suppose that \cref{assump:bound_dNN,assump:curvature} hold.
For any critical point $\varTheta\in\strongsp$ of $\widetilde{J}$ with $\widetilde{J}(\varTheta)\le E$, there exist $C>0$, $r>0$ and $\alpha\in(0,1/2]$ such that
\begin{equation}
         \abs{\widetilde{J}(\varPhi)-\widetilde{J}({\varTheta})}^{1-\alpha}\leq C\norm{D\widetilde{J}(\varPhi)}_{\Hilbertsp},\label{eq:grad_ineq_Jtil}
\end{equation}
for all $\varPhi\in \Hilbertsp$ with $\norm{\varPhi-{\varTheta}}_{\Hilbertsp}<r$.
\end{theorem}
\begin{lemma}[Fixed-center realization of measurable couplings]\label{lem:fixed_center}
    Let $\overline{\eta}$, $\zeta\in\Ltwotheta$ and choose Borel representatives of $t\mapsto\overline{\eta}_t$ and $t\mapsto\zeta_t$.
    By \cref{lem:existence_lift} applied to the first coordinate, there is a Borel map $\vartheta\colon I\times I\to\R^m$, depending only on $\overline{\eta}$, with $\vartheta(t,\bullet)_\#\Leb=\overline{\eta}_t$ for a.e.~$t$; set $\varTheta_t(u,v)\coloneqq\vartheta(t,u)$, so that $\Law\varTheta_t=\overline{\eta}_t$.
    Then for every Borel family $t\mapsto\gamma_t\in\Gamma(\overline{\eta}_t,\zeta_t)$ there exists $\varPhi\in\Hilbertsp$ with $\Law(\varTheta_t,\varPhi_t)=\gamma_t$ for a.e.~$t$; in particular $\Law\varPhi_t=\zeta_t$.
    If each $\gamma_t$ is optimal, then $\norm{\varPhi-\varTheta}_{\Hilbertsp}^2=\int_0^1 W_2(\overline{\eta}_t,\zeta_t)^2\dd{t}$.
    The center $\varTheta$ depends only on $\overline{\eta}$, not on $\zeta$ or $\gamma$.
\end{lemma}
\begin{proof}
    Disintegrating $\dd{t}\gamma_t$ with respect to its $(t,\theta)$-marginal $\dd{t}\overline{\eta}_t$ by \cite[Theorem 5.3.1]{AGS} gives a Borel kernel $(t,\theta)\mapsto\kappa(t,\theta,\bullet)$ with $\gamma_t(\dd\theta,\dd\varphi)=\overline{\eta}_t(\dd\theta)\kappa(t,\theta,\dd\varphi)$ for $\dd{t}\overline{\eta}_t(\dd\theta)$-a.e.\ $(t,\theta)$.
    By the kernel representation \cite[Lemma 4.22]{Kallenberg2021} there is a Borel map $F\colon I\times\R^m\times I\to\R^m$ with $F(t,\theta,\bullet)_\#\Leb=\kappa(t,\theta,\bullet)$ for $\dd{t}\overline{\eta}_t(\dd\theta)$-a.e.\ $(t,\theta)$, and we set $\varPhi_t(u,v)\coloneqq F(t,\vartheta(t,u),v)$.
    For a bounded Borel function $f\colon\R^m\times\R^m\to\R$,
    \begin{align}
        \Expect\qty[f(\varTheta_t,\varPhi_t)]&=\int_0^1\int_0^1 f\qty(\vartheta(t,u),F(t,\vartheta(t,u),v))\dd{v\dd u}\\
        &=\int_{\R^m}\int_{\R^m}f(\theta,\varphi)~\kappa(t,\theta,\dd\varphi)\overline{\eta}_t(\dd\theta)=\int f\dd{\gamma_t},
    \end{align}
    so $\Law(\varTheta_t,\varPhi_t)=\gamma_t$; applying this identity to $f_N(\theta,\varphi)=\min\{\abs{\varphi}^2,N\}$ and letting $N\to\infty$ by monotone convergence gives $\varPhi\in\Hilbertsp$.
    When each $\gamma_t$ is optimal, the same truncation of $\abs{\theta-\varphi}^2$ gives $\norm{\varPhi-\varTheta}_{\Hilbertsp}^2=\int_0^1\int\abs{\theta-\varphi}^2\dd{\gamma_t}\dd{t}=\int_0^1 W_2(\overline{\eta}_t,\zeta_t)^2\dd{t}$.
\end{proof}

Assuming \cref{thm:grad_ineq_Jtilde} for the moment, we prove \cref{thm:grad_ineq} using the calculus established in \cref{sec:calc}.
\begin{proof}[Proof of \cref{thm:grad_ineq}]
    Let $r>0$, $C>0$ and $\alpha\in(0,1/2]$ be the constants of \cref{thm:grad_ineq_Jtilde} at the critical point produced below.
    By \cref{lem:Linf_crit_pt}, there is $R_0>0$ with $\supp\overline{\eta}_t\subset\Set{\theta\in\R^m|\abs{\theta}\le R_0}$ for a.e.~$t$; since the center $\vartheta(t,\bullet)$ of \cref{lem:fixed_center} for $\overline{\eta}$ realizes $\overline{\eta}_t$, its values lie in $\supp\overline{\eta}_t$ for a.e.~$(t,u)$, so $\varTheta_\ast\coloneqq\varTheta$ lies in $\strongsp$ with $\Law\varTheta_{\ast,t}=\overline{\eta}_t$.
    By \cref{cor:slope,prop:structure_Ldiff}, $\norm{D\widetilde{J}(\varTheta_\ast)}_{\Hilbertsp}=\norm{\nabla J\qty[\overline{\eta}]}_{L^2(I;L^2(\overline{\eta}_t))}=\abs{\partial J}(\overline{\eta})=0$; hence $\varTheta_\ast$ is a critical point of $\widetilde{J}$ with $\widetilde{J}(\varTheta_\ast)=J(\overline{\eta})\le E$.

    Let $\zeta\in\metricsp$ satisfy $W_2(\zeta,\overline{\eta})_{L^2(I)}<r$.
    By \cite[Corollary 5.22]{villani_oldnew} and a measurable selection, there is a Borel family $t\mapsto\gamma_t\in\OptPlan(\overline{\eta}_t,\zeta_t)$, and \cref{lem:fixed_center} applied to this $\gamma$ with the same center $\varTheta_\ast$ yields $\varPhi\in\Hilbertsp$ with $\Law\varPhi_t=\zeta_t$ and
    \(
        \norm{\varPhi-\varTheta_\ast}_{\Hilbertsp}^2=\int_0^1 W_2(\overline{\eta}_t,\zeta_t)^2\dd{t}=W_2(\zeta,\overline{\eta})_{L^2(I)}^2<r^2.
    \)
    \cref{thm:grad_ineq_Jtilde} at $\varTheta_\ast$ then gives
    \(
        \abs{\widetilde{J}(\varPhi)-\widetilde{J}(\varTheta_\ast)}^{1-\alpha}\leq C\norm{D\widetilde{J}(\varPhi)}_{\Hilbertsp}.
    \)
    By \eqref{eq:def_Jtil}, we have $\widetilde{J}(\varPhi)=J(\zeta)$ and $\widetilde{J}(\varTheta_\ast)=J(\overline{\eta})$.
    Again by \cref{cor:slope,prop:structure_Ldiff},
    \[
        \norm{D\widetilde{J}(\varPhi)}_{\Hilbertsp}^2=\int_0^1\Expect\qty[\abs{\nabla J\qty[\zeta](t,\varPhi_t)}^2]\dd{t}=\norm{\nabla J\qty[\zeta]}_{L^2(I;L^2(\zeta_t))}^2=\abs{\partial J}^2(\zeta).
    \]
    Therefore $\abs{J(\zeta)-J(\overline{\eta})}^{1-\alpha}\leq C\abs{\partial J}(\zeta)$, which proves the theorem.
\end{proof}
In the remainder of this section, we prove \cref{thm:grad_ineq_Jtilde}.
% \begin{remark}
%     The space $\Hilbertsp$ is a linear space, which makes it amenable to the functional-analytic arguments below. 
%     In contrast, the weak topology of $\Hilbertsp$ is too weak to obtain compactness that is sufficiently strong to show convergence of the gradient flow.
%     Thus, the space $\Ltwotheta$ in \cref{thm:global_conv} is used instead of $\Hilbertsp$.
% \end{remark}

\subsection{The second-variation operator}\label{subsec:second_variation}
For $\varTheta$, $\varPhi\in\Hilbertsp$, the first variation of the state is determined by the first-variation formula \eqref{eq:differential_formula} of \cref{lem:linearizedODE},
\begin{equation}\label{eq:first_variation_op}
    Y_t^{\Law(\varTheta,\varPhi)}(x)=\int_0^t\Expect\qty[M^{\Law\varTheta}(x;t,u)\qty(\mdif{\theta}v)\qty(X_u^{\Law\varTheta}(x),\varTheta_u)\qty[\varPhi_u]]\dd{u},
\end{equation}
where $M^{\Law\varTheta}(x;t,s)\in\R^{d\times d}$ is the fundamental solution from $s$ to $t$ of the linearized ODE $\dot{z}_t=\Expect[(\mdif{x}v)(X_t^{\Law\varTheta}(x),\varTheta_t)]z_t$, and $Y^{\Law(\varTheta,\varPhi)}\in C(I\times K_X;\R^d)$ depends linearly on $\varPhi\in\Hilbertsp$.
The Gronwall estimate of \cref{lem:supp_bound} for the flow and \eqref{eq:Gronwall_bound_M} for the fundamental solution, together with $\int_0^1\Expect[\abs{\varTheta_t}^2]\dd{t}=\norm{\varTheta}^2_{\Hilbertsp}$ and $\int_0^1\Expect[\abs{\varTheta_t}^p]\dd{t}\le1+\norm{\varTheta}^2_{\Hilbertsp}$ for $p<2$, yield for every $R>0$
\begin{equation}\label{eq:uniform_Hball_bounds}
\begin{aligned}
    \adjustlimits\sup_{\norm{\varTheta}_{\Hilbertsp}\le R} \sup_{t\in I,x\in K_X}\abs{X_t^{\Law\varTheta}(x)}<\infty,&&
    \adjustlimits\sup_{\norm{\varTheta}_{\Hilbertsp}\le R} \sup_{0\le s\le t\le1,x\in K_X}\norm{M^{\Law\varTheta}(x;t,s)}_{\opn}<\infty.
\end{aligned}
\end{equation}
Since $\abs{(\mdif{\theta}v)(x,\theta)}\le C_0(1+\abs{\theta}^{(p-1)_+})(1+\abs{x})$ by \eqref{eq:Cth} with $(p-1)_+<1$, the Cauchy--Schwarz inequality and \eqref{eq:uniform_Hball_bounds} give $\norm{Y^{\Law(\varTheta,\varPhi)}}_{C(I\times K_X)}\le C_R\norm{\varPhi}_{\Hilbertsp}$ for $\norm{\varTheta}_{\Hilbertsp}\le R$.

By a computation as in \cite[Proposition 5.4, Remark 2.17]{Scagliotti2023}, the second variation of the loss part $\widetilde{L}(\varTheta)=\int_{\XtimesY}\ell(X_1^{\Law\varTheta}(x),y)\dd{\mu_0(x,y)}$ at $\varTheta$ is, for $\varPsi$, $\varPhi\in\Hilbertsp$,
\begin{equation}
\begin{aligned}
    D^2_\varTheta\widetilde{L}(\varTheta)\qty[\varPsi,\varPhi]=\int_{\XtimesY}\biggl(&\Hessian_x\ell(X_1^{\Law\varTheta}(x),y)\qty[D_\varTheta X_1^{\Law\varTheta}\varPsi,D_\varTheta X_1^{\Law\varTheta}\varPhi]\\
    &{}+\nabla_x\ell(X_1^{\Law\varTheta}(x),y)\cdot D^2_\varTheta X_1^{\Law\varTheta}\qty[\varPsi,\varPhi]\biggr)\dd{\mu_0(x,y)},
\end{aligned}
    \label{eq:noncptop1}
\end{equation}
where $D_\varTheta X_1^{\Law\varTheta}\varPhi=Y_1^{\Law(\varTheta,\varPhi)}$ and the second variation of the terminal state is
\begin{align}
\begin{aligned}
    D^2_\varTheta X^{\Law\varTheta}_1\qty[\varPsi,\varPhi]
    =\int_0^1\Expect\biggl[
    M^{\Law\varTheta}(x;1,t)\biggl(
    &\pdv[2]{v}{x}\qty(X_t^{\Law\varTheta}(x),\varTheta_t)\qty[Y_t^{\Law(\varTheta,\varPsi)},Y_t^{\Law(\varTheta,\varPhi)}] \\
     + &\pdv[2]{v}{x}{\theta}\qty(X_t^{\Law\varTheta}(x),\varTheta_t)\qty[Y_t^{\Law(\varTheta,\varPsi)},\varPhi_t] \\
     + &\pdv[2]{v}{\theta}{x}\qty(X_t^{\Law\varTheta}(x),\varTheta_t)\qty[\varPsi_t,Y_t^{\Law(\varTheta,\varPhi)}] \\
     + &\pdv[2]{v}{\theta}\qty(X_t^{\Law\varTheta}(x),\varTheta_t)
    \qty[\varPsi_t,\varPhi_t]\biggr)\biggr]\dd{t}.
\end{aligned}
\label{eq:D^2X}
\end{align}
With the regularization term in \eqref{eq:def_Jtil} added, the second variation of $\widetilde{J}$ is a bounded symmetric bilinear form on $\Hilbertsp$, and it induces the bounded self-adjoint \emph{linewise second-variation operator} $A_\varTheta+K_\varTheta$ determined by
\begin{equation}\label{eq:def_secondvariation}
    \la(A_\varTheta+K_\varTheta)\varPsi,\varPhi\ra_{\Hilbertsp}=D^2_\varTheta\widetilde{J}(\varTheta)\qty[\varPsi,\varPhi]
\end{equation}
for $\varPsi$, $\varPhi\in\Hilbertsp$, where the principal part $A_\varTheta$ is the multiplication operator arising from the last term of \eqref{eq:D^2X}.

With the costate $\varpi_\varTheta(t,x,y)\coloneqq\nabla_x\ell(X_1^{\Law\varTheta}(x),y)^\top M^{\Law\varTheta}(x;1,t)\in\R^{1\times d}$ and the symmetric operator given by
\(
    \la q,\pdv[2]{v}{\theta}(z,\theta)\ra_{\R^d}\coloneqq\sum_{i=1}^d q_i\pdv[2]{v_i}{\theta}(z,\theta)
\)
for $q\in\R^{1\times d}$ on $\R^m$, the principal part $A_\varTheta$ is given by
\begin{equation}
\begin{aligned}
    \qty(A_\varTheta\varPhi)_t(\omega)
    ={}&\biggl(\epsilon\Id_m\\
    &{}+\int_{\XtimesY}\la\varpi_\varTheta(t,x,y),\pdv[2]{v}{\theta}\qty(X_t^{\Law\varTheta}(x),\varTheta_t(\omega))\ra_{\R^d}\dd{\mu_0(x,y)}\biggr)\varPhi_t(\omega).
\end{aligned}
\label{eq:noncpt_Hess}
\end{equation}
The compact part $K_\varTheta$ is the operator induced by the terminal loss in \eqref{eq:noncptop1} together with the first three terms of \eqref{eq:D^2X}.
The operator $A_\varTheta+K_\varTheta$ represents the second variation of $\widetilde{J}$ along the finite-dimensional affine subspaces of $\Hilbertsp$ spanned by line segments and two-dimensional planes through $\varTheta$.
In contrast with \cite{yamamoto2025hessian}, where the multiplication part of the second variation vanishes at absolutely continuous critical points, our \(A_\varTheta\) is coercive, the parameter laws may be atomic, and \(K_\varTheta\) is a compact, possibly indefinite perturbation.

\begin{proposition}[Stability of the state and its variations]\label{prop:state_stability}
    Let $\varTheta_n\to\varTheta$ in $\Hilbertsp$. Then
    \begin{align}
        X_t^{\Law\varTheta_n}(x)&\to X_t^{\Law\varTheta}(x)&&\text{uniformly for }t\in I,\ x\in K_X,\\
        M^{\Law\varTheta_n}(x;t,s)&\to M^{\Law\varTheta}(x;t,s)&&\text{uniformly for }0\le s\le t\le1,\ x\in K_X,
    \end{align}
    and $\sup_{\norm{\varPhi}_{\Hilbertsp}\le1}\norm{Y^{\Law(\varTheta_n,\varPhi)}-Y^{\Law(\varTheta,\varPhi)}}_{C(I\times K_X)}\to0$.
\end{proposition}
\begin{proof}
    Subtracting the equations $X_t^{\Law\varTheta}(x)=x+\int_0^t\Expect[v(X_s^{\Law\varTheta}(x),\varTheta_s)]\dd{s}$ for $\varTheta_n$ and $\varTheta$, and using $\abs{(\mdif{x}v)(x,\theta)}\le C_x(1+\abs{\theta}^2)$ from \eqref{eq:Cx} to control the Lipschitz dependence on the state, the Gronwall inequality gives
    \begin{align}
        &\sup_{x\in K_X}\abs{X_t^{\Law\varTheta_n}(x)-X_t^{\Law\varTheta}(x)}\\
        \le{}& C_R\int_0^1\sup_{x\in K_X}\Expect\qty[\abs{v\qty(X_s^{\Law\varTheta}(x),(\varTheta_n)_s)-v\qty(X_s^{\Law\varTheta}(x),\varTheta_s)}]\dd{s}.
    \end{align}
    The right-hand side tends to $0$ because $\varTheta_n\to\varTheta$ in $\Hilbertsp$ and $\abs{(\mdif{\theta}v)(x,\theta)}\le C_0(1+\abs{\theta}^{(p-1)_+})(1+\abs{x})$ make $s\mapsto\sup_{x\in K_X}\Expect[\abs{v(X_s^{\Law\varTheta}(x),(\varTheta_n)_s)-v(X_s^{\Law\varTheta}(x),\varTheta_s)}]$ converge to $0$ in $L^1(I)$; this proves the first convergence.
    The fundamental solution $M^{\Law\varTheta}$ solves a linear ODE with coefficient $\Expect[(\mdif{x}v)(X_t^{\Law\varTheta}(x),\varTheta_t)]$, which converges in $L^1(I)$ uniformly in $x\in K_X$ by the first convergence together with $\abs{(\pdv[2]{v}{x})(x,\theta)}\le C_x(1+\abs{\theta}^2)$ and $\abs{(\pdv[2]{v}{x}{\theta})(x,\theta)}\le C_x(1+\abs{\theta})(1+\abs{x})$ from \eqref{eq:Cx}; the Gronwall inequality for linear ODEs then yields the uniform convergence of $M^{\Law\varTheta_n}$.
    Finally, in \eqref{eq:first_variation_op} the coefficient converges uniformly in $x\in K_X$,
    \[
        \norm{\sup_{x\in K_X}\abs{(\mdif{\theta}v)(X_t^{\Law\varTheta_n}(x),(\varTheta_n)_t)-(\mdif{\theta}v)(X_t^{\Law\varTheta}(x),\varTheta_t)}}_{L^2(I\times\Omega)}\to0
    \]
    by \eqref{eq:Cx}, \eqref{eq:Cthth} and the convergence of $X^{\Law\varTheta_n}$; combined with the uniform convergence of $M^{\Law\varTheta_n}$ and the bound $\norm{Y^{\Law(\varTheta,\varPhi)}}_{C(I\times K_X)}\le C_R\norm{\varPhi}_{\Hilbertsp}$, this gives
    \[
        \norm{Y^{\Law(\varTheta_n,\varPhi)}-Y^{\Law(\varTheta,\varPhi)}}_{C(I\times K_X)}\le\omega_n\norm{\varPhi}_{\Hilbertsp}
    \]
    with $\omega_n\to0$ independent of $\varPhi$, which is the last convergence.
\end{proof}

To analyze $K_\varTheta$ we introduce the weight $w_\varTheta(t)\coloneqq1+\Expect[\abs{\varTheta_t}^2]\in L^1(I)$ and the second-order coefficients
\begin{align}
    G^{xx}_\varTheta(t,\omega,x,y)&\coloneqq\varpi_\varTheta(t,x,y)\qty(\pdv[2]{v}{x})(X_t^{\Law\varTheta}(x),\varTheta_t(\omega)),\\
    G^{x\theta}_\varTheta(t,\omega,x,y)&\coloneqq\varpi_\varTheta(t,x,y)\qty(\pdv[2]{v}{x}{\theta})(X_t^{\Law\varTheta}(x),\varTheta_t(\omega)),
\end{align}
which belong to $L^1(I\times\Omega\times\supp\mu_0)$ and $L^2(I\times\Omega\times\supp\mu_0)$, respectively, by \eqref{eq:uniform_Hball_bounds} and \eqref{eq:Cx}.
\begin{proposition}[Properties of $K_\varTheta$]\label{prop:K_properties}
    For each $\varTheta\in\Hilbertsp$ and $R>0$ with $\norm{\varTheta}_{\Hilbertsp}\le R$, the operator $K_\varTheta\colon\Hilbertsp\to\Hilbertsp$ is bounded and self-adjoint with $\norm{K_\varTheta}_{\Bcal(\Hilbertsp)}\le C_R$, and it is compact.
    If $\varTheta\in\strongsp$, then $K_\varTheta\colon\Hilbertsp\to\strongsp$ is bounded.
    Moreover, $\varTheta_n\to\varTheta$ in $\Hilbertsp$ implies $\norm{K_{\varTheta_n}-K_\varTheta}_{\Bcal(\Hilbertsp)}\to0$.
\end{proposition}
\begin{proof}
    \textbf{Boundedness and self-adjointness.}
    By \eqref{eq:uniform_Hball_bounds}, the costate $\varpi_\varTheta$ is bounded uniformly on $I\times\supp\mu_0$ for $\norm{\varTheta}_{\Hilbertsp}\le R$.
    Pairing the four terms of $K_\varTheta$ with $\varPsi$, $\varPhi\in\Hilbertsp$ and using \eqref{eq:Cx} and $\norm{Y^{\Law(\varTheta,\varPhi)}}_{C(I\times K_X)}\le C_R\norm{\varPhi}_{\Hilbertsp}$, we close the term carrying $\pdv[2]{v}{x}$ through the second moment $\int_0^1\Expect[1+\abs{\varTheta_t}^2]\dd{t}=1+\norm{\varTheta}^2_{\Hilbertsp}$ and obtain $\abs{\la\varPsi,K_\varTheta\varPhi\ra}\le C_R\norm{\varPsi}_{\Hilbertsp}\norm{\varPhi}_{\Hilbertsp}$; hence $K_\varTheta\in\Bcal(\Hilbertsp)$ with $\norm{K_\varTheta}\le C_R$. The symmetry of the mixed second derivatives of $v$ makes the bilinear form symmetric.

    \textbf{Compactness.}
    Let $T_{\varTheta,1}\colon\Hilbertsp\to L^2(\supp\mu_0;\R^d)$, $T_{\varTheta,1}\varPhi\coloneqq Y^{\Law(\varTheta,\varPhi)}(1,\bullet)$, and $T_{\varTheta,w}\colon\Hilbertsp\to L^2(I\times\supp\mu_0;\R^d)$, $(T_{\varTheta,w}\varPhi)(t,\bullet)\coloneqq w_\varTheta(t)^{\nicefrac{1}{2}}Y^{\Law(\varTheta,\varPhi)}(t,\bullet)$.
    By \eqref{eq:first_variation_op}, these are Hilbert--Schmidt operators with kernels over $(u,\omega)\in I\times\Omega$
    \begin{align}
        &M^{\Law\varTheta}(x;1,u)(\mdif{\theta}v)(X_u^{\Law\varTheta}(x),\varTheta_u(\omega)),\\
        &\mathbf{1}_{\{u\le t\}}w_\varTheta(t)^{1/2}M^{\Law\varTheta}(x;t,u)(\mdif{\theta}v)(X_u^{\Law\varTheta}(x),\varTheta_u(\omega)),
    \end{align}
    respectively (the expectation in \eqref{eq:first_variation_op} is the pairing with $\varPhi$, not part of the kernel); their Hilbert--Schmidt norms are finite because $\int_0^1 w_\varTheta(t)\dd{t}<\infty$ and $\int_0^1\Expect[(1+\abs{\varTheta_u}^{(p-1)_+})^2]\dd{u}<\infty$ as $2(p-1)_+<2$.
    Define the bounded multiplication operators $H_\varTheta\coloneqq\Hessian_x\ell(X_1^{\Law\varTheta},y)$ on $L^2(\supp\mu_0;\R^d)$ and $S_\varTheta$ on $L^2(I\times\supp\mu_0;\R^d)$, the latter acting by the pointwise symmetric form
    \[
        \la S_\varTheta(t,x,y)a,b\ra\coloneqq w_\varTheta(t)^{-1}\Expect\qty[\varpi_\varTheta(t,x,y)\qty(\pdv[2]{v}{x})(X_t^{\Law\varTheta}(x),\varTheta_t)[a,b]],
    \]
    and the bounded operator $R_\varTheta\colon\Hilbertsp\to L^2(I\times\supp\mu_0;\R^d)$ by
    \[
        (R_\varTheta\varPhi)(t,x,y)\coloneqq w_\varTheta(t)^{-1/2}\Expect\qty[\varpi_\varTheta(t,x,y)\qty(\pdv[2]{v}{x}{\theta})(X_t^{\Law\varTheta}(x),\varTheta_t)[\bullet,\varPhi_t]].
    \]
    Here $\abs{S_\varTheta}\le C_R$ and $\norm{R_\varTheta}\le C_R$, since $\Expect[\abs{\varTheta_t}^2]\le w_\varTheta(t)$ and, by Cauchy--Schwarz, $\Expect[\abs{(\pdv[2]{v}{x}{\theta})\varPhi_t}]\le C w_\varTheta(t)^{\nicefrac{1}{2}}\Expect[\abs{\varPhi_t}^2]^{1/2}$.
    Then the four terms of $K_\varTheta$ factorize as
    \begin{equation}\label{eq:K_factorization}
        K_\varTheta=T_{\varTheta,1}^\ast H_\varTheta T_{\varTheta,1}+T_{\varTheta,w}^\ast S_\varTheta T_{\varTheta,w}+R_\varTheta^\ast T_{\varTheta,w}+T_{\varTheta,w}^\ast R_\varTheta,
    \end{equation}
    each summand being a Hilbert--Schmidt operator composed with bounded operators, hence compact, as in \cite[Proposition 2.21]{Scagliotti2023}.

    \textbf{Regularization.}
    If $\varTheta\in\strongsp$, then all coefficients above are bounded, so
    \begin{align}
        \abs{Y^{\Law(\varTheta,\varPhi)}(t,x)}\le C\norm{\varPhi}_{L^1(I\times\Omega)},&&
        \abs{\la\varPsi,K_\varTheta\varPhi\ra}\le C\norm{\varPsi}_{\Hilbertsp}\norm{\varPhi}_{L^1(I\times\Omega)}
    \end{align}
    for all $\varPsi$, $\varPhi\in\Hilbertsp$.
    Because $K_\varTheta$ is self-adjoint, $\la K_\varTheta\varPsi,\varPhi\ra=\la\varPsi,K_\varTheta\varPhi\ra$, so $\abs{\la K_\varTheta\varPsi,\varPhi\ra}\le C\norm{\varPsi}_{\Hilbertsp}\norm{\varPhi}_{L^1(I\times\Omega)}$.
    Since $\Hilbertsp=L^2(I\times\Omega;\R^m)$ is dense in $L^1(I\times\Omega;\R^m)$, the functional $\varPhi\mapsto\la K_\varTheta\varPsi,\varPhi\ra$ extends to a bounded functional on $L^1(I\times\Omega;\R^m)$, and the duality $(L^1)^\ast=L^\infty$ shows $K_\varTheta\varPsi\in\strongsp$ with $\norm{K_\varTheta\varPsi}_{\strongsp}\le C\norm{\varPsi}_{\Hilbertsp}$, i.e., $K_\varTheta\colon\Hilbertsp\to\strongsp$ is bounded.

    \textbf{Operator-norm continuity.}
    Comparing the bilinear forms of $K_{\varTheta_n}$ and $K_\varTheta$ directly, term by term through \eqref{eq:K_factorization}, gives
    \begin{align}
        \norm{K_{\varTheta_n}-K_\varTheta}_{\Bcal(\Hilbertsp)}
        \le{}&C_R\left(\sup_{\norm{\varPhi}_{\Hilbertsp}\le1}\norm{Y^{\Law(\varTheta_n,\varPhi)}-Y^{\Law(\varTheta,\varPhi)}}_{C(I\times K_X)}\right.\\
        &\quad{}+\sup_{(x,y)\in\supp\mu_0}\abs{\Hessian_x\ell(X_1^{\Law\varTheta_n}(x),y)-\Hessian_x\ell(X_1^{\Law\varTheta}(x),y)}\\
        &\quad{}+\norm{G^{xx}_{\varTheta_n}-G^{xx}_\varTheta}_{L^1(I\times\Omega\times\supp\mu_0)}\\
        &\quad{}+\norm{G^{x\theta}_{\varTheta_n}-G^{x\theta}_\varTheta}_{L^2(I\times\Omega\times\supp\mu_0)}\left.\vphantom{\sup_{\norm{\varPhi}_{\Hilbertsp}\le1}\norm{Y^{\Law(\varTheta_n,\varPhi)}-Y^{\Law(\varTheta,\varPhi)}}_{C(I\times K_X)}}\right).
    \end{align}
    Each term tends to $0$: the first by \cref{prop:state_stability}; the second because $X_1^{\Law\varTheta_n}\to X_1^{\Law\varTheta}$ uniformly on $K_X$ and $\Hessian_x\ell$ is continuous; and the last two because $\varpi_{\varTheta_n}\to\varpi_\varTheta$ uniformly and, by \cref{prop:state_stability} with \eqref{eq:Cx},
    \begin{align}
        \qty(\pdv[2]{v}{x})(X^{\Law\varTheta_n},\varTheta_n)&\to\qty(\pdv[2]{v}{x})(X^{\Law\varTheta},\varTheta)&&\text{in }L^1(I\times\Omega\times\supp\mu_0),\\
        \qty(\pdv[2]{v}{x}{\theta})(X^{\Law\varTheta_n},\varTheta_n)&\to\qty(\pdv[2]{v}{x}{\theta})(X^{\Law\varTheta},\varTheta)&&\text{in }L^2(I\times\Omega\times\supp\mu_0).
    \end{align}
    These two convergences follow from the strong convergence $\varTheta_n\to\varTheta$ in $\Hilbertsp=L^2$ together with the continuity and the growth bounds \eqref{eq:Cx} of $\pdv[2]{v}{x}$ and $\pdv[2]{v}{x}{\theta}$, which make the associated Nemytskii operators continuous, by the Vitali convergence theorem (the required equi-integrability of $1+\abs{(\varTheta_n)_t}^2$ follows from $\varTheta_n\to\varTheta$ in $L^2$).
    Hence $\norm{K_{\varTheta_n}-K_\varTheta}_{\Bcal(\Hilbertsp)}\to0$.
\end{proof}

The principal part $A_\varTheta=\epsilon\Id+b_\varTheta$, where $b_\varTheta$ is the multiplication by the integral in \eqref{eq:noncpt_Hess}, is bounded with $\norm{A_\varTheta}_{\Bcal(\Hilbertsp)}\le C_R$ for $\norm{\varTheta}_{\Hilbertsp}\le R$ by \eqref{eq:uniform_Hball_bounds}; together with \cref{prop:K_properties} this gives $A_\varTheta+K_\varTheta\in\Bcal(\Hilbertsp)$ with $\norm{A_\varTheta+K_\varTheta}_{\Bcal(\Hilbertsp)}\le C_R$.
The decomposition \eqref{eq:def_secondvariation} yields the line-segment chain rule underlying the gradient inequality.
\begin{proposition}[Line-segment chain rule]\label{prop:H1}
    For $\varTheta^0$, $\varTheta^1\in\Hilbertsp$, set $\varTheta^r\coloneqq\varTheta^0+r(\varTheta^1-\varTheta^0)$ and $h\coloneqq\varTheta^1-\varTheta^0$.
    Then $r\mapsto D\widetilde{J}(\varTheta^r)$ is of class $C^1(I;\Hilbertsp)$ with $\dv{r}D\widetilde{J}(\varTheta^r)=(A_{\varTheta^r}+K_{\varTheta^r})h$, and
    \begin{equation}\label{eq:H1}
        D\widetilde{J}(\varTheta^1)-D\widetilde{J}(\varTheta^0)=\int_0^1(A_{\varTheta^r}+K_{\varTheta^r})h\dd{r}
    \end{equation}
    as a Bochner integral in $\Hilbertsp$.
    In particular, $D\widetilde{J}$ is Lipschitz on every bounded subset.
\end{proposition}
\begin{proof}
    Fix $R\ge\norm{\varTheta^0}_{\Hilbertsp}+\norm{\varTheta^1}_{\Hilbertsp}$, so $\norm{\varTheta^r}_{\Hilbertsp}\le R$ for $r\in I$, and let $\varPsi\in\Hilbertsp$.
    On the two-dimensional affine plane $\varTheta^{r,\tau}\coloneqq\varTheta^r+\tau\varPsi$, the state $X^{\Law\varTheta^{r,\tau}}$ and its first and second variations solve the linear ODEs \eqref{eq:differential_formula} and \eqref{eq:D^2X}.
    The coefficients and inhomogeneous terms of these equations depend continuously on $(r,\tau)$ in $L^1$, by the continuity and the growth bounds \eqref{eq:Cth}, \eqref{eq:Cx}, \eqref{eq:Cthth} together with the uniform bounds \eqref{eq:uniform_Hball_bounds} and $(p-1)_+<1$; continuous dependence for linear ODEs then makes the first and second variations of the state continuous in $(r,\tau)$, and composing with $\nabla_x\ell$ and $\Hessian_x\ell$, continuous on the reachable compact set, yields that $(r,\tau)\mapsto\widetilde{J}(\varTheta^{r,\tau})$ is of class $C^2$ on this plane.
    Consequently $\eval{\partial_\tau\widetilde{J}(\varTheta^{r,\tau})}_{\tau=0}=\la D\widetilde{J}(\varTheta^r),\varPsi\ra$, and the mixed derivative $\eval{\partial_r\partial_\tau\widetilde{J}(\varTheta^{r,\tau})}_{\tau=0}$ equals the bilinear form induced by \eqref{eq:D^2X} and the regularization term evaluated at $(h,\varPsi)$, that is, $\la(A_{\varTheta^r}+K_{\varTheta^r})h,\varPsi\ra$.
    By dominated convergence for $A_{\varTheta^r}$ and \cref{prop:K_properties} for $K_{\varTheta^r}$, the map $r\mapsto(A_{\varTheta^r}+K_{\varTheta^r})h$ is continuous in $\Hilbertsp$ and bounded by $C_R\norm{h}_{\Hilbertsp}$, hence Bochner integrable.
    The equality of mixed partials on the plane gives $\dv{r}\la D\widetilde{J}(\varTheta^r),\varPsi\ra=\la(A_{\varTheta^r}+K_{\varTheta^r})h,\varPsi\ra$; integrating in $r$ and varying $\varPsi$ yields \eqref{eq:H1}, and $\norm{D\widetilde{J}(\varTheta^1)-D\widetilde{J}(\varTheta^0)}_{\Hilbertsp}\le C_R\norm{h}_{\Hilbertsp}$.
\end{proof}

The principal part is coercive up to the margin \eqref{eq:coercivity_margin}.
\begin{proposition}[Principal coercivity]\label{prop:H2}
    For every $R>0$ and $\varTheta\in\Hilbertsp$ with $\norm{\varTheta}_{\Hilbertsp}\le R$,
    % \begin{equation}\label{eq:H2}
    it holds that
    \(
        \la A_\varTheta\varPhi,\varPhi\ra_{\Hilbertsp}\ge a_R\norm{\varPhi}_{\Hilbertsp}^2
    \)
    for $\varPhi\in\Hilbertsp$.
    % \end{equation}
    % In particular, if $a_R>0$, then $A_\varTheta\ge a_R\Id$ on $\Hilbertsp$.
\end{proposition}
\begin{proof}
    By \eqref{eq:noncpt_Hess}, we have
    \(
        \la A_\varTheta\varPhi,\varPhi\ra_{\Hilbertsp}=\epsilon\norm{\varPhi}_{\Hilbertsp}^2+\int_0^1\Expect\qty[\la b_\varTheta(t,\bullet)\varPhi_t,\varPhi_t\ra]\dd{t}.
    \)
    The bound $\abs{(\pdv[2]{v}{\theta})(x,\theta)}\le C_{\theta\theta}(1+\abs{x}^2)$ from \eqref{eq:Cthth} gives, for a.e.~$(t,\omega)$ and all $\norm{\varTheta}_{\Hilbertsp}\le R$,
    \[
    \begin{aligned}
        \norm{b_\varTheta(t,\omega)}_{\opn}&\le C_{\theta\theta}\int_{\XtimesY}\abs{\varpi_\varTheta(t,x,y)}\qty(1+\abs{X_t^{\Law\varTheta}(x)}^2)\dd{\mu_0(x,y)}\\
        &=C_{\theta\theta}\int_{\XtimesY}\abs{\nabla_x\varphi_t^{\Law\varTheta}(z,y)}\qty(1+\abs{z}^2)\dd{\mu_t^{\Law\varTheta}(z,y)}\le\epsilon-a_R.
    \end{aligned}
    \]
    Here the equality follows from $\mu_t^{\Law\varTheta}=(X_t^{\Law\varTheta}\times\Id_\Y)_\#\mu_0$ and $\varpi_\varTheta(t,x,y)^\top=\nabla_x\varphi_t^{\Law\varTheta}(X_t^{\Law\varTheta}(x),y)$, which is \cref{prop:grad} combined with the cocycle identity for the fundamental solution, and the last inequality is the definition \eqref{eq:coercivity_margin} of $a_R$, since $\norm{\varTheta}_{\Hilbertsp}\le R$ implies $\Law\varTheta\in\Dcal_R$.
    Hence $\la A_\varTheta\varPhi,\varPhi\ra_{\Hilbertsp}\ge\epsilon\norm{\varPhi}_{\Hilbertsp}^2-(\epsilon-a_R)\norm{\varPhi}_{\Hilbertsp}^2
    % =a_R\norm{\varPhi}_{\Hilbertsp}^2
    $.
\end{proof}

\subsection{Strong-space analyticity}
We establish the real analyticity of the lifted gradient $D\widetilde{J}$ on the strong space $\strongsp$.
Note that any critical point $\varTheta^\ast\in\Hilbertsp$ belongs to $\strongsp$ by the same argument as in \cref{lem:Linf_crit_pt}.

\begin{proposition}[Analyticity of the lifted gradient]\label{prop:Analytic_J}
    Let $\varTheta^\ast\in\strongsp$.
    Under \cref{assump:bound_dNN}, there exists $\rho>0$ such that
    \(
        D\widetilde{J}\colon B_{\strongsp}(\varTheta^\ast,\rho)\to\strongsp
    \)
    is real analytic.
    Moreover, there exists $C>0$ such that
    \(
        \norm{D^n(D\widetilde{J})(\varTheta^\ast)}_{\Bcal^n(\strongsp;\strongsp)}\le Cn!\rho^{-n}
    \)
    for $n\in\N$.
\end{proposition}
\begin{proof}
    Since $\varTheta^\ast\in\strongsp$, its essential range is contained in a compact set $K_\theta\subset\R^m$, and by \eqref{eq:uniform_Hball_bounds} the state $X^{\Law\varTheta^\ast}$ ranges over a compact set $K_z\subset\R^d$.

    \textbf{Step 1 (complexification).}
    Since $v$ is real-analytic, it extends holomorphically to $\widehat{v}$ on a complex tube around $K_z\times K_\theta$ in $\C^d\times\C^m$, and by \cref{assump:bound_dNN} there exist a complex neighborhood $U$ of $K_z$ and a continuous extension $\widehat{\ell}\colon U\times\Y\to\C$, holomorphic in its first variable, locally bounded uniformly in $y\in\Y$, and agreeing with $\ell$ on $(U\cap\R^d)\times\Y$.
    By the finite-dimensional Cauchy estimates, the Taylor series of $\widehat{v}$ converges uniformly on a smaller complex tube; hence the Nemytskii operator $u\mapsto\widehat{v}\circ u$ is holomorphic, into the complexification of $L^\infty$, on the open set of $u\in L^\infty$ whose essential range has positive distance from the complement of that tube, and likewise for $\widehat{\ell}$ and for the derivatives of $\widehat{v}$.

    \textbf{Step 2 (holomorphic dependence of the state and costate).}
    Define, on a neighborhood of $(\varTheta^\ast,X^{\Law\varTheta^\ast})$ in the complexifications of $\strongsp$ and $C(I\times K_X;\R^d)$, the map
    \(
        (\varTheta,Z)\mapsto Z(t,x)-x-\int_0^t\Expect\qty[\widehat{v}(Z(s,x),\varTheta_s)]\dd{s},
    \)
    which is holomorphic by Step 1 and vanishes at $(\varTheta^\ast,X^{\Law\varTheta^\ast})$.
    Its partial derivative in $Z$ at this point is the identity minus the forward Volterra integral operator
    \(
        W\mapsto\int_0^t\Expect\qty[(\mdif{z}\widehat{v})(X_s^{\Law\varTheta^\ast}(x),\varTheta^\ast_s)]W(s,x)\dd{s},
    \)
    whose coefficient has operator norm at most some $\beta<\infty$, uniformly for $s\in I$ and $x\in K_X$, by Step 1; since the integration is over $0\le s\le t$, the $n$-fold iterate is an integral over the simplex $\{0\le s_1\le\dots\le s_n\le t\}$ of volume $t^n/n!$, so the $n$-th power of this operator has norm at most $\beta^n/n!$ and the Neumann series inverts the derivative without any smallness condition.
    By the analytic implicit function theorem \cite[Theorem 4.5.4]{Buffoni03}, the equation defining the state has, near $\varTheta^\ast$, a unique solution $Z^\varTheta$ depending holomorphically on $\varTheta$, with $Z^\varTheta=X^{\Law\varTheta}$ for real $\varTheta$.
    Likewise, the costate $\widehat{\varpi}_\varTheta\in C(I\times\supp\mu_0;\C^{1\times d})$, a row vector, solves the backward Volterra equation
    \[
        \widehat{\varpi}_\varTheta(t,x,y)=\nabla_z\widehat{\ell}(Z_1^\varTheta(x),y)^\top+\int_t^1\widehat{\varpi}_\varTheta(s,x,y)\Expect\qty[(\mdif{z}\widehat{v})(Z_s^\varTheta(x),\varTheta_s)]\dd{s};
    \]
    the map $(\varTheta,\varpi)\mapsto\varpi-\nabla_z\widehat{\ell}(Z_1^\varTheta,\bullet)^\top-\int_\bullet^1\varpi(s)\Expect[(\mdif{z}\widehat{v})(Z_s^\varTheta,\varTheta_s)]\dd{s}$ is holomorphic, and its partial derivative in $\varpi$ at $\varTheta^\ast$ is the identity minus a backward Volterra operator, whose integration over $t\le s\le1$ yields the same factorial bound $\beta^n/n!$, hence invertible; so $\varTheta\mapsto\widehat{\varpi}_\varTheta$ is holomorphic near $\varTheta^\ast$, and for real $\varTheta$ it coincides with the costate $\varpi_\varTheta$ of \cref{subsec:second_variation}.

    \textbf{Step 3 (analyticity of the gradient and the factorial estimate).}
    By \cref{prop:grad} and the lifting \eqref{eq:def_Jtil}, for $\dd{\mathbb P\dd t}$-a.e.\ $(t,\omega)$,
    \[
        D\widetilde{J}(\varTheta)_t(\omega)=\epsilon\varTheta_t(\omega)+\int_{\XtimesY}(\mdif{\theta}v)(X_t^{\Law\varTheta}(x),\varTheta_t(\omega))^\top\varpi_\varTheta(t,x,y)^\top\dd{\mu_0(x,y)}.
    \]
    The right-hand side is the restriction to the real subspace of the map sending $\varTheta$ to
    \(
        \epsilon\varTheta_t(\omega)+\int_{\XtimesY}(\mdif{\theta}\widehat{v})(Z_t^\varTheta(x),\varTheta_t(\omega))^\top\widehat{\varpi}_\varTheta(t,x,y)^\top\dd{\mu_0(x,y)},
    \)
    which is holomorphic by Step 2 and hence extends $D\widetilde{J}$ holomorphically; therefore $D\widetilde{J}$ is real-analytic on a ball $B_{\strongsp}(\varTheta^\ast,\rho)$.
    After decreasing $\rho$ if necessary, \cite[Proposition 4.3.4]{Buffoni03} gives the local factorial estimate $\norm{D^n(D\widetilde{J})(\varTheta^\ast)}_{\Bcal^n(\strongsp;\strongsp)}\le Cn!\rho^{-n}$.
\end{proof}

\begin{corollary}\label{cor:analytic_Jtil}
    Under \cref{assump:bound_dNN}, the restriction of $\widetilde{J}$ to $\strongsp$ is real analytic, and at every $\varTheta^\ast\in\strongsp$ the Fr\'{e}chet derivative of $D\widetilde{J}\colon\strongsp\to\strongsp$ is the restriction $(A_{\varTheta^\ast}+K_{\varTheta^\ast})|_{\strongsp}$ of the second-variation operator \eqref{eq:def_secondvariation}.
\end{corollary}
\begin{proof}
    Since $\widetilde{J}(\varTheta^\ast+h)-\widetilde{J}(\varTheta^\ast)=\int_0^1\la D\widetilde{J}(\varTheta^\ast+sh),h\ra_{\Hilbertsp}\dd{s}$ and $D\widetilde{J}$ is real analytic on $\strongsp$ by \cref{prop:Analytic_J}, the lifted energy $\widetilde{J}|_{\strongsp}$ is real analytic.
    For the derivative, by \cref{prop:H1} the difference quotient $s^{-1}(D\widetilde{J}(\varTheta^\ast+s\varPhi)-D\widetilde{J}(\varTheta^\ast))$ converges to $(A_{\varTheta^\ast}+K_{\varTheta^\ast})\varPhi$ in $\Hilbertsp$ as $s\to0$, while by \cref{prop:Analytic_J} it converges to $D(D\widetilde{J})(\varTheta^\ast)\varPhi$ in $\strongsp\hookrightarrow\Hilbertsp$; uniqueness of the limit gives $D(D\widetilde{J})(\varTheta^\ast)=(A_{\varTheta^\ast}+K_{\varTheta^\ast})|_{\strongsp}$.
\end{proof}

\subsection{A nonsmooth \L{}ojasiewicz--Simon reduction}\label{subsec:feireisl}
We prove \cref{thm:grad_ineq_Jtilde} by adapting the nonsmooth \L{}ojasiewicz--Simon reduction of Feireisl, Issard-Roch, and Petzeltov\'{a}~\cite{FEIREISL20041} to a locally nonlinear compact perturbation.
The reduction combines linewise second variations on $H$ with analyticity on the strong space.

\begin{lemma}[Nonsmooth \L{}ojasiewicz--Simon reduction]\label{lem:feireisl}
    Let $H$ be a Hilbert space with identity operator $\Id$, $V\hookrightarrow H$ a continuously and densely embedded Banach space, and $G\in C^1(B_H(0,r_0))$ on a ball $B_H(0,r_0)\subset H$ with $H$-gradient $DG$ and $DG(0)=0$.
    Assume there are bounded self-adjoint operators $A_u,K_u$ on $H$ and a constant $a>0$ such that, for $u,v\in B_H(0,r_0)$,
    \begin{myenum}
        \item\label{enum:F1} $DG(u)-DG(v)=\int_0^1(A_{w_s}+K_{w_s})(u-v)\dd{s}$ along $w_s\coloneqq v+s(u-v)$;
        \item\label{enum:F2} $A_u\ge a\Id$ and $\sup_{\norm{u}_H<r_0}(\norm{A_u}+\norm{K_u})<\infty$;
        \item\label{enum:F3} each $K_u$ is compact and $\norm{K_u-K_0}_{\Bcal(H)}\to0$ as $u\to0$;
        \item\label{enum:F4} $K_0\colon H\to V$ is bounded and $A_0\colon V\to V$ is a bounded isomorphism;
        \item\label{enum:F5} $DG|_V$ is real analytic near $0$ with $D_V(DG|_V)(0)=(A_0+K_0)|_V$.
    \end{myenum}
    Then, there exist $\alpha\in(0,1/2]$, $C>0$ and $r>0$ such that
    \(
        \abs{G(u)-G(0)}^{1-\alpha}\le C\norm{DG(u)}_H
    \)
    for $\norm{u}_H<r$.
\end{lemma}
We prove the above by extending the nonsmooth \L{}ojasiewicz--Simon reduction of \cite[Theorem 5.1]{FEIREISL20041} from a fixed linear compact perturbation to a locally varying family \(K_u\).
\begin{proof}
    \textbf{Step 1 (finite-rank correction).}
    By \ref{enum:F3}, $K_0$ is compact and self-adjoint, so its spectrum is discrete away from $0$ and only finitely many eigenvalues $\kappa_j$ satisfy $\kappa_j<-a/2$; with the corresponding spectral projections $P_j$, set $Q\coloneqq\sum_{\kappa_j<-a/2}(-\kappa_j-a/2)P_j$, a finite-rank, self-adjoint, nonnegative operator with $K_0+Q\ge-(a/2)\Id$.
    Each eigenvector $e_j$ with $\kappa_j\ne0$ satisfies $e_j=\kappa_j^{-1}K_0e_j\in V$ by \ref{enum:F4}, so $N\coloneqq\Ran Q\subset V$ is finite-dimensional and $Q\colon H\to V$ is bounded.

    \textbf{Step 2 (local strong monotonicity).}
    By \ref{enum:F2} and \ref{enum:F3}, there is $r_1\in(0,r_0)$ with $\norm{K_u-K_0}\le a/4$ for $\norm{u}_H<r_1$; setting $c\coloneqq a/4$, we then have $A_u+K_u+Q\ge c\Id$.
    Writing $F\coloneqq DG+Q$, \ref{enum:F1} and the convexity of $B_H(0,r_1)$ give, for $\norm{u}_H,\norm{v}_H<r_1$,
    \begin{align*}
        \la F(u)-F(v),u-v\ra&\ge c\norm{u-v}_H^2,
        &
        \norm{F(u)-F(v)}_H&\le L_F\norm{u-v}_H,
    \end{align*}
    with $L_F\coloneqq\sup_{u}(\norm{A_u}+\norm{K_u})+\norm{Q}$.

    \textbf{Step 3 (Lipschitz inverse on $H$).}
    Fix $\rho<r_1$ and let $\overline{B}_\rho$ be the closed ball of radius $\rho$ in $H$.
    For $\norm{y}_H<c\rho$ and $0<\tau<2c/L_F^2$, the map $T_y(u)\coloneqq P_{\overline{B}_\rho}(u-\tau(F(u)-y))$, where $P_{\overline{B}_\rho}$ denotes the metric projection onto $\overline{B}_\rho$, is a contraction by strong monotonicity and the Lipschitz bound, so it has a unique fixed point $u_y$, which solves the variational inequality $\la F(u_y)-y,w-u_y\ra\ge0$ for $w\in\overline{B}_\rho$; testing with $w=0$ and using $F(0)=DG(0)+Q0=0$ and strong monotonicity give $c\norm{u_y}_H^2\le\la y,u_y\ra$, so $\norm{u_y}_H\le c^{-1}\norm{y}_H<\rho$ and hence $F(u_y)=y$.
    Thus $F$ has a Lipschitz inverse $\Lambda_H\coloneqq F^{-1}\colon B_H(0,\delta)\to B_H(0,\rho)$ with $\delta\coloneqq\nicefrac{c\rho}{2}$ and $\norm{\Lambda_H(y_1)-\Lambda_H(y_2)}_H\le c^{-1}\norm{y_1-y_2}_H$.

    \textbf{Step 4 (analytic inverse on $V$).}
    By \ref{enum:F5} and the finite-rank construction above, $F|_V=DG|_V+Q$ is real analytic with $D_V(F|_V)(0)=A_0+K_0+Q$.
    This operator is self-adjoint with $\la(A_0+K_0+Q)h,h\ra\ge\nicefrac{a}{2}\norm{h}_H^2$, so by the Lax--Milgram theorem it is an isomorphism of $H$; in particular $\norm{u}_H\le C\norm{f}_H$ when $(A_0+K_0+Q)u=f$.
    For $f\in V$, writing $A_0u=f-(K_0+Q)u\in V$ and using \ref{enum:F4} gives $u=A_0^{-1}(f-(K_0+Q)u)\in V$ with
    \(
        \norm{u}_V\le\norm{A_0^{-1}}_{\Bcal(V)}\qty(\norm{f}_V+\norm{K_0+Q}_{\Bcal(H,V)}\norm{u}_H)\le C\norm{f}_V,
    \)
    so $A_0+K_0+Q$ is also an isomorphism of $V$.
    The analytic implicit function theorem~\cite[Theorem 4.5.4]{Buffoni03} then yields a real-analytic local inverse $\Lambda_V$ of $F|_V$, and by uniqueness of the $H$-inverse, $\Lambda_V(\xi)=\Lambda_H(\xi)$ for $\xi\in N$ sufficiently close to $0$.

    \textbf{Step 5 (finite-dimensional reduction).}
    Define $g(\xi)\coloneqq G(\Lambda_V(\xi))$ for $\xi\in N$ near $0$.
    Since $G(x)-G(0)=\int_0^1\la DG(sx),x\ra\dd{s}$ and $DG|_V$ is analytic, $G|_V$ is real analytic; hence $g$ is real analytic on the finite-dimensional space $N$ with $Dg(0)=0$.
    The classical \L{}ojasiewicz inequality gives $\alpha\in(0,1/2]$ and $C_0>0$ with $\abs{g(\xi)-g(0)}^{1-\alpha}\le C_0\norm{Dg(\xi)}_{N^\ast}$ near $0$.
    % (when $N=\{0\}$, Step 6 applies with $z=0$ and gives $\abs{G(x)-G(0)}\le C\norm{DG(x)}_H^2$ directly, i.e., $\alpha=1/2$)

    \textbf{Step 6 (transfer to $H$).}
    For $\norm{x}_H$ small, set $z\coloneqq\Lambda_H(Qx)=\Lambda_V(Qx)$, so $F(z)=Qx$ and $DG(z)=Q(x-z)$.
    From $F(x)-F(z)=DG(x)$ and strong monotonicity,
    \begin{align*}
        \norm{x-z}_H&\le c^{-1}\norm{DG(x)}_H,
        &
        \norm{DG(z)}_H&\le(\norm{Q}/c)\norm{DG(x)}_H,
    \end{align*}
    while, with $D\Lambda_V$ bounded near $0$, the identity $Dg(Qx)[\eta]=\la DG(z),D\Lambda_V(Qx)\eta\ra$ gives
    \(
        \norm{Dg(Qx)}_{N^\ast}\le C\norm{DG(x)}_H\), and therefore
    \(
        \abs{G(z)-G(0)}^{1-\alpha}\le C\norm{DG(x)}_H.
    \)
    Finally,
    \(
        \abs{G(x)-G(z)}\le\norm{DG(z)}_H\norm{x-z}_H+\frac{L_F}{2}\norm{x-z}_H^2\le C\norm{DG(x)}_H^2,
    \)
    and since $1-\alpha\ge1/2$ and $\norm{DG(x)}_H\le1$ for $x$ small, subadditivity of $t\mapsto t^{1-\alpha}$ yields $\abs{G(x)-G(0)}^{1-\alpha}\le C\norm{DG(x)}_H$.
\end{proof}

We prove \ref{enum:F1}--\ref{enum:F5} for $\widetilde{J}$ and conclude.
\begin{proof}[Proof of \cref{thm:grad_ineq_Jtilde}]
    Let $\varTheta\in\strongsp$ be a critical point with $\widetilde{J}(\varTheta)\le E$.
    As $\widetilde{J}(\varTheta)\le E$ gives $\norm{\varTheta}_{\Hilbertsp}^2\le2E/\epsilon$, choosing the radius $r_0<1$ ensures $\norm{\varTheta+u}_{\Hilbertsp}\le\sqrt{2E/\epsilon}+1=R_E$ for $u\in B_{\Hilbertsp}(0,r_0)$.
    Translate $\varTheta$ to $0$ by setting $G(u)\coloneqq\widetilde{J}(\varTheta+u)-\widetilde{J}(\varTheta)$ on $H=\Hilbertsp$, with $V=\strongsp$; then $DG(u)=D\widetilde{J}(\varTheta+u)$ and $DG(0)=0$ since $\varTheta$ is critical.
    With $A_u\coloneqq A_{\varTheta+u}$ and $K_u\coloneqq K_{\varTheta+u}$ from \eqref{eq:def_secondvariation}, \cref{prop:H1} is \ref{enum:F1}.
    For \ref{enum:F2}, the coercivity $A_u\ge a_{R_E}\Id$ with $a_{R_E}>0$ is \cref{prop:H2} together with \cref{assump:curvature}, while the uniform bound $\sup_u(\norm{A_u}+\norm{K_u})<\infty$ follows from \eqref{eq:uniform_Hball_bounds}, the boundedness of $A_\varTheta$, and \cref{prop:K_properties}.
    This proposition gives \ref{enum:F3} and the regularization $K_0\colon\Hilbertsp\to\strongsp$ in \ref{enum:F4}; and since $A_0=\epsilon\Id+b_\varTheta$ is multiplication by a symmetric matrix field with $A_0\ge a_{R_E}\Id_m$ a.e.\ and $L^\infty$ coefficients, $A_0\colon\strongsp\to\strongsp$ is an isomorphism with $\norm{A_0^{-1}}\le a_{R_E}^{-1}$, completing \ref{enum:F4}; and \cref{cor:analytic_Jtil} is \ref{enum:F5}.
    \cref{lem:feireisl} with $a=a_{R_E}$ then yields $\alpha\in(0,1/2]$, $C>0$ and $r>0$ with $\abs{\widetilde{J}(\varPhi)-\widetilde{J}(\varTheta)}^{1-\alpha}\le C\norm{D\widetilde{J}(\varPhi)}_{\Hilbertsp}$ for $\norm{\varPhi-\varTheta}_{\Hilbertsp}<r$.
    Here the finite-rank correction $Q=Q_\varTheta$ of \cref{lem:feireisl} is built from $K_\varTheta$, so it and the constants $\alpha$, $C$, $r$ depend on the critical point $\varTheta$.
\end{proof}

\subsection{Long-time convergence}
Based on the inequality in \cref{thm:grad_ineq} and the abstract convergence theorem \cite[Theorem 3.24]{Hauer19}, we establish the convergence result stated in \cref{thm:global_conv}.

\begin{proof}[Proof of \cref{thm:global_conv}]
By \cref{prop:Orbit_in_Sobolev}, the $\omega$-limit set of the flow
% \[
% % \omega(\eta)\coloneqq
% \Set{\eta^\ast\in\Ltwotheta|
% \begin{array}{c}
%      \text{there exists a subsequence $(\tau_n)_{n=1}^\infty\subset[0,+\infty)$}  \\
%      \text{such that }\tau_n\to\infty\text{ and}
%      \\\eta(\tau_n)\to\eta^\ast\text{ strongly in }\Ltwotheta.
% \end{array}
% }
% \] 
is nonempty.
Fix $R>\sqrt{2J(\eta_0)/\epsilon}$.
Since $J$ is nonincreasing along the curve and lower semicontinuous, the closure of the orbit lies in the sublevel set $\{J\le J(\eta_0)\}$, on which $J_R=J$ and $\abs{\partial J_R}=\abs{\partial J}$ by \cref{lem:localization}; as $J_R$ is $\lambda_R^J$-convex along generalized geodesics by \cref{cor:conv_JR}, the slope $\abs{\partial J}$ is lower semicontinuous there by \cite[Proposition 2.22]{Hauer19}.
Hence \cite[Proposition 2.37]{Hauer19} shows that every element of the $\omega$-limit set is a critical point of $J$.
Fix such an $\eta^\ast$.
Since $\int_0^1m_2(\eta^\ast_t)\dd{t}\le\nicefrac{2J(\eta_0)}{\epsilon}<R^2$, \cref{lem:localization} gives $J_R=J$ and $\abs{\partial J_R}=\abs{\partial J}$ on a ball around $\eta^\ast$; in particular, $\eta^\ast$ is critical for $J_R$, and \cref{thm:grad_ineq} supplies the local slope inequality for $J_R$ around $\eta^\ast$.
We now apply \cite[Theorem 3.24]{Hauer19} to the localized functional $J_R$ with the strong upper gradient $\abs{\partial J_R}$: the curve $\eta(\bullet)$ is a curve of maximal slope for $J_R$ by the localization argument in the proof of \cref{thm:wellposedness_grad}, $\abs{\partial J_R}$ is a strong upper gradient by the $\lambda_R^J$-convexity of \cref{cor:conv_JR} and \cite[Corollary 2.4.10]{AGS}, and $J_R$ is lower semicontinuous.
Together with the precompactness established in \cref{prop:Orbit_in_Sobolev}, these verify all the hypotheses of \cite[Theorem 3.24]{Hauer19}, and $\eta(\tau)\to\eta^\ast$ in $\Ltwotheta$ as $\tau\to+\infty$.
\end{proof}

% \begin{corollary}[Convergence rates]\label{cor:rate}
%     Let $\eta$ be a curve of maximal slope that converges to $\eta^\ast$ in \cref{thm:global_conv}, and let $C$ and $\alpha$ be the constants in \cref{thm:grad_ineq}. 
%     Then there exists another constant $C^\prime>0$ such that
%     \begin{align}
%         W_2(\eta(\tau),\eta^\ast)_{L^2(I)}&{}\leq C^\prime\left\{
%         \begin{array}{ll}
%             \tau^{-\nicefrac{\alpha}{1-2\alpha}} &(\alpha\in(0,1/2)) \\
%             \exp(-\dfrac{\alpha\tau}{2C^2})&(\alpha=1/2)
%         \end{array}
%         \right..
%     \end{align}
% \end{corollary}
% This follows from \cite[Theorem 3.27]{Hauer19}.

\section{Examples, sharpness, and limitations}\label{sec:examples}
%%%%%%%%%%%%%%%%%
Throughout this section, $K_X\coloneqq\pi_{\R^d}(\supp\mu_0)$ denotes the projection of $\supp\mu_0$ onto $\R^d$ as in \cref{sec:grad_ineq}.
\begin{example}[An analytic vector field satisfying the derivative bounds]\label{ex:analytic_NN}
    The two-neuron field $v(x,\theta)=\tanh(\theta_1 x+\theta_2)-\tanh(\theta_3 x+\theta_4)$ with $d=1$ and $m=4$ is real-analytic and satisfies \eqref{eq:bound_dNN} with any $p\in[0,2)$.
    Other real-analytic activations satisfying the same derivative bounds, such as GELU~\cite{GELU}, SiLU~\cite{SiLU}, and Mish~\cite{misra2020mish}, can be treated similarly, whereas \textup{ReLU}~\cite{Nair10ReLU} is not real-analytic.
\end{example}
\subsection{Nonconvexity under the parameter-Hessian condition}
Bounds on the parameter Hessian of a model appear in the lazy-training analysis of \cite{NEURIPS2019_ae614c55} and, combined with the uniform conditioning of the tangent kernel, yield Polyak--\L{}ojasiewicz inequalities for over-parametrized square losses \cite{LIU202285}.
Here no conditioning of the tangent kernel is assumed, and \cref{assump:curvature} controls only the principal part of the second variation, so that the objective may remain genuinely nonconvex, as the following example shows.

\begin{example}[\Cref{assump:curvature} does not imply convexity]\label{ex:nonconvex}
    Take $d=m=1$, $Y=\{0\}$, $\mu_0=\delta_{(0,0)}$, $v(x,\theta)=\theta$, and $\ell(x,0)=c(1+\cos x)$ with $c>\epsilon$.
    Since $\pdv[2]{v}{\theta}=0$, the coercivity margin is $a_R=\epsilon>0$ for every $R$, so \cref{assump:curvature} holds.
    For the constant curve $\eta_t^s=\delta_s$ one has $X_1^{\eta^s}=s$ and
    \(
        J(\eta^s)=c(1+\cos s)+\frac{\epsilon}{2}s^2\), \( \eval{\dv[2]{s}J(\eta^s)}_{s=0}=\epsilon-c<0,
    \)
    so $J$ is nonconvex already along this one-parameter family.
    Hence the standing assumptions allow genuinely nonconvex objectives, and the \L{}ojasiewicz--Simon mechanism of \cref{sec:grad_ineq} is not redundant.
    In the borderline case $c=\epsilon$, the same family gives $J(\eta^s)-J(\eta^0)\sim\nicefrac{\epsilon s^4}{24}$ and $\abs{\partial J}(\eta^s)=\abs{\epsilon s-\epsilon\sin s}\sim\nicefrac{\epsilon\abs{s}^3}{6}$, so the exponent $\alpha=1/2$ is not forced by the assumptions.
\end{example}

\subsection{Checkable sufficient conditions}\label{subsec:checkable}
The margin \eqref{eq:coercivity_margin} admits an explicit lower bound.
By \eqref{eq:Gronwall_bound_X} and \eqref{eq:Gronwall_bound_M} with $\int_0^1m_p(\eta_t)\dd{t}\le1+R^2$ for $\eta\in\Dcal_R$, $\abs{X_t^{\eta}(x)}\le R_X$ and $\norm{M^{\eta}(x;1,t)}_{\opn}\le\e^{C_x(1+R^2)}$ for $x\in K_X$, where $R_X\coloneqq\qty(\sup_{x\in K_X}\abs{x}+C_0(2+R^2))\e^{C_0(2+R^2)}$, so that
\[
    a_R\ge\epsilon-C_{\theta\theta}\,\e^{C_x(1+R^2)}\qty(1+R_X^2)\sup_{\abs{z}\le R_X,y\in\Y}\abs{\nabla_x\ell(z,y)}.
\]
The right-hand side is computable from the constants in \cref{assump:bound_dNN}, the support of $\mu_0$, and the energy level, and its positivity at $R=R_E$ implies \cref{assump:curvature}.
In particular, for architectures linear in the parameters, as in \cref{ex:nonconvex}, one has $\pdv[2]{v}{\theta}=0$ and $a_R=\epsilon$ for every $R$, so \cref{assump:curvature} holds for every $\epsilon>0$.
Moreover, replacing $v$ by $\alpha v$ scales $C_0$, $C_x$ and $C_{\theta\theta}$ linearly in $\alpha$, so \cref{assump:curvature} holds for all sufficiently small $\alpha>0$, although this scaling changes the model.

\subsection{Failure of convergence without the parameter-Hessian condition}
Finally, we show that \cref{assump:curvature} cannot in general be omitted from \cref{thm:global_conv}.
Fix $\epsilon>0$ and define
\(
    \Psi(s)\coloneqq\frac{s}{(1+s^2)^{\nicefrac{1}{3}}}+\frac{s}{3(1+s^2)^{\nicefrac{2}{3}}}\), \(f_\epsilon(\theta)\coloneqq\epsilon\Psi(\theta^3-\theta),
\)
so that $\Psi$ is real-analytic, odd and strictly increasing, and $f_\epsilon(\theta)=\epsilon\theta+O(\abs{\theta}^{-5})$ as $\abs{\theta}\to\infty$.
Choose $C_\epsilon>0$ such that $b_\epsilon(\theta)\coloneqq C_\epsilon+\int_0^\theta(f_\epsilon(r)-\epsilon r)\dd{r}$ is nonnegative; $b_\epsilon$ is then bounded and real-analytic.
Take $d=3$, $m=1$, $\Y=\{0\}$, $\mu_0=\delta_{(0,0)}$, the network \(v(x,\theta)\coloneqq(b_\epsilon(\theta),\theta,f_\epsilon(\theta))\), and the loss \(\ell(x,0)\coloneqq x_1+\sqrt{1+(x_2x_3)^2}-x_2x_3\).
\begin{proposition}[A nonconvergent flow outside the parameter-Hessian regime]\label{prop:nonconvergence}
The functions $v$ and $\ell$ defined above are real-analytic and satisfy \cref{assump:bound_dNN} with $p=1$, and there exists an initial datum $\eta_0$ satisfying \cref{assump:initial_data} such that the curve of maximal slope for $J$ starting from $\eta_0$ does not converge in $\Ltwotheta$ as $\tau\to+\infty$; moreover, $a_R<0$ for every $R>0$.
\end{proposition}
\begin{proof}
\textbf{Verification of \cref{assump:bound_dNN}.}
Since $v$ is independent of $x$ and has linear growth in $\theta$, and $\abs{\ell(x,0)}\le C(1+\abs{x}^2)$, the bounds \eqref{eq:bound_dNN} hold with $p=1$.
Since $b_\epsilon\ge0$, every state generated by the model has a nonnegative first coordinate, while $\sqrt{1+s^2}-s>0$, so $L(\eta)\ge0$ for every $\eta\in\Ltwotheta$.

\textbf{Reduction to a nonlocal gradient flow.}
Consider laws that are constant in depth, $\eta_t\equiv\nu$, where $\nu$ has zero mean and compact support; any such law satisfies \cref{assump:initial_data}.
Set $B_\nu\coloneqq\int_{\R}b_\epsilon\dd{\nu}$ and $F_\nu\coloneqq\int_{\R}f_\epsilon\dd{\nu}$.
Since $v$ is independent of $x$, one has \(X_t^\eta=(tB_\nu,0,tF_\nu)\) and \(\nabla_x\varphi_t^\eta=(1,-F_\nu,0)\), and therefore \(\nabla_\theta\fdv{J}{\eta}[\eta](t,\theta)=b_\epsilon'(\theta)-F_\nu+\epsilon\theta=f_\epsilon(\theta)-F_\nu\).
The right-hand side is independent of $t$, and its mean against $\nu$ vanishes, so the depth-constant zero-mean class is invariant, and on this class the gradient flow is the nonlocal equation
\(
    \partial_\tau\nu_\tau=\partial_\theta\qty(\qty(f_\epsilon(\theta)-F_{\nu_\tau})\nu_\tau),
\)
which is studied by Park and Pego \cite{ParkPego25}.
% Along this equation, $\dv{\tau}J(\eta(\tau))=-\int_{\R}\abs{f_\epsilon-F_{\nu_\tau}}^2\dd{\nu_\tau}$ is nonpositive, so the solution remains in the initial energy sublevel; 
Choose $R>\sqrt{2J(\eta_0)/\epsilon}$.
By \cref{lem:localization,lem:maximal_slope=grad_flow} and the localization argument in the proof of \cref{thm:wellposedness_grad}, this solution is the curve of maximal slope for $J$ starting from $\eta_0$.

\textbf{Nonconvergence.}
The proof of \cite[Theorem 3]{ParkPego25} carries over to $f_\epsilon$, with the cubic-specific constants replaced by constants depending on $f_\epsilon$, once one verifies the two structural properties identified at the beginning of \cite[Section 5]{ParkPego25}: the graph of $f_\epsilon$ is N-shaped, and the three distinct roots of $f_\epsilon(z)=s$ satisfy a linear relation.
% is formulated for the cubic nonlinearity $z^3-z$, Park and Pego state at the beginning of \cite[Section 5]{ParkPego25} that their analysis extends to N-shaped nonlinearities for which the distinct roots of $f(z)=s$ satisfy a linear relation.
Indeed, since $\Psi'>0$, the sign of $f_\epsilon'(z)=\epsilon\Psi'(z^3-z)(3z^2-1)$ is that of $3z^2-1$, so the graph of $f_\epsilon$ is N-shaped; moreover, $f_\epsilon(z)=s$ if and only if $z^3-z=\Psi^{-1}(s/\epsilon)$, and hence the three distinct roots satisfy $z_l+z_m+z_r=0$.
Consequently, the initial-data construction in \cite[Section 5.3]{ParkPego25}, together with the ordering of the transition times and the phase-ratio argument in \cite[Sections 5.4--5.5]{ParkPego25}, yields a bounded zero-mean initial realization $u_0\in L^\infty(I)$, whose law $\nu_0\coloneqq(u_0)_\#\Leb$ is purely atomic and compactly supported, such that the corresponding solution $u(\bullet,\tau)$ does not converge in $L^2(I)$.
Since $u_0$ may be replaced by its monotone rearrangement and the order of particles is preserved, $u(\bullet,\tau)$ remains the quantile function of $\nu_\tau$, so $W_2(\nu_\tau,\nu_\sigma)=\norm{u(\bullet,\tau)-u(\bullet,\sigma)}_{L^2(I)}$; hence $\nu_\tau$ does not converge in $\Ptwo(\R)$, and $\eta(\tau)$ does not converge in $\Ltwotheta$.

\textbf{Negativity of the coercivity margin.}
Here $D_\theta^2v=(b_\epsilon'',0,f_\epsilon'')$, so at the zero parameter field \(A_{0}=(\epsilon+b_\epsilon''(0))\Id=f_\epsilon'(0)\Id=-\nicefrac{4\epsilon}{3}\Id\).
Since the zero field belongs to every ball, \cref{prop:H2} gives $a_R\le-\nicefrac{4\epsilon}{3}<0$ for every $R>0$.
Thus, the example satisfies all the assumptions of \cref{thm:global_conv} except \cref{assump:curvature}, while the trajectory does not converge.
\end{proof}

%%%%%%%%%%%%%%%%%%%%
% \clearpage
\appendix
\crefalias{section}{appendix}

\section{Calculus on Wasserstein-space valued \texorpdfstring{$L^2$}{L\texttwosuperior} functions}\label{sec:calc}
%%%%%%%%%%%%%%%%%%%%%%%%%%%%%%%%%%%%%%%%%%%%%%%
In this section, we adapt the calculus developed on $\Ptwo(\R^m)$ in \cite[Chapter 5]{Carmona2018} to $L^2(I;\Ptwo(\R^m))$, which is our focus.
We omit proofs nearly identical to those in \cite{Carmona2018}.

In the sequel, we use the fact that the Lebesgue measurable mapping $\eta_\bullet\in\Ltwotheta$ can be regarded as a probability measure $\int\eta_t\dd{t}$ on $I\times\R^m$ as indicated in \cite[Section 5.3]{AGS}.
This observation reveals that the topology of $\Ltwotheta$ induced by the distance is stronger than the narrow topology of $\Pcal(I\times\R^m)$.

\begin{lemma}\label{lem:eqiv_P_2}
    For any $\eta^1,\eta^2\in\Ltwotheta$, it holds that
    % \begin{equation}
    \(
        W_2\qty(\int\eta^1_t\dd{t},\int\eta^2_t\dd{t})\leq 
        W_2(\eta^1,\eta^2)_{L^2(I)}.
        % \label{eq:Wass_Jensen}
    % \end{equation}
    \)
    In particular, the closed balls in $\Ltwotheta$ are narrowly closed and compact in $\Pcal(I\times\R^m)$.
\end{lemma}
\begin{proof}
    The inequality is a consequence of the Jensen inequality.
    For the topological property, the second moments of the joint measures $\int\eta_t\dd{t}$ are uniformly bounded on a closed ball, so the ball is tight and hence narrowly relatively compact.
    To see that it is narrowly closed, let $\eta^n$ belong to the ball centered at $\eta^0$ and converge narrowly, under the identification $\eta\leftrightarrow\int\eta_t\dd{t}$, to $\eta$.
    Choose a measurable family $\gamma_t^n\in\OptPlan(\eta_t^n,\eta_t^0)$ and extract a narrow limit of the measures $\int\gamma_t^n\dd{t}$ by \cite[Proposition 7.1.5]{AGS}; its disintegration $\int\gamma_t\dd{t}$ satisfies $\gamma_t\in\Gamma(\eta_t,\eta_t^0)$ for a.e.~$t$, and the lower semicontinuity of the quadratic cost gives $\int_0^1W_2^2(\eta_t,\eta_t^0)\dd{t}\le\liminf_{n\to\infty}\int_0^1W_2^2(\eta_t^n,\eta_t^0)\dd{t}$.
\end{proof}

\subsection{Lifting and L-derivative}\label{subsec:Lcalculus}
In subsequent discussions, $(\Omega,\Fcal,\mathbb{P})$ is an atomless probability space; without loss of generality we take $\Omega=I\times I$, $\Fcal=\Bcal(I\times I)$, and $\mathbb{P}=\Leb\otimes\Leb$ the product Lebesgue measure, whose two coordinates serve as a center and an independent randomization in the lift of \cref{lem:fixed_center}.
For simplicity, we write $\Omega$ instead of $(\Omega,\Fcal,\mathbb{P})$. 
On this $\Omega$, the law of a random variable $X\in L^2(\Omega;\R^m)$ is $\Law X=X_\#\mathbb{P}\in\Ptwo(\R^m)$.  

In mean-field optimal control problems, a function $u\colon\Ptwo(\R^m)\to\R$ on the probability measure space $\Ptwo(\R^m)$ can be ``lifted'' to a function $\widetilde{u}\colon L^2(\Omega;\R^m)\to\R$ on the space of random variables $L^2(\Omega;\R^m)$ in order to apply functional-analytic arguments to $u$. 
To employ a similar approach for $J$ in \eqref{eq:def_J}, we adapt the L-derivative from $\Ptwo(\R^m)$ to $\Ltwotheta$ in this section.
Before introducing the L-derivative, we state that there are random variables $\varTheta_\bullet\in\varThetaSpace$ corresponding to $\eta_\bullet\in\Ltwotheta$.

\begin{lemma}[Lifting measures to random variables]\label{lem:existence_lift}
    For $\eta_\bullet\in\Ltwotheta$, there exists $\varTheta_\bullet\in L^2(I;L^2(\Omega;\R^m))$ such that $\Law\varTheta_t=\eta_t$ for a.e.~$t\in I$.
\end{lemma}
\begin{proof}
    By \cite[Lemma 5.29]{Carmona2018}, there is a measurable map $\psi\colon\Ptwo(\R^m)\to L^2(\Omega;\R^m)$ with $\Law\psi(\eta^\prime)=\eta^\prime$.
    Then $\varTheta\coloneqq\psi\circ\eta$ is measurable and $\norm{\varTheta}^2_{L^2(I;L^2(\Omega;\R^m))}=\int_0^1m_2(\eta_t)\dd{t}<+\infty$.
\end{proof}

For a function $E\colon\Ltwotheta\to\R$, the \emph{lifting} of $E$, which is denoted by $\tildeE$, is defined to be $\tildeE\colon L^2(I;L^2(\Omega;\R^m))\ni\varTheta\mapsto E(\Law(\varTheta_\bullet))\in\R$.

\begin{definition}[L-derivative]\label{def:Ldiff}
    A function $E$ is said to be \emph{L-differentiable} at $\eta\in\Ltwotheta$ if the lifting $\tildeE$ is Fr\'{e}chet differentiable at some $\varTheta\in L^2(I;L^2(\Omega;\R^m))$ such that $\Law\varTheta_t=\eta_t$ for a.e.~$t\in I$, and \emph{continuously L-differentiable} if $\tildeE$ is of class $C^1$ on $L^2(I;L^2(\Omega;\R^m))$.
\end{definition}
Throughout, we identify the Fr\'{e}chet derivative $D\tildeE(\varTheta)\in(L^2(I;L^2(\Omega;\R^m)))^\ast$ with its Riesz representative in $L^2(I;L^2(\Omega;\R^m))$.
\begin{proposition}[Structure of the L-derivative]\label{prop:structure_Ldiff}
    Let $E\colon\Ltwotheta\to\R$ be a proper and lower semicontinuous functional. 
    If $E$ is L-differentiable at $\eta_0\in\Ltwotheta$, then, for every $\varTheta\in L^2(I;L^2(\Omega;\R^m))$ with $\Law\varTheta_t=(\eta_0)_t$ for a.e.~$t\in I$, the lifting $\tildeE$ is differentiable at $\varTheta$.
    In addition, the joint law of $(\varTheta_t,(D\tildeE(\varTheta))_t)$ for a.e.~$t\in I$ is independent of the random variable $\varTheta$ as long as $\Law\varTheta_t=(\eta_0)_t$ for a.e.~$t\in I$.
    Furthermore, there exists a measurable mapping $\xi\colon I\times\R^m\to\R^m$ such that $((D\tildeE)(\varTheta))_t=\xi(t,\varTheta_t)$ for $\varTheta\in\varThetaSpace$ with $\Law\varTheta_t=(\eta_0)_t$ for a.e.~$t\in I$.
\end{proposition}
Based on \cref{prop:structure_Ldiff}, we write $\xi$ as $\nabla E[\eta]$ and call it the L-derivative of $E$.
This is the Bochner-space version of \cite[Proposition 5.25]{Carmona2018}.
% One first proves the assertion for simple-in-time random variables, where the finite-dimensional statement applies layer by layer; passing to the limit in $\varThetaSpace$ gives the general case.
% Since $\tildeE(\varTheta)=E(\Law\varTheta_\bullet)$ depends only on the law of the lift, the joint law of $(\varTheta,D\tildeE(\varTheta))$ is independent of the chosen lift, and hence $D\tildeE(\varTheta)_t$ can be represented as $\xi(t,\varTheta_t)$.
It follows first for simple-in-time lifts by applying the pointwise result layerwise, and then for general lifts by approximation in \(L^2(I;L^2(\Omega;\R^m))\); invariance under the choice of lift gives the representation \(\xi(t,\varTheta_t)\).

\subsection{Functional derivative}\label{subsec:Functional}
To explicitly compute the L-derivative, we use the linear functional derivative of \cref{def:fdv}.
\begin{proposition}[Connection between $\nabla$ and $\fdv{\eta}$]\label{prop:Ldiff_fdv}
Let $E$ be a functional on $\Ltwotheta$ admitting a linear functional derivative.
Suppose that the function $\R^m\ni\theta\mapsto\fdv{E}{\eta}[\eta](t,\theta)\in\R$ is differentiable for each $(\eta,t)$ and that the derivative $\nabla_\theta\fdv{E}{\eta}$ satisfies the following conditions:
    \begin{myenum}
        \item For a.e.~$t\in I$, $\nabla_\theta\fdv{E}{\eta}[\bullet](t,\bullet)\colon\Ltwotheta\times\R^m\to\R^m$ is continuous.\label{enum:5}
        \item For each $\eta\in\Ltwotheta$ and $\theta\in\R^m$, $\nabla_\theta\fdv{E}{\eta}[\eta](\bullet,\theta)\colon I\to\R^m$ is measurable.\label{enum:6}
        \item For any bounded subset $K\subset\Ltwotheta$, there exists $C_K>0$ such that $\abs{\nabla_\theta\fdv{E}{\eta}[\eta](t,\theta)}\le C_K(1+\abs{\theta})$ for a.e.~$t\in I$, every $\eta\in K$, and every $\theta\in\R^m$.\label{enum:7}
    \end{myenum}
Then, $E$ is continuously L-differentiable and
% \begin{equation}
\(
    \nabla E\qty[\eta](t,\theta) =\nabla_\theta\fdv{E}{\eta}\qty[\eta](t,\theta)
\)
    % \label{eq:Ldiff_fdv}
% \end{equation}
holds for any $\eta\in\Ltwotheta$, $\theta\in\R^m$, and a.e.~$t\in I$.
\end{proposition}
\cref{prop:Ldiff_fdv} follows the lines of \cite[Lemma B.1]{cardaliaguet2020splitting} and \cite[Proposition 5.48]{Carmona2018}.
% One writes \eqref{eq:Taylor_def} against the coupling $\Law(\varTheta_t,\varTheta^\prime_t)$ induced by a pair of lifts of $\eta$ and $\eta^\prime$ and applies the first-order Taylor expansion of $\fdv{E}{\eta}[\eta](t,\bullet)$ in the $\theta$-variable, the $1$-growth assumption \cref{enum:7} giving the required integrability.
% If $\varTheta_n\to\varTheta$ in $L^2(I;L^2(\Omega;\R^m))$, then $\Law\varTheta_n\to\Law\varTheta$ in $\Ltwotheta$; the continuity and $1$-growth of $\nabla_\theta\fdv{E}{\eta}$, together with the uniform integrability provided by the strong $L^2$-convergence, give the convergence of the lifted gradients by the Vitali convergence theorem, so the lifting is of class $C^1$.
Apply \eqref{eq:Taylor_def} to the coupling induced by a pair of lifts and use the first-order Taylor expansion in \(\theta\); the linear-growth bound gives the required integrability.
Continuity of the lifted gradient follows from Vitali's theorem under strong \(L^2\)-convergence.

\subsection{Wasserstein derivative}\label{subsec:Wdiff}
Identifying \(\eta\) with \(\int\eta_t\dd{t}\) shows that the metric \(W_2(\bullet,\bullet)_{L^2(I)}\) is the fibered quadratic Wasserstein distance of \cite{Peszek2023}.
The notions of tangent vector, subdifferential, generalized geodesic, and gradient flow used below are the corresponding fibered notions written in our notation.
For \(\eta,\eta^\prime\in\Ltwotheta\), let \(\OptPlan(\eta,\eta')\) denote the set of measurable families defined by
        \begin{equation}
        \begin{aligned}
            &\OptPlan(\eta,\eta^\prime)
            \coloneqq\Set{\pi\colon I\to\Pcal(\R^{2m});\text{measurable}|
            \begin{array}{c}
                 \pi_t\text{ is an optimal transport plan}   \\
                 \text{from $\eta_t$ to $\eta^\prime_t$ for a.e.~}t\in I
            \end{array}
            }.
        \end{aligned} 
        \label{eq:optimal_plans}
        \end{equation}
        This set is nonempty by \cite[Corollary 5.22]{villani_oldnew}.
\begin{definition}[Wasserstein derivative]\label{def:Wass_diff}
Let $\eta\in\Ltwotheta$ and $E$ be a functional on $\Ltwotheta$. 
Set a formal tangent space
        \[
        \Tan_\eta\Ltwotheta\coloneqq\overline{\Set{\nabla_\theta\varphi\colon I\to \Ccinf(\R^m)|\varphi\in\Ccinf(I\times\R^m)}}^{\norm{\bullet}_{L^2\qty(\int\eta_t\dd{t})}}.
        \]
    \begin{myenum}  
        \item A tangent vector $\xi\in\Tan_\eta\Ltwotheta$ belongs to $\partial^- E(\eta)$ if
        
        \[\hspace{-2em}
            \begin{aligned}
                E(\eta^\prime)-E(\eta)\geq{}&\inf_{\pi\in\OptPlan(\eta,\eta^\prime)}\iint_{I\times\R^{2m}}\la\xi_t(\theta^1),\theta^2-\theta^1\ra\dd{\pi_t(\theta^1,\theta^2)\dd t}+o(W_2(\eta^\prime,\eta)_{L^2(I)}),
            \end{aligned}
        \]
        for all $\eta^\prime\in\Ltwotheta$.
        \item A tangent vector $\xi\in\Tan_\eta\Ltwotheta$ belongs to the superdifferential $\partial^+ E(\eta)$ if $-\xi\in\partial^-(-E)(\eta)$.
        \item A functional $E$ is said to be Wasserstein-differentiable at $\eta$ if $\partial^- E(\eta)\neq\emptyset$ and $\partial^+ E(\eta)\neq\emptyset$.
    \end{myenum}
\end{definition}
This is the fibered Fr\'{e}chet subdifferential of \cite{Peszek2023}, restricted to its tangent representatives.
As stated in \cite[Proposition 5.63, Theorem 5.64]{Carmona2018}, the Wasserstein derivative is equal to the L-derivative:
\begin{proposition}[$\partial=\nabla$]\label{prop:L=Wdiff}
Let $E$ be a functional on $\Ltwotheta$.
    \begin{myenum}
        \item If $E$ is Wasserstein-differentiable at $\eta$, then $\partial^- E(\eta)=\partial^+E(\eta)$ and both are singletons.
        We denote the unique element by $\partial E[\eta]$.
        \item If $E$ is continuously L-differentiable, then $E$ is Wasserstein-differentiable at any $\eta\in\Ltwotheta$. 
        Furthermore, it holds that
        \(
            \nabla E\qty[\eta]=\partial E\qty[\eta].
        \)
    \end{myenum}
\end{proposition}
We briefly recall the argument for (2): choose a Borel family $t\mapsto\gamma_t\in\OptPlan(\eta_t,\eta^\prime_t)$ as in \eqref{eq:optimal_plans} and, by \cref{lem:fixed_center}, lifts $\varTheta$, $\varPhi\in\Hilbertsp$ with $\Law(\varTheta_t,\varPhi_t)=\gamma_t$ for a.e.~$t\in I$, so that $\norm{\varPhi-\varTheta}_{\Hilbertsp}^2=W_2^2(\eta,\eta^\prime)_{L^2(I)}$.
The argument of \cite[Proposition 5.63, Theorem 5.64]{Carmona2018}, applied to perturbations localized on measurable subsets of $I$, first for simple-in-time vector fields and then by density, gives $\nabla E[\eta]\in\Tan_\eta\Ltwotheta$.
The Fr\'{e}chet expansion of $\tildeE$ at $\varTheta$ and the representation of \cref{prop:structure_Ldiff} turn the linear term into $\int_0^1\int_{\R^{2m}}\la\nabla E[\eta](t,\theta^1),\theta^2-\theta^1\ra\dd{\gamma_t(\theta^1,\theta^2)}\dd{t}$, which yields $\nabla E[\eta]\in\partial^-E(\eta)\cap\partial^+E(\eta)$.
The uniqueness statement in (1) follows from the corresponding pointwise Wasserstein result \cite[Proposition 5.63, Theorem 5.64]{Carmona2018}, integrated in $t$.

We then introduce the convexity along generalized geodesics, based on \cite[Section 9.2]{AGS}.
Let us denote by $\pi^i$ and $\pi^{i,j}$ the canonical projections defined by
\begin{align}
    \pi^i&\colon\R^{3m}\ni(\theta^0,\theta^1,\theta^2)\longmapsto\theta^i\in\R^{m},\\
    \pi^{i,j}&\colon\R^{3m}\ni(\theta^0,\theta^1,\theta^2)\longmapsto(\theta^i,\theta^j)\in\R^{2m},
\end{align}
and we write $\pi^{i,j}_\#\boldsymbol{\eta}\coloneqq(\pi^{i,j}_\#\boldsymbol{\eta}_t)_{t\in I}\in L^2(I;\Ptwo(\R^{2m}))$ for a curve $\boldsymbol{\eta}\in L^2(I;\Ptwo(\R^{3m}))$.
\begin{definition}[Convexity along generalized geodesics]\label{def:generalized_convexity}
    Let $\eta^0$, $\eta^1$, $\eta^2\in\Ltwotheta$, and $\boldsymbol{\eta}\in L^2(I;\Ptwo(\R^{3m}))$ be a $3$-plan satisfying
    % \begin{align*}
        \(\pi^{0,1}_\#\boldsymbol{\eta}\in\OptPlan(\eta^0,\eta^1)\) and
        % &
        \(\pi^{0,2}_\#\boldsymbol{\eta}\in\OptPlan(\eta^0,\eta^2)\).
    % \end{align*}
    \begin{myenum}
        \item We define a generalized geodesic $\eta^{1\to 2}\coloneqq(\eta^{1\to 2}_\tau)_{\tau\in I}$ induced by $\boldsymbol{\eta}$ to be 
        \(
        \eta^{1\to 2}_\tau=\qty(\pi^{1\to 2}_\tau)_\#\boldsymbol{\eta}\in L^2(I;\Ptwo(\R^{m})),
        \)
        where $\pi^{1\to 2}_\tau\coloneqq(1-\tau)\pi^1+\tau\pi^2$.
        \item For a number $\lambda\in\R$, a function $E\colon\Ltwotheta\to(-\infty,+\infty]$ with the domain $D(E)$ is $\lambda$\emph{-convex along generalized geodesics} if, for every $\eta^0,\eta^1,\eta^2\in D(E)$, there exists a generalized geodesic $\eta^{1\to2}$ induced by a $3$-plan $\boldsymbol{\eta}$ such that
        \[
            E(\eta^{1\to2}_\tau)\leq (1-\tau)E(\eta^1)+\tau E(\eta^2)-\frac{\lambda}{2}\tau(1-\tau)W_{\boldsymbol{\eta}}^2(\eta^1,\eta^2)_{L^2(I)},
        \]
        where $W_{\boldsymbol{\eta}}^2(\eta^1,\eta^2)_{L^2(I)}\coloneqq\int_0^1\int_{\R^{3m}}\abs{\theta^2-\theta^1}^2\dd{\boldsymbol{\eta}_t(\theta^0,\theta^1,\theta^2)}\dd{t}\geq W_2^2(\eta^1,\eta^2)_{L^2(I)}$.\label{item:convexity}
    \end{myenum}
\end{definition}
\begin{remark}[Another characterization of convexity]\label{rmk:another_convexity}
    % If $I\ni \tau\mapsto E(\eta^{1\to2}_\tau)$ is $\lambda W_{\boldsymbol{\eta}}^2(\eta^1,\eta^2)_{L^2(I)}$-convex in the traditional sense, then $E$ is $\lambda$-convex in the sense of \cref{def:generalized_convexity}. 
    % Moreover, following the proof of \cite[Lemma 6]{Bonnet21intrinsic}, it is sufficient to show that $\tau\mapsto E(\eta^{1\to2}_\tau)$ is differentiable and $\tau\mapsto\dv{\tau}E(\eta^{1\to 2}_\tau)$ is Lipschitz continuous for $E$ to satisfy the convexity.
    % More precisely, if the derivative is Lipschitz continuous with a constant at most $CW_{\boldsymbol{\eta}}^2(\eta^1,\eta^2)_{L^2(I)}$, uniformly over the generalized geodesics of \cref{def:generalized_convexity}, then $E$ is $(-C)$-convex along generalized geodesics.
    By the argument of \cite[Lemma 6]{Bonnet21intrinsic}, if \(\tau\mapsto E(\eta_\tau^{1\to2})\) is differentiable and its derivative is Lipschitz with constant at most \(C W_{\boldsymbol\eta}^2(\eta^1,\eta^2)_{L^2(I)}\), uniformly in the generalized geodesic, then \(E\) is \((-C)\)-convex along the curve.
\end{remark}
When $E$ is $\lambda$-convex along generalized geodesics, the subdifferential $\partial^-E$ determines the slope $\abs{\partial E}$:
\begin{lemma}[Minimal selection]\label{lem:minimal_select}
    Let $E\colon\Ltwotheta\to(-\infty,+\infty]$ be proper, lower semicontinuous, and $\lambda$-convex along generalized geodesics. Then, the following are equivalent:
    \begin{myenum}
        \item $\eta\in D(\abs{\partial E})\coloneqq\Set{\eta^\prime\in\Ltwotheta|\abs{\partial E}(\eta^\prime)<+\infty}$.
        \item $\partial^-E(\eta)\neq\emptyset$ and
        \[
            \abs{\partial E}(\eta)=\min_{\xi\in\partial^- E(\eta)}\norm{\xi}_{L^2(I;L^2(\eta_t))},
        \]
        where
        $
            \norm{\xi}_{L^2(I;L^2(\eta_t))}\coloneqq(\int_0^1\int_{\R^m}\abs{\xi_t(\theta)}^2\dd{\eta_t(\theta)}\dd{t})^{\nicefrac{1}{2}}.
        $
    \end{myenum}
\end{lemma}
Under the preceding identification, the minimal-selection assertion follows from \cite[Proposition 3.32]{Peszek2023}; the converse implication follows from the subdifferential inequality and the Cauchy--Schwarz inequality.
The continuity-equation characterization of absolutely continuous curves in this fibered space is given in \cite[Proposition 3.21]{Peszek2023}; curves of maximal slope coincide with solutions of the following gradient flow equation:
\begin{definition}[A version of gradient flows in {\cite[Definition 3.33]{Peszek2023}}]\label{def:grad_flow_continuity}
    An absolutely continuous curve $\eta\colon(0,+\infty)\to\Ltwotheta$ is said to be \emph{a solution of the gradient flow equation} $v(\tau)\in-\partial^- E(\eta(\tau))$ if $\eta$ satisfies the following conditions:
    \begin{myenum}
        \item There exists a family of Borel vector fields $(v(\tau))_{\tau\in(0,+\infty)}$ such that \(v(\tau)\in\Tan_{\eta(\tau)}\Ltwotheta,\) for a.e.~$\tau\in(0,+\infty)$.
        \item $\norm{v(\tau)}_{L^2(I;L^2(\eta(\tau)_t))}\in L^2_{\textup{loc}}(0,+\infty)$.
        \item The continuity equation $\partial_\tau\eta(\tau)+\Div_\theta(v(\tau)\eta(\tau))=0$ holds in the sense of distributions, i.e., 
        \(
        \int_0^\infty\int_0^1\int_{\R^m}(\partial_\tau\phi+\nabla_\theta\phi\cdot v(\tau))\dd{\eta(\tau)_t}\dd{t}\dd{\tau}=0,
        \)
        for all $\phi\in\Ccinf((0,+\infty)\times I\times\R^m)$.
        \item $v(\tau)\in-\partial^-E(\eta(\tau))$ for a.e.~$\tau\in(0,+\infty)$.
    \end{myenum}
\end{definition}
\begin{lemma}[A version of {\cite[Theorem 3.34]{Peszek2023}}]\label{lem:maximal_slope=grad_flow}
 Let $E$ be as in \cref{lem:minimal_select}.
Then $\eta\in\AC^2(0,+\infty;\Ltwotheta)$ is a curve of maximal slope for $E$ with respect to $\abs{\partial E}$ if and only if $\eta$ is a solution of the gradient flow equation of $E$.
\end{lemma}
Under the preceding identification, the assertion follows from \cite[Theorem 3.34]{Peszek2023}.

%%%%%%%%%%%%%%%%%
\section{Existence of a minimizer and well-posedness of the flow}\label{sec:exist_wellposed}
\subsection{Existence of a minimizer}
First, we formulate the optimization problem precisely and establish the existence of a minimizer.
We work in the setting of \cref{subsec:intro_model}, where $\X=\R^d$, $\Y\subset\R^d$ is compact, $\mu_0\in\Pc(\XtimesY)$, and $X^\eta$, $L$, and $J$ are defined by \eqref{eq:forward_flow} and \eqref{eq:loss_term_L}--\eqref{eq:J} with $\epsilon>0$.
The optimization of the continuous DNN is the minimization problem
\begin{equation}
    \minimize_{\eta\in\Ltwotheta} J(\eta).
    \label{eq:def_J}
\end{equation}

The existence of a minimizer requires neither the analyticity nor the second-order bounds in \cref{assump:bound_dNN}; the following conditions suffice.

\begin{assumption}[Conditions for the existence of a minimizer]\label{assump:NN}
The functions $v$ and $\ell$ are continuous, $\abs{\ell(x,y)}\le A+B\abs{x}^2$ on $\XtimesY$ for some $A$, $B>0$, and $L\ge0$ on $\Ltwotheta$.
Moreover, there exist $C>0$ and $p\in[0,2)$ such that
\begin{align}
    \abs{v(x,\theta)}&{}\leq C(1+\abs{\theta}^{p})(1+\abs{x}),\\
    \abs{v(x^1,\theta)-v(x^2,\theta)}&{}\leq C(1+\abs{\theta}^2)\abs{x^1-x^2},
\end{align}
for every $(x,\theta)\in\R^d\times\R^m$ and $x^1,x^2\in\R^d$.
\end{assumption}

The subquadratic growth condition on $v$ is used in the coercivity argument below.

\begin{theorem}[Existence of a minimizer for \eqref{eq:def_J}]\label{thm:ex_minima}
Under \cref{assump:NN}, there exists a minimizer $\eta^\ast\in\Ltwotheta$ of $J$ defined in \eqref{eq:J}.
\end{theorem}

\begin{proof}
    Let us first take a minimizing sequence $(\eta^n)_n$ of $J$ in $\Ltwotheta$ and set $\mu^n\coloneqq\mu^{\eta^n}$ to be the solution of \eqref{eq:ODE2} in $C(I;\Pc(\XtimesY))$ associated with each $\eta^n$.
    The regularization term bounds $\int_0^1m_2(\eta^n_t)\dd{t}$ uniformly in $n$, so the supports $\supp\mu^n_t$ lie in a common compact set by \cref{lem:supp_bound}, on which $W_2^2\le CW_1$.
    By \eqref{eq:ODE2} and the growth condition on $v$ in \cref{assump:NN},
    \[
        W_1\qty(\mu^n_t,\mu^n_s)\le C\int_s^t\qty(1+m_p(\eta^n_r))\dd{r}\le C\qty(\abs{t-s}+\abs{t-s}^{1-\nicefrac{p}{2}}\qty(\int_s^t m_2(\eta^n_r)\dd{r})^{\nicefrac{p}{2}})
    \]
    for $s\le t$ by the H\"{o}lder inequality, so the curves $(\mu^n)_n$ are equicontinuous.
    By virtue of the Benamou--Brenier formula~\cite{Benamou2000}, the Arzel\`{a}--Ascoli theorem, and \cref{lem:eqiv_P_2}, there exist a subsequence of $(n)$, still denoted by $n$, and a pair $(\mu^\ast,\eta^\ast)$ with $\mu^\ast\in C(I;\Ptwo(\XtimesY))$ and $\eta^\ast\in\Ltwotheta$ such that
    \begin{align}
        \mu^{n}\to \mu^\ast &\text{ strongly in }C \qty(I;\Ptwo(\XtimesY)),\label{eq:mu_uniform} \\ 
        \eta^{n}\to\eta^\ast &\text{ narrowly in }\Pcal(I\times\R^m).\label{eq:eta_narrow}  
    \end{align}
    The loss term converges by the convergence \eqref{eq:mu_uniform} of $\mu^n_1$ in $\Ptwo(\XtimesY)$ and the continuity of $\ell$ with $2$-growth, while the regularization term is lower semicontinuous under the narrow convergence \eqref{eq:eta_narrow} by \cite[Lemma 5.1.7]{AGS}; hence
    \(
        \int_{\XtimesY}\ell\dd{\mu^\ast_1}+\nicefrac{\epsilon}{2}\iint_{I\times\R^m}\abs{\theta}^2\dd{\eta_t^\ast(\theta)}\dd{t}\le\inf_{\Ltwotheta} J.
    \)
    We next show that $\mu^\ast=\mu^{\eta^\ast}$, i.e.,
    \[
    \begin{aligned}
        &\int_{0}^{1} \int_{\XtimesY}\partial_{t} \phi(t,x,y)\dd{\mu_t^{\ast}(x,y)}\dd{t}\\
        &\quad+\int_{0}^{1} \int_{\XtimesY}\int_{\R^m} \nabla_{x} \phi(t,x,y) \cdot v(x,\theta)\dd{\eta^\ast_t(\theta)}\dd{\mu_t^{\ast}(x,y)}\dd{t}=0,
    \end{aligned}
    \]
    for all $\phi\in \Ccinf((0,1)\times\XtimesY)$.
    Since the weak formulation of \eqref{eq:ODE2} holds for each $(\mu^n,\eta^n)$ and the nonlinear term splits, with $f=\nabla_x\phi\cdot v$, as $\int f\dd{\eta^n}\dd{\mu^n}-\int f\dd{\eta^\ast}\dd{\mu^\ast}=\int f\dd{(\eta^n-\eta^\ast)}\dd{\mu^\ast}+\int f\dd{\eta^n}\dd{(\mu^n-\mu^\ast)}$, it suffices to show that
    \begin{equation}
        \lim_{n\to\infty}\int_0^1\int_{\XtimesY}\int_{\R^m}\nabla_x\phi\cdot v\dd{(\eta^n_t-\eta^\ast_t)}\dd{\mu^\ast _t}\dd{t}=0\label{eq:I_3}
    \end{equation}
    and
    \begin{equation}
        \lim_{n\to\infty}\int_0^1\int_{\XtimesY}\int_{\R^m}\nabla_x\phi\cdot v\dd{(\mu^n_t-\mu^\ast_t)}\dd{\eta^n_t}\dd{t}=0.\label{eq:I_4}
    \end{equation}
    
    As noted above, the $2$-moments of $(\eta^n)_n$ are uniformly bounded, so the family has uniformly integrable $p$-moments for the exponent $p\in[0,2)$ in \cref{assump:NN}; together with \cite[Lemma 5.1.7]{AGS} and the narrow convergence \eqref{eq:eta_narrow}, this yields \eqref{eq:I_3}.
    
    For \eqref{eq:I_4}, the uniform convergence \eqref{eq:mu_uniform} and the Kantorovich--Rubinstein duality $W_1(\mu,\nu)=\sup_{\norm{\varphi}_{\textup{Lip}}\leq1}\int_X\varphi\dd{(\mu-\nu)}$ \cite[Theorem 1.14]{VillaniTopic}, with $\norm{\varphi}_{\textup{Lip}}$ the Lipschitz constant of $\varphi$, yield
    \begin{align}
        &\abs{\int_0^1\int_{\XtimesY}\int_{\R^m}\nabla_x\phi\cdot v\dd{(\mu^n_t-\mu^\ast_t)}\dd{\eta^n_t}\dd{t}}\\
        \leq{}&\int_{0}^1\int_{\R^m}\norm{\nabla_x\phi\cdot v}_{\textup{Lip}}W_1(\mu^n_t,\mu^\ast_t)\dd{\eta^n_t(\theta)}\dd{t}\\
        \leq{}&C\int_{0}^1W_1(\mu^n_t,\mu^\ast_t)\int_{\R^m}(1+\abs{\theta}^2)\dd{\eta^n_t(\theta)}\dd{t}\\
        \leq{}&C\sup_{s\in I}W_1(\mu^n_s,\mu^\ast_s)\int_{0}^1\int_{\R^m}(1+\abs{\theta}^2)\dd{\eta^n_t(\theta)}\dd{t}\leq C\sup_{s\in I}W_2(\mu^n_s,\mu^\ast_s)\to0,
    \end{align}
    as $n\to\infty$.
    The initial condition $\mu_0^\ast=\mu_0$ follows from the convergence \eqref{eq:mu_uniform}, and the uniqueness of the solution of \eqref{eq:ODE2} yields $\mu^\ast=\mu^{\eta^\ast}$.
    Hence the left-hand side above equals $J(\eta^\ast)\ge\inf_{\Ltwotheta}J$, and $\eta^\ast$ is a minimizer.
\end{proof}

\subsection{Well-posedness}\label{subsec:wellposed_proofs}
We prove the results stated in \cref{subsec:wellposed}.
\begin{proof}[Proof of \cref{cor:conv_JR}]
    The map $\eta\mapsto\int_0^1m_2(\eta_t)\dd{t}=W_2(\eta,\delta_{0,\bullet})_{L^2(I)}^2$, where $\delta_{0,\bullet}$ denotes the constant curve at the Dirac mass at the origin, is continuous on $\Ltwotheta$, so $\Dcal_R$ is closed in $\Ltwotheta$ and $J_R$ inherits properness and lower semicontinuity from $J$; since $J_R\ge0$, the coercivity condition of \cite[Definition 3.30]{Peszek2023} holds.
    For the convexity, let $\eta^{1\to2}$ be a generalized geodesic between $\eta^1$ and $\eta^2$.
    If $\eta^1\notin\Dcal_R$ or $\eta^2\notin\Dcal_R$, the defining inequality holds trivially.
    Otherwise, since $\abs{(1-\tau)\theta^1+\tau\theta^2}^2\leq(1-\tau)\abs{\theta^1}^2+\tau\abs{\theta^2}^2$, the geodesic remains in $\Dcal_R$; along it, $L$ is $\lambda_R$-convex by \cref{lem:conv_J} and the regularization term $\nicefrac{\epsilon}{2}\int_0^1m_2(\eta_t)\dd{t}$ is $\epsilon$-convex, since its second derivative along $\eta^{1\to2}(\tau)$ equals $\epsilon\int_0^1\int_{\R^{3m}}\abs{\theta^2-\theta^1}^2\dd{\boldsymbol{\eta}_t}\dd{t}=\epsilon W_{\boldsymbol{\eta}}^2(\eta^1,\eta^2)_{L^2(I)}$.
\end{proof}

\begin{proof}[Proof of \cref{lem:localization}]
    By the triangle inequality for $W_2(\bullet,\bullet)_{L^2(I)}$ with the curve $\delta_{0,\bullet}$, every $\zeta\in\Ltwotheta$ with $W_2(\zeta,\eta)_{L^2(I)}<R-(\int_0^1m_2(\eta_t)\dd{t})^{\nicefrac{1}{2}}$ belongs to $\Dcal_R$, so $J_R=J$ on this ball.
    The slope in \cref{subsec:grad_flow} is defined through the values of the functional in an arbitrarily small ball around $\eta$, and the subdifferential $\partial^-$ in \cref{sec:calc} through a first-order expansion with an $o(W_2(\zeta,\eta)_{L^2(I)})$ remainder; hence both are determined by the functional near $\eta$, and the equalities follow.
\end{proof}

\begin{proof}[Proof of \cref{thm:wellposedness_grad}]
    Fix $R>\sqrt{2J(\eta_0)/\epsilon}$.
    By \cref{cor:conv_JR}, the functional $J_R$ is proper, lower semicontinuous, coercive and $\lambda_R^J$-convex along generalized geodesics.
    Thus, by \cite[Theorem 3.35(i)]{Peszek2023}, the localized functional $J_R$ admits a unique gradient flow starting from $\eta_0$, which is a curve of maximal slope $\eta(\bullet)$ for $J_R$ with respect to $\abs{\partial J_R}$ by \cref{lem:maximal_slope=grad_flow}.
    Since $\tau\mapsto J_R(\eta(\tau))$ is nonincreasing, the curve stays in the sublevel set
    \(
        \Set{\eta\in\Ltwotheta|J(\eta)\le J(\eta_0)}\subset\Dcal_{\sqrt{2J(\eta_0)/\epsilon}},
    \)
    and every point of this sublevel set satisfies the strict bound in \cref{lem:localization}.
    Hence $J_R(\eta(\tau))=J(\eta(\tau))$ and $\abs{\partial J_R}(\eta(\tau))=\abs{\partial J}(\eta(\tau))$ for all $\tau\ge0$, so the energy-dissipation inequality defining the curve of maximal slope for $J_R$ coincides with the one for $J$, and $\eta(\bullet)$ is a curve of maximal slope for $J$ with respect to $\abs{\partial J}$.

    Conversely, since $J$ is nonincreasing along every curve of maximal slope for $J$, any such curve starting from a point with $J<\nicefrac{\epsilon R^2}{2}$ satisfies the strict bound in \cref{lem:localization} for all $\tau$, so it is a curve of maximal slope for $J_R$; by \cref{lem:maximal_slope=grad_flow}, it solves the gradient flow equation of $J_R$.
    Now let $\eta^1$ and $\eta^2$ be curves of maximal slope for $J$ starting from $\eta_0^1$ and $\eta_0^2$, and take $R>\max_{i=1,2}\sqrt{2J(\eta_0^i)/\epsilon}$.
    The contraction estimate then follows from \cite[Theorem 3.35(iii)]{Peszek2023}, which proves the uniqueness.
\end{proof}

\section{Technical proofs}\label{sec:technical_proofs}
\subsection{\texorpdfstring{$\lambda$}{lambda}-convexity of \texorpdfstring{$L$}{L}}\label{appendix:lambda_convex}
\begin{proof}[Proof of \cref{lem:conv_J}]
    Let $\eta^1$, $\eta^2\in\Dcal_R$, and let $\eta^{1\to2}$ be a generalized geodesic induced by a $3$-plan $\boldsymbol{\eta}$; as in the proof of \cref{cor:conv_JR}, $\eta^{1\to2}(\tau)\in\Dcal_R$ for every $\tau\in I$, and $\int_0^1\int_{\R^{3m}}(1+\abs{\theta^1}+\abs{\theta^2})^2\dd{\boldsymbol{\eta}_t}\dd{t}\le C(1+R^2)$.
    We prove that $I\ni\tau\mapsto L(\eta^{1\to2}(\tau))$ is differentiable with a Lipschitz continuous derivative (as in \cite[Lemma B.2]{NEURIPS2018_a1afc58c}). \cref{prop:grad} yields the differentiability, and the derivative is explicitly computed as follows:
    \begin{align}
    &\dv{L\qty(\eta^{1\to2}\qty(\tau))}{\tau}{}=\underbrace{\iint\limits_{I\times\R^{3m}}\la\nabla_x\varphi_t^{\eta^{1\to2}(\tau)},(\mdif{\theta} v)_{(1-\tau)\theta^1+\tau\theta^2}\qty[\theta^2-\theta^1]\ra_{\mu_t^{\eta^{1\to2}(\tau)}}\dd{\boldsymbol{\eta}_t\dd t}}_{\eqqcolon h(\tau)},
    \end{align}
    based on \cref{def:Wass_diff}. 
    Therefore, we need to prove that there exists a constant $C_R>0$, depending only on $R$, such that 
    \[
        \abs{h\qty(\tau^2)-h\qty(\tau^1)}\leq\int_0^1\int_{\R^{3m}}\qty(\abs{\one}+\abs{\two}+\abs{\three})\dd{\boldsymbol{\eta}_t}\dd{t}\leq C_RW^2_{\boldsymbol{\eta}}(\eta^1,\eta^2)_{L^2(I)}\abs{\tau^2-\tau^1},
    \]
    for $\tau^1$, $\tau^2\in I$, where
    \begin{align}
    \one\coloneqq&\la\nabla\qty(\varphi_t^{\eta(\tau^2)}-\varphi_t^{\eta(\tau^1)}),(\mdif{\theta} v)_{\theta(\tau^2)}\qty[\theta^2-\theta^1]\ra_{\mu_t^{\eta(\tau^2)}},\label{eq:first_h}\\
    \two
        \coloneqq&\la\nabla_x\varphi_t^{\eta(\tau^1)},\qty((\mdif{\theta} v)_{\theta(\tau^2)}-(\mdif{\theta} v)_{\theta(\tau^1)})\qty[\theta^2-\theta^1]\ra_{\mu_t^{\eta(\tau^2)}},\label{eq:second_h}\\
    \three
        \coloneqq&\la\nabla_x\varphi_t^{\eta(\tau^1)},(\mdif{\theta} v)_{\theta(\tau^1)}\qty[\theta^2-\theta^1]\ra_{\mu_t^{\eta(\tau^2)}-\mu_t^{\eta(\tau^1)}},\label{eq:third_h}
    \end{align}
    and for simplicity we write $\eta(\tau^i)\coloneqq\eta^{1\to2}(\tau^i)$ and $\theta(\tau^i)\coloneqq(1-\tau^i)\theta^1+\tau^i\theta^2$ for $i=1,2$.
    Hereafter, we evaluate the Lipschitz continuity in three parts.
    From \cref{lem:supp_bound} and $\eta(\tau^i)\in\Dcal_R$, there exists a ball $B_{\XtimesY}$, depending only on $R$, which includes the supports of $\mu^{\eta(\tau^i)}_t$ for each $t\in I$ and $i=1,2$; we write $B_X$ for the projection of $B_{\XtimesY}$ onto $\R^d$.
    Thus, for $\one$ of \eqref{eq:first_h}, we obtain
    \begin{align}
        &\int_0^1\int_{\R^{3m}}\abs{\one}\dd{\boldsymbol{\eta}_t}\dd{t}\\
        \leq{}&\sup_{\substack{s\in I,\\
        (x,y)\in B_{\XtimesY}}}\abs{\nabla\qty(\varphi_s^{\eta(\tau^2)}-\varphi_s^{\eta(\tau^1)})(x,y)}\\
        &\times{\int_0^1\int_{\R^{3m}}\underbrace{\abs{(\mdif{\theta} v)(x,\theta(\tau^2))}}_{\leq C_R(1+\abs{\theta^1}+\abs{\theta^2})\text{ by \cref{assump:bound_dNN}}}\abs{\theta^2-\theta^1}\dd{\boldsymbol{\eta}_t(\theta^0,\theta^1,\theta^2)}\dd{t}}\\
        \leq{}&C_RW_{\boldsymbol{\eta}}(\eta^1,\eta^2)_{L^2(I)}\sup_{\substack{s\in I,\\
        (x,y)\in B_{\XtimesY}}}\abs{\nabla\qty(\varphi_s^{\eta(\tau^2)}-\varphi_s^{\eta(\tau^1)})(x,y)},
    \end{align}
    where the last step uses the Cauchy--Schwarz inequality and the moment bound on $\Dcal_R$, so that a single factor $W_{\boldsymbol{\eta}}(\eta^1,\eta^2)_{L^2(I)}$ appears; the second factor will contribute the other one.
    The difference $\varphi_t^{\eta(\tau^2)}-\varphi_t^{\eta(\tau^1)}$ is computed as 
    \begin{align}
        &\qty(\varphi_t^{\eta(\tau^2)}-\varphi_t^{\eta(\tau^1)})(x,y)\\
        ={}&\ell(X_{t,1}^{\eta(\tau^2)}(x),y)-\ell(X_{t,1}^{\eta(\tau^1)}(x),y)\\
        ={}&\int_0^1\nabla_x\ell\qty(
        \qty(\qty(1-s)X_{t,1}^{\eta\qty(\tau^2)}+sX_{t,1}^{\eta\qty(\tau^1)})\qty(x)
        ,y)\cdot\qty(X_{t,1}^{\eta(\tau^2)}-X_{t,1}^{\eta(\tau^1)})(x)\dd{s}\\
        ={}&(\tau^2-\tau^1)\iint_{I^2}\nabla_x\ell\qty(
        \qty(\qty(1-s)X_{t,1}^{\eta(\tau^2)}+sX_{t,1}^{\eta(\tau^1)})(x)
        ,y)\\
        &\cdot{\eval{\dv{X_{t,1}^{\eta(\tau)}(x)}{\tau}}_{\tau=(1-u)\tau^1+u\tau^2}\dd{u\dd s}},
    \end{align}
    using the fundamental theorem of calculus.
    As in \cref{lem:linearizedODE}, the derivative $\dv{\tau}X_{t,1}^{\eta(\tau)}$ is calculated as follows:
    \begin{align}
        \dv{\tau}X_{t,1}^{\eta(\tau)}(x)&=\int_t^1\int_{\R^{3m}}M^{\eta(\tau)}(x;1,s)(\mdif{\theta} v)(X_{t,s}^{\eta(\tau)}(x),(1-\tau)\theta^1+\tau\theta^2)\\
        &\quad\times(\theta^2-\theta^1)\dd{\boldsymbol{\eta}_s(\theta^0,\theta^1,\theta^2)}\dd{s},
    \end{align}
    where the integration is over $s\in[t,1]$ because the flow $X_{t,\bullet}^{\eta(\tau)}(x)$ starts from $x$ at layer $t$, so the $\tau$-perturbation acts only on the layers after $t$, and $M^{\eta(\tau)}(x;1,s)$ is the fundamental solution from $s$ to $1$ of the linearized ODE
    \(
        \dot{z}_t=\int(\mdif{x}v)(X_t^{\eta(\tau)}(x),\theta)z_t\dd{\eta(\tau)_t(\theta)},
    \)
    which is bounded uniformly for $x\in B_X$ by \eqref{eq:Gronwall_bound_M}.
    This formula gives
    \(
        \abs{\dv{\tau}X_{t,1}^{\eta(\tau)}(x)}\le C_RW_{\boldsymbol{\eta}}(\eta^1,\eta^2)_{L^2(I)}
    \)
    by the Cauchy--Schwarz inequality and the moment bound on $\Dcal_R$, hence
    \(
        \sup_{t\in I,x\in B_X}\abs{X_{t,1}^{\eta(\tau^2)}(x)-X_{t,1}^{\eta(\tau^1)}(x)}\le C_R\abs{\tau^2-\tau^1}W_{\boldsymbol{\eta}}(\eta^1,\eta^2)_{L^2(I)}.
    \)
    The difference $M^{\eta(\tau^2)}(x;1,t)-M^{\eta(\tau^1)}(x;1,t)$ of the fundamental matrices obeys the same type of bound: it solves the linearized ODE for $\tau^2$ with the source given by the difference of the coefficients applied to $M^{\eta(\tau^1)}(x;1,t)$.
    This coefficient difference is bounded through \eqref{eq:Cx} by the state difference above and by
    \(
        \abs{\theta(\tau^2)-\theta(\tau^1)}=\abs{\tau^2-\tau^1}\abs{\theta^2-\theta^1}
    \),
    so integrating against $\boldsymbol{\eta}$ and applying the Gronwall inequality yield
    \[
        \sup_{t\in I,x\in B_X}\norm{M^{\eta(\tau^2)}(x;1,t)-M^{\eta(\tau^1)}(x;1,t)}_{\opn}\le C_R\abs{\tau^2-\tau^1}W_{\boldsymbol{\eta}}(\eta^1,\eta^2)_{L^2(I)}.
    \]
    Since $\nabla_x\varphi_t^{\eta(\tau)}(x,y)=M^{\eta(\tau)}(x;1,t)^\top\nabla_x\ell(X_{t,1}^{\eta(\tau)}(x),y)$ by \cref{prop:grad}, combining the two displays with the boundedness of $\nabla_x\ell$ and $\Hessian_x\ell$ on the compact set gives
    \[
        \sup_{\substack{t\in I,(x,y)\in B_{\XtimesY}}}\abs{\nabla\qty(\varphi_t^{\eta(\tau^2)}-\varphi_t^{\eta(\tau^1)})(x,y)}\leq C_R\abs{\tau^2-\tau^1}W_{\boldsymbol{\eta}}(\eta^1,\eta^2)_{L^2(I)},
    \]
    and therefore $\int_0^1\int_{\R^{3m}}\abs{\one}\dd{\boldsymbol{\eta}_t}\dd{t}\le C_R\abs{\tau^2-\tau^1}W^2_{\boldsymbol{\eta}}(\eta^1,\eta^2)_{L^2(I)}$.
    
    Subsequently, we evaluate $\abs{\two}$ of \eqref{eq:second_h}. 
    Using \cref{lem:supp_bound}, we obtain
    \begin{align}
        &\int_0^1\int_{\R^{3m}}\abs{\two}\dd{\boldsymbol{\eta}_t}\dd{t}\\
        \leq{}&\int_0^1\int_{\R^{3m}}\abs{\la\nabla_x\varphi_t^{\eta(\tau^1)},\qty((\mdif{\theta} v)_{\theta(\tau^2)}-(\mdif{\theta} v)_{\theta(\tau^1)})\qty[\theta^2-\theta^1]\ra_{\mu_t^{\eta(\tau^2)}}}\dd{\boldsymbol{\eta}_t}\dd{t}\\
        \leq{}&\abs{\tau^2-\tau^1}\int_0^1\int_{\R^{3m}}\int_{\XtimesY}\abs{\nabla_x\varphi_t^{\eta(\tau^1)}(x,y)}\underbrace{\int_0^1\norm{\qty(D^2_\theta v)(x,{(1-s)\theta(\tau^1)+s\theta(\tau^2)})}_{\opn}\dd{s}}_{\leq C_{\theta\theta}(1+\abs{x}^2)\text{ by \cref{assump:bound_dNN}}}\\
        &\pushright{\abs{\theta^2-\theta^1}^2\dd{\mu_t^{\eta(\tau^2)}}\dd{\boldsymbol{\eta}_t}\dd{t}}\\
        \leq{}&C_R\abs{\tau^2-\tau^1}\sup_{\substack{s\in I,\\
        (x,y)\in B_{\XtimesY}}}\abs{\nabla_x\varphi_s^{\eta(\tau^1)}}\int_0^1\int_{\R^{3m}}\abs{\theta^2-\theta^1}^2\dd{\boldsymbol{\eta}_t}\dd{t}\\
        \leq{}&C_R W^2_{\boldsymbol{\eta}}(\eta^1,\eta^2)_{L^2(I)}\abs{\tau^2-\tau^1}.
    \end{align}

    Finally, we bound \eqref{eq:third_h} using $\mu_t^{\eta(\tau^i)}=(X_t^{\eta(\tau^i)}\times\Id_\Y)_\#\mu_0$, which pairs the two measures through the common initial datum and uses no regularity of $\ell$ in $y$:
    \begin{align}
        &\int_0^1\int_{\R^{3m}}\abs{\three}\dd{\boldsymbol{\eta}_t}\dd{t}\\
        \leq{}& \sup_{t\in I,x\in B_X}\abs{X_t^{\eta(\tau^2)}(x)-X_t^{\eta(\tau^1)}(x)}\\
        &\times\int_0^1\int_{\R^{3m}}\sup_{(x,y)\in B_{\XtimesY}}\abs{\nabla_x\qty(\qty(\nabla_x\varphi_t^{\eta(\tau^1)}(x,y))^\top(\mdif{\theta} v)(x,\theta(\tau^1))(\theta^2-\theta^1))}\dd{\boldsymbol{\eta}_t}\dd{t}\\
        \leq{}& C_R\abs{\tau^2-\tau^1}W^2_{\boldsymbol{\eta}}(\eta^1,\eta^2)_{L^2(I)}.
    \end{align}
    Indeed, the $x$-derivative of the integrand is bounded, through \eqref{eq:Cth}, \eqref{eq:Cx} and the boundedness of $\nabla_x\varphi_t^{\eta(\tau^1)}$ and $\Hessian_x\varphi_t^{\eta(\tau^1)}$ on $B_{\XtimesY}$ from \cref{assump:bound_dNN}, by $C_R(1+\abs{\theta^1}+\abs{\theta^2})\abs{\theta^2-\theta^1}$, whose integral against $\boldsymbol{\eta}$ is at most $C_RW_{\boldsymbol{\eta}}(\eta^1,\eta^2)_{L^2(I)}$ by the Cauchy--Schwarz inequality and the moment bound on $\Dcal_R$; and, analogously to the terminal-state formula above, the flow started at layer $0$ satisfies
    \begin{align}
        \dv{\tau}X_t^{\eta(\tau)}(x)&=\int_0^t\int_{\R^{3m}}M^{\eta(\tau)}(x;t,s)(\mdif{\theta} v)(X_s^{\eta(\tau)}(x),\theta(\tau))(\theta^2-\theta^1)\dd{\boldsymbol{\eta}_s}\dd{s},
    \end{align}
    with integration over $s\in[0,t]$, so that
    \[
        \sup_{t\in I,x\in B_X}\abs{X_t^{\eta(\tau^2)}(x)-X_t^{\eta(\tau^1)}(x)}\le C_R\abs{\tau^2-\tau^1}W_{\boldsymbol{\eta}}(\eta^1,\eta^2)_{L^2(I)}.
    \]
    Consequently, $h$ is Lipschitz continuous with constant $C_RW^2_{\boldsymbol{\eta}}(\eta^1,\eta^2)_{L^2(I)}$, and \cref{rmk:another_convexity} yields the $\lambda_R$-convexity with $\lambda_R=-C_R$.
\end{proof}

\subsection{Strong compactness of the orbit of the gradient flow}\label{sec:compactness}
\begin{proof}[Proof of \cref{prop:Orbit_in_Sobolev}]
    As shown in \cref{prop:orbit_in_L^infty}, $\eta(\tau)$ has support on a bounded set $D\subset\R^m$ with diameter independent of $\tau$.
    By virtue of \cite[Proposition 6.4]{LAVENANT2019688}, it is sufficient to bound the Dirichlet energy
    \begin{equation}
        \Dir(\eta(\tau))=\lim_{\varepsilon\searrow0}\frac{C_1}{\varepsilon^{3}}\iint_{I_\varepsilon}W_2^2(\eta(\tau)_t,\eta(\tau)_s)\dd{t\dd s},\label{eq:Dirichlet}
    \end{equation}
    by a constant independent of $\tau$, where $I_\varepsilon\coloneqq\Set{(t,s)\in I^2|\abs{t-s}\leq\varepsilon}$ and $C_1>0$ is the constant defined in \cite[Definition 3.24]{LAVENANT2019688}.
    The integrand on the right-hand side of \eqref{eq:Dirichlet} is split, with the flow map $\Phi^{\R^m}$ of \eqref{eq:integral_eq_nonlocal}, as
    \begin{align}
        W_2(\eta(\tau)_t,\eta(\tau)_s)\leq{}&W_2\qty(\Phi^{\R^m}_\tau\qty[\eta_0](t,\bullet)_\#\eta(0)_t,\Phi^{\R^m}_\tau\qty[\eta_0](t,\bullet)_\#\eta(0)_s)\label{eq:initial_bound}\\
        &+W_2\qty(\Phi^{\R^m}_\tau\qty[\eta_0](t,\bullet)_\#\eta(0)_s,\Phi^{\R^m}_\tau\qty[\eta_0](s,\bullet)_\#\eta(0)_s).\label{eq:flow_bound}
    \end{align}
    Set $a\coloneqq a_{R_E}$, the coercivity margin \eqref{eq:coercivity_margin} at the energy level $E=J(\eta_0)$; it is positive by \cref{assump:curvature}.
    Since $J(\eta(\tau))\le J(\eta_0)=E$ along the curve of maximal slope, $\eta(\tau)\in\Dcal_{R_E}$, so the definition \eqref{eq:coercivity_margin} of $a=a_{R_E}$ gives $C_{\theta\theta}\int_{\XtimesY}\abs{\nabla_x\varphi_t^{\eta(\tau)}}(1+\abs{x}^2)\dd{\mu_t^{\eta(\tau)}}\le\epsilon-a$ for every $t\in I$ and $\tau\ge0$.

    \emph{Initial-layer term \eqref{eq:initial_bound}.}
    Let $D_0\subset\R^m$ be a bounded set containing $\supp\eta(0)_t$ for a.e.~$t\in I$, as in \cref{assump:initial_data}.
    For $\theta^1$, $\theta^2\in D_0$, a computation as in \eqref{eq:similar_comp_to_scagliotti} gives
    \begin{equation}
    \begin{aligned}
        &\dv{\tau}\abs{\Phi^{\R^m}_\tau\qty[\eta_0](t,\theta^1)-\Phi^{\R^m}_\tau\qty[\eta_0](t,\theta^2)}^2\\
        \leq{}&-2\epsilon\abs{\Phi^{\R^m}_\tau\qty[\eta_0](t,\theta^1)-\Phi^{\R^m}_\tau\qty[\eta_0](t,\theta^2)}^2
        +2\abs{\Phi^{\R^m}_\tau\qty[\eta_0](t,\theta^1)-\Phi^{\R^m}_\tau\qty[\eta_0](t,\theta^2)}\\
        &\times\abs{\nabla_\theta\qty(\fdv{L}{\eta}\qty[\eta(\tau)]\qty(\Phi_\tau\qty[\eta_0]\qty(t,\theta^1))-\fdv{L}{\eta}\qty[\eta(\tau)]\qty(\Phi_\tau\qty[\eta_0]\qty(t,\theta^2)))}.
    \end{aligned}
        \label{similar_comp_to_scagliotti2}
    \end{equation}
    By the fundamental theorem of calculus and $\norm{D^2_\theta v}_{\opn}\le C_{\theta\theta}(1+\abs{x}^2)$ from \eqref{eq:Cthth}, the last factor is bounded by
    \begin{equation}
    \begin{aligned}
        &\abs{\nabla_\theta\qty(\fdv{L}{\eta}\qty[\eta(\tau)]\qty(\Phi_\tau\qty[\eta_0]\qty(t,\theta^1))-\fdv{L}{\eta}\qty[\eta(\tau)]\qty(\Phi_\tau\qty[\eta_0]\qty(t,\theta^2)))}\\
        &\quad\le(\epsilon-a)\abs{\Phi^{\R^m}_\tau\qty[\eta_0](t,\theta^1)-\Phi^{\R^m}_\tau\qty[\eta_0](t,\theta^2)},
    \end{aligned}
    \label{eq:inital_bound1}
    \end{equation}
    because $C_{\theta\theta}\sup_{t}\int_{\XtimesY}\abs{\nabla_x\varphi_t^{\eta(\tau)}}(1+\abs{x}^2)\dd{\mu_t^{\eta(\tau)}}\le\epsilon-a$ by \eqref{eq:coercivity_margin} and the energy-sublevel bound.
    Hence the right-hand side of \eqref{similar_comp_to_scagliotti2} is at most
    \(
        -2a\abs{\Phi^{\R^m}_\tau\qty[\eta_0](t,\theta^1)-\Phi^{\R^m}_\tau\qty[\eta_0](t,\theta^2)}^2,
    \)
    and the Gronwall inequality gives
    \begin{equation}
        \abs{\Phi^{\R^m}_\tau\qty[\eta_0](t,\theta^1)-\Phi^{\R^m}_\tau\qty[\eta_0](t,\theta^2)}\le\e^{-a\tau}\abs{\theta^1-\theta^2}.\label{eq:initial_bound2}
    \end{equation}
    Integrating \eqref{eq:initial_bound2} against an optimal plan between $\eta(0)_t$ and $\eta(0)_s$ yields $\eqref{eq:initial_bound}\le\e^{-a\tau}W_2(\eta(0)_t,\eta(0)_s)$.

    \emph{Flow term \eqref{eq:flow_bound}.}
    We first record the layer-difference estimates: for every $\tau\ge0$ and a.e.~$s$, $t\in I$,
    \[
    \begin{gathered}
        \sup_{x\in B_X}\abs{X_{0,t}^{\eta(\tau)}(x)-X_{0,s}^{\eta(\tau)}(x)}\le C\abs{t-s},\qquad
        W_1\qty(\mu_t^{\eta(\tau)},\mu_s^{\eta(\tau)})\le C\abs{t-s},\\
        \sup_{(x,y)\in B_{\XtimesY}}\abs{\nabla_x\varphi_t^{\eta(\tau)}(x,y)-\nabla_x\varphi_s^{\eta(\tau)}(x,y)}\le C\abs{t-s},
    \end{gathered}
    \]
    where $C$ depends only on the support bound of \cref{prop:orbit_in_L^infty} and the energy level $E$.
    The first estimate follows by integrating the state equation between $s$ and $t$, the drift being bounded on the reachable compact set uniformly in $\tau$; the second by pushing the common initial datum $\mu_0$ forward with the two layers; and the third from the representation $\nabla_x\varphi_t^{\eta(\tau)}(x,y)=M^{\eta(\tau)}(x;1,t)^\top\nabla_x\ell(X_{t,1}^{\eta(\tau)}(x),y)$, the flow identity $X_{s,1}^{\eta(\tau)}=X_{t,1}^{\eta(\tau)}\circ X_{s,t}^{\eta(\tau)}$ and the corresponding cocycle identity for the fundamental matrices, together with the Gronwall bounds \cref{eq:Gronwall_bound_X,eq:Gronwall_bound_M} and the Lipschitz continuity of $\nabla_x\ell$ and $\mdif{x}v$ on that set.
    For $\theta\in\supp\eta(0)_s$, we apply the computation of \eqref{eq:similar_comp_to_scagliotti} to the difference $\Phi^{\R^m}_\tau[\eta_0](t,\theta)-\Phi^{\R^m}_\tau[\eta_0](s,\theta)$, splitting the drift terms at the common evaluation point $\Phi^{\R^m}_\tau[\eta_0](s,\theta)$; this gives a curvature term, bounded as in \eqref{eq:inital_bound1} and absorbed into $-2a\abs{\cdots}^2$, together with the source
    \begin{align}
        & \abs{\nabla_\theta\la\nabla\qty(\varphi_t^{\eta(\tau)}-\varphi_s^{\eta(\tau)}), v\qty(\bullet,\Phi^{\R^m}_\tau\qty[\eta_0]\qty(s,\theta))\ra_{\mu_t^{\eta(\tau)}}}\\
        &+\abs{\nabla_\theta\la\nabla_x\varphi_s^{\eta(\tau)}, v\qty(\bullet,\Phi^{\R^m}_\tau\qty[\eta_0]\qty(s,\theta))\ra_{\mu_t^{\eta(\tau)}-\mu_s^{\eta(\tau)}}}\leq C\abs{t-s},
    \end{align}
    where the two source terms are estimated as in \eqref{eq:first_h} and \eqref{eq:third_h} in \cref{appendix:lambda_convex}, now using the layer-difference estimates above; the constant $C$ is independent of $\tau$ and of $\Dir(\eta(\tau))$ for the same reason.
    Thus $\dv{\tau}\abs{\Phi^{\R^m}_\tau[\eta_0](t,\theta)-\Phi^{\R^m}_\tau[\eta_0](s,\theta)}^2\le-2a\abs{\cdots}^2+2C\abs{t-s}\abs{\cdots}$, and since $\Phi^{\R^m}_0[\eta_0](t,\theta)=\theta=\Phi^{\R^m}_0[\eta_0](s,\theta)$, the Gronwall inequality gives
    % \begin{equation}
    \(
        \abs{\Phi^{\R^m}_\tau\qty[\eta_0](t,\theta)-\Phi^{\R^m}_\tau\qty[\eta_0](s,\theta)}\le\frac{C}{a}\abs{t-s},
        % \label{eq:flow_bound1}
    % \end{equation}
    \)
    so $\eqref{eq:flow_bound}\le(C/a)\abs{t-s}$.

    Combining the two terms,
    \begin{equation}
        W_2(\eta(\tau)_t,\eta(\tau)_s)\le\e^{-a\tau}W_2(\eta(0)_t,\eta(0)_s)+\frac{C}{a}\abs{t-s},\label{eq:orbit_contraction}
    \end{equation}
    with $C$ independent of $\tau$ and of the Dirichlet energy being estimated.
    Inserting \eqref{eq:orbit_contraction} into \eqref{eq:Dirichlet} yields a bound
    \(
        \Dir(\eta(\tau))\le2\e^{-2a\tau}\Dir(\eta(0))+2\qty(\frac{C}{a})^2\lim_{\varepsilon\searrow0}\frac{C_1}{\varepsilon^{3}}\iint_{I_\varepsilon}\abs{t-s}^2\dd{t\dd s}\le C^\prime,
    \)
     uniform in $\tau\ge0$, since $\Dir(\eta(0))<\infty$ by \cref{assump:initial_data}.
    Together with \eqref{eq:rep_for_nonlocal}, \cref{prop:Sobolev} and \cite[Proposition 6.4]{LAVENANT2019688}, this proves the claim.
\end{proof}

%%%%%%%%%%%%%%%%%%%%
\section*{Acknowledgments}
%%%%%%%%%%%%%%%%%%%
The author would like to thank Goro Akagi for valuable discussions concerning the gradient inequality and Norikazu Saito for valuable advice on the exposition.

\printbibliography

\end{document}